\documentclass[twoside,11pt]{article}

\usepackage{blindtext}
\PassOptionsToPackage{table}{xcolor}

\usepackage[abbrvbib, preprint]{jmlr2e}

\usepackage{epsfig,epsf,psfrag}
\usepackage{amsmath,graphicx,bm,xcolor,url,overpic}
\usepackage{array}
\usepackage{verbatim}
\usepackage{bm}
\usepackage{algorithmic}
\usepackage{algorithm}
\usepackage{enumitem}
\usepackage{verbatim}
\usepackage{textcomp}
\usepackage{mathrsfs}
\usepackage{epstopdf}
\usepackage{subfigure}
\usepackage{acronym}
\usepackage{diagbox}
\usepackage{tikz}
\usetikzlibrary{graphs, positioning, quotes, shapes.geometric}
\usepackage{thmtools, thm-restate}

\newcommand{\openone}{\leavevmode\hbox{\small1\normalsize\kern-.33em1}} 

\catcode`~=11 \def\UrlSpecials{\do\~{\kern -.15em\lower .7ex\hbox{~}\kern .04em}} \catcode`~=13 

\allowdisplaybreaks[4]

\newcommand{\kl}{D_{\mathrm{kl}}}

\newcommand{\calA}{\mathcal{A}}
\newcommand{\calB}{\mathcal{B}}

\newcommand{\calE}{\mathcal{E}}
\newcommand{\calF}{\mathcal{F}}
\newcommand{\calG}{\mathcal{G}}
\newcommand{\calH}{\mathcal{H}}

\newcommand{\calL}{\mathcal{L}}

\newcommand{\calN}{\mathcal{N}}

\newcommand{\calR}{\mathcal{R}}
\newcommand{\calS}{\mathcal{S}}
\newcommand{\calT}{\mathcal{T}}

\newcommand{\rmb}{\mathrm{b}}

\newcommand{\rmc}{\mathrm{c}}

\newcommand{\rmd}{\mathrm{d}}

\newcommand{\rmg}{\mathrm{g}}

\newcommand{\bbE}{\mathbb{E}}

\newcommand{\bbN}{\mathbb{N}}

\newcommand{\bbP}{\mathbb{P}}

\newcommand{\bbV}{\mathbb{V}}

\DeclareMathAlphabet{\mathbsf}{OT1}{cmss}{bx}{n}
\DeclareMathAlphabet{\mathssf}{OT1}{cmss}{m}{sl}

\DeclareSymbolFont{bsfletters}{OT1}{cmss}{bx}{n}  
\DeclareSymbolFont{ssfletters}{OT1}{cmss}{m}{n}
\DeclareMathSymbol{\bsfGamma}{0}{bsfletters}{'000}
\DeclareMathSymbol{\ssfGamma}{0}{ssfletters}{'000}
\DeclareMathSymbol{\bsfDelta}{0}{bsfletters}{'001}
\DeclareMathSymbol{\ssfDelta}{0}{ssfletters}{'001}
\DeclareMathSymbol{\bsfTheta}{0}{bsfletters}{'002}
\DeclareMathSymbol{\ssfTheta}{0}{ssfletters}{'002}
\DeclareMathSymbol{\bsfLambda}{0}{bsfletters}{'003}
\DeclareMathSymbol{\ssfLambda}{0}{ssfletters}{'003}
\DeclareMathSymbol{\bsfXi}{0}{bsfletters}{'004}
\DeclareMathSymbol{\ssfXi}{0}{ssfletters}{'004}
\DeclareMathSymbol{\bsfPi}{0}{bsfletters}{'005}
\DeclareMathSymbol{\ssfPi}{0}{ssfletters}{'005}
\DeclareMathSymbol{\bsfSigma}{0}{bsfletters}{'006}
\DeclareMathSymbol{\ssfSigma}{0}{ssfletters}{'006}
\DeclareMathSymbol{\bsfUpsilon}{0}{bsfletters}{'007}
\DeclareMathSymbol{\ssfUpsilon}{0}{ssfletters}{'007}
\DeclareMathSymbol{\bsfPhi}{0}{bsfletters}{'010}
\DeclareMathSymbol{\ssfPhi}{0}{ssfletters}{'010}
\DeclareMathSymbol{\bsfPsi}{0}{bsfletters}{'011}
\DeclareMathSymbol{\ssfPsi}{0}{ssfletters}{'011}
\DeclareMathSymbol{\bsfOmega}{0}{bsfletters}{'012}
\DeclareMathSymbol{\ssfOmega}{0}{ssfletters}{'012}

\DeclareMathOperator*{\argmax}{argmax}

\newcommand{\Ber}{\mathrm{Bern}}

\usepackage{xcolor}
\definecolor{darkblue}{RGB}{0,0,150}
\hypersetup{
    colorlinks=true,
    citecolor=darkblue,
    linkcolor=darkblue,
    urlcolor=darkblue
}

\usetikzlibrary{arrows.meta}
\usetikzlibrary{calc}
\usetikzlibrary{decorations.pathreplacing} 
\tikzset{
    axis/.style={-{Latex[length=2.5mm]}, thick},
    regretbar/.style={draw=black, line width=0.7pt},
    counted/.style={red, line width=1.2pt},
    free/.style={blue, dashed, line width=1.1pt},
    ellipsis/.style={line width=0.8pt},
}
\usetikzlibrary{fit}

\definecolor{natureBlack}{HTML}{000000}
\definecolor{natureOrange}{HTML}{E69F00}
\definecolor{natureSkyBlue}{HTML}{56B4E9}
\definecolor{natureGreen}{HTML}{009E73}
\definecolor{natureVermillion}{HTML}{D55E00}
\definecolor{acadBlue}{RGB}{0,114,178}     
\definecolor{acadRed}{RGB}{213,94,0}       
\definecolor{acadGreen}{RGB}{0,158,115}    
\definecolor{acadPurple}{RGB}{118,42,131}  

\usepackage{acronym}
\acrodef{mab}[MAB]{Multi-Armed Bandit}
\acrodef{bai}[BAI]{Best Arm Identification}
\acrodef{rm}[RM]{Regret Minimization}
\acrodef{fe}[FE]{Free Exploration}
\acrodef{ra}[RA]{Regret Accumulation}
\usepackage{mathtools}
\mathtoolsset{showonlyrefs}

\usepackage{adjustbox}
\usepackage{booktabs}
\usepackage{array}
\usepackage{tabularx} 
\usepackage{makecell}

\newcommand{\Tfe}{T_{\mathrm{FE}}} 
\newcommand{\Saved}{\mathrm{Save}}
\newcommand{\whatSaved}{\widehat{\mathrm{Saved}}}
\newcommand{\ife}{i_{\mathrm{FE}}}
\newcommand{\FE}{\textrm{FE}}
\newcommand{\RA}{\textrm{RA}}
\newcommand{\idagger}{{i^{\dagger}}}

\newcommand{\barDelta}{\bar{\Delta}}

\newcommand{\hatmu}{\hat{\mu}}
\newcommand{\ouralg}{\textsc{UFE-KLUCB-H}}
\newcommand{\ouralgfe}{\textsc{UFE}}
\newcommand{\klucbh}{\textsc{KLUCB-H}}
\newcommand{\klnu}{D_{\nu}}

\usepackage{bbm}

\newcounter{algline}[algorithm]
\newcommand{\alglabel}[1]{%
  \refstepcounter{algline}\label{#1}%
}

\usepackage{lastpage}
\jmlrheading{23}{2022}{1-\pageref{LastPage}}{1/21; Revised 5/22}{9/22}{21-0000}{Yunlong Hou, Zixin Zhong, Vincent~Y.~F.~Tan. \vspace{2em}Under review}

\ShortHeadings{On the Benefits of Free Exploration}{Hou, Zhong and Tan}
\firstpageno{1}

\begin{document}

\title{On the Benefits of Free Exploration for Regret Minimization in Multi-Armed Bandits}

\author{\name Yunlong Hou \email yhou@u.nus.edu \\
       \addr Department of Mathematics,\\ National University of Singapore
       \AND
       \name Zixin Zhong \email zixinzhong@hkust-gz.edu.cn \\
       \addr Data Science and Analytics Thrust, \\ Hong Kong University of Science and Technology (Guangzhou)
       \AND
       \name Vincent~Y.~F.~Tan \email vtan@nus.edu.sg \\
       \addr  Department of Mathematics,\\ Department of Electrical and Computer Engineering, \\ National University of Singapore}

\editor{My editor}

\maketitle

\begin{abstract}%
    We study a  stochastic multi-armed bandit problem where an agent is granted a free exploration budget before regret accumulates, a setting not captured by the classic regret minimization or pure exploration paradigms. 
    The goal is to design an adaptive policy that strategically explores the bandit instance in the initial  free exploration phase and minimizes the cumulative regret in the subsequent phase.
    We formalize this \emph{regret minimization with free exploration} problem and identify an interesting regime where the free exploration budget scales logarithmically with the time horizon.
    To quantify the amount of regret saved with high probability as a result of the availability of the free exploration phase, we introduce a novel set of policies known as  $(\alpha,\beta)$-probably saving policies. 
    We propose a two-phase, probably saving algorithm, \ouralg{}, which consists of a principled free exploration policy, \ouralgfe{}, and a history-aware regret minimization policy \klucbh{}. 
   Instance-dependent upper bounds on  \ouralg{} are derived, showing
that  \ouralg{} accumulates strictly less regret than policies that do not have access
to a free exploration phase. Complementarily, we derive instance-dependent lower bounds
based on novel multi-instance perturbation arguments tailored to the free-exploration
setting, demonstrating the near-optimality of  \ouralg{} for two-valued bandits.
Our upper and lower bounds reveal sharp phase transitions in the accumulated regret
depending on the amount of available free exploration.
    Simulations are conducted to demonstrate that forced exploration and adaptivity in the algorithm lead to greater regret savings.
\end{abstract}


\section{Introduction}\label{sec:intro}
In this paper, we study a non-standard \ac{rm} problem in which the learner is endowed with a budget of \ac{fe} before regret begins to accumulate. Classical \ac{rm} formulations aim to minimize cumulative regret from the outset, forcing the learner to balance exploration and exploitation throughout the entire horizon. By contrast, pure exploration problems such as \ac{bai} or simple regret minimization concentrate on information acquisition and evaluate performance {\em only on the final decision}, placing much greater emphasis on exploration. Our formulation   interpolates between these two classical regimes by allowing an initial exploration phase that is free of regret, followed by a standard \ac{rm} phase; see Figure~\ref{fig:comparison_formulation} for an illustration.

 Such staged exploration-and-deployment paradigms arise naturally in many modern learning systems. In practice, learning systems are rarely deployed directly into fully regret-bearing environments. Instead, they are typically developed through multiple stages, such as offline experimentation, simulation, sandbox testing, or beta deployment  before large-scale deployment. During these preliminary stages, the learner can often collect informative samples at negligible operational cost because mistakes do not directly affect end users, revenue, or safety-critical performance. Consequently, many practical systems naturally possess a limited exploration budget prior to deployment, during which aggressive information acquisition is possible before regret begins to accumulate. This phenomenon is common across robotics,  online experimentation, and broader machine learning pipelines.

Motivated by these practical considerations, our work advances the theoretical understanding of the exploration–exploitation trade-off in a direction that has received scant prior attention. We describe two representative motivating examples in the next subsection.

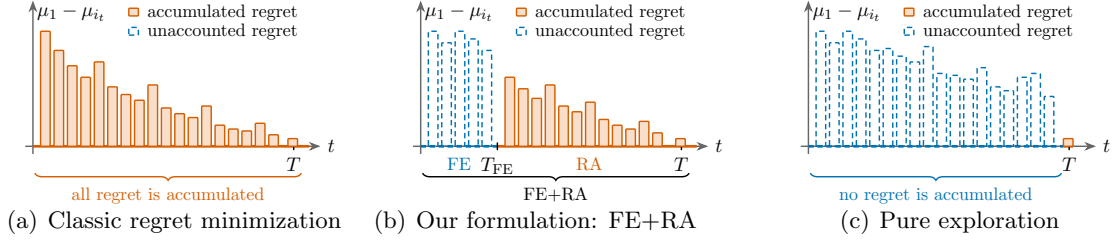
\begin{figure}[t]
\centering

\subfigure[Classic regret minimization]{
    \resizebox{0.31\textwidth}{!}{
        \begin{tikzpicture}[
    xscale=0.75, yscale=0.75,
    >=Stealth,
    thick,
    bar/.style={fill=acadRed!20, draw=acadRed, thick, rounded corners=0.5pt},
    unaccounted/.style={draw=acadBlue, densely dashed, thick, rounded corners=0.5pt}
]

\path[use as bounding box] (-0.8, -1.2) rectangle (8.0, 3.8);

\draw[->, thick, gray!80!black] (-0.2, 0) -- (7.5, 0) node[right, text=black, font=\normalsize] {$t$};
\draw[->, thick, gray!80!black] (0, -0.2) -- (0, 3.8) node[anchor=north west, text=black, inner sep=2pt, font=\normalsize] {$\mu_1 - \mu_{i_t}$};

\draw[acadRed, line width=1.5pt] (0,0) -- (7.2, 0);

\draw[thick, black] (6.8, 0.1) -- (6.8, -0.1);
\node[below, font=\normalsize] at (6.8, -0.1) {$T$};

\draw[thick, acadRed, decoration={brace, mirror, amplitude=4pt}, decorate] 
    (0.05, -0.7) -- (7.0, -0.7) node[midway, below=4pt, font=\small] {all regret is accumulated};

\draw[bar] (0.2, 0) rectangle (0.45, 3.0);
\draw[bar] (0.55, 0) rectangle (0.8, 2.5);
\draw[bar] (0.9, 0) rectangle (1.15, 2.1);
\draw[bar] (1.25, 0) rectangle (1.5, 1.8);
\draw[bar] (1.6, 0) rectangle (1.85, 2.2); 
\draw[bar] (1.95, 0) rectangle (2.2, 1.55);
\draw[bar] (2.3, 0) rectangle (2.55, 1.35);
\draw[bar] (2.65, 0) rectangle (2.9, 1.2);
\draw[bar] (3.0, 0) rectangle (3.25, 1.6); 
\draw[bar] (3.35, 0) rectangle (3.6, 1.0);
\draw[bar] (3.7, 0) rectangle (3.95, 0.85);
\draw[bar] (4.05, 0) rectangle (4.3, 0.75);
\draw[bar] (4.4, 0) rectangle (4.65, 1.05); 
\draw[bar] (4.75, 0) rectangle (5.0, 0.55);
\draw[bar] (5.1, 0) rectangle (5.35, 0.45);
\draw[bar] (5.45, 0) rectangle (5.7, 0.4);
\draw[bar] (5.8, 0) rectangle (6.05, 0.6); 
\draw[bar] (6.15, 0) rectangle (6.4, 0.3);

\draw[bar] (6.65, 0) rectangle (6.9, 0.2);

\begin{scope}[shift={(2.5, 3.4)}]
    \draw[bar] (0, 0) rectangle (0.25, 0.2);
    \node[right, font=\small, text=black] at (0.3, 0.1) {accumulated regret};
    \draw[unaccounted] (0, -0.5) rectangle (0.25, -0.3);
    \node[right, font=\small, text=black] at (0.3, -0.4) {unaccounted regret};
\end{scope}

\end{tikzpicture}
    }
}
\hfill
\subfigure[Our formulation: \ac{fe}+\ac{ra}]{
    \resizebox{0.31\textwidth}{!}{
        \begin{tikzpicture}[
    xscale=0.75, yscale=0.75,
    >=Stealth,
    thick,
    bar/.style={fill=acadRed!20, draw=acadRed, thick, rounded corners=0.5pt},
    unaccounted/.style={draw=acadBlue, densely dashed, thick, rounded corners=0.5pt}
]

\path[use as bounding box] (-0.8, -1.2) rectangle (8.0, 3.8);

\draw[->, thick, gray!80!black] (-0.2, 0) -- (7.5, 0) node[right, text=black, font=\normalsize] {$t$};
\draw[->, thick, gray!80!black] (0, -0.2) -- (0, 3.8) node[anchor=north west, text=black, inner sep=2pt, font=\normalsize] {$\mu_1 - \mu_{i_t}$};

\draw[acadBlue, line width=1.5pt, densely dashed] (0,0) -- (2.0, 0);
\draw[acadRed, line width=1.5pt] (2.0,0) -- (7.2, 0);

\draw[thick, black] (2.0, 0.1) -- (2.0, -0.1);
\node[below, font=\normalsize] at (2.0, -0.1) {$T_{\text{FE}}$};
\draw[thick, black] (6.8, 0.1) -- (6.8, -0.1);
\node[below, font=\normalsize] at (6.8, -0.1) {$T$};

\node[below, font=\small, acadBlue] at (1.0, -0.1) {FE};
\node[below, font=\small, acadRed] at (4.4, -0.1) {RA};

\draw[thick, black, decoration={brace, mirror, amplitude=4pt}, decorate] 
    (0.05, -0.7) -- (7.0, -0.7) node[midway, below=4pt, font=\small] {FE+RA};

\draw[unaccounted] (0.2, 0) rectangle (0.45, 3.0);
\draw[unaccounted] (0.55, 0) rectangle (0.8, 2.7);
\draw[unaccounted] (0.9, 0) rectangle (1.15, 3.0);
\draw[unaccounted] (1.25, 0) rectangle (1.5, 2.8); 
\draw[unaccounted] (1.6, 0) rectangle (1.85, 2.5);

\draw[bar] (2.2, 0) rectangle (2.45, 1.8);
\draw[bar] (2.55, 0) rectangle (2.8, 1.5);
\draw[bar] (2.9, 0) rectangle (3.15, 1.25);
\draw[bar] (3.25, 0) rectangle (3.5, 1.6); 
\draw[bar] (3.6, 0) rectangle (3.85, 1.05);
\draw[bar] (3.95, 0) rectangle (4.2, 0.9);
\draw[bar] (4.3, 0) rectangle (4.55, 1.15); 
\draw[bar] (4.65, 0) rectangle (4.9, 0.7);
\draw[bar] (5.0, 0) rectangle (5.25, 0.55);
\draw[bar] (5.35, 0) rectangle (5.6, 0.45);
\draw[bar] (5.7, 0) rectangle (5.95, 0.65); 
\draw[bar] (6.05, 0) rectangle (6.3, 0.35);

\draw[bar] (6.65, 0) rectangle (6.9, 0.2);

\begin{scope}[shift={(2.5, 3.4)}]
    \draw[bar] (0, 0) rectangle (0.25, 0.2);
    \node[right, font=\small, text=black] at (0.3, 0.1) {accumulated regret};
    \draw[unaccounted] (0, -0.5) rectangle (0.25, -0.3);
    \node[right, font=\small, text=black] at (0.3, -0.4) {unaccounted regret};
\end{scope}

\end{tikzpicture}
    }
}
\hfill
\subfigure[Pure exploration]{
    \resizebox{0.31\textwidth}{!}{
        \begin{tikzpicture}[
    xscale=0.75, yscale=0.75,
    >=Stealth,
    thick,
    bar/.style={fill=acadRed!20, draw=acadRed, thick, rounded corners=0.5pt},
    unaccounted/.style={draw=acadBlue, densely dashed, thick, rounded corners=0.5pt}
]

\path[use as bounding box] (-0.8, -1.2) rectangle (8.0, 3.8);

\draw[->, thick, gray!80!black] (-0.2, 0) -- (7.5, 0) node[right, text=black, font=\normalsize] {$t$};
\draw[->, thick, gray!80!black] (0, -0.2) -- (0, 3.8) node[anchor=north west, text=black, inner sep=2pt, font=\normalsize] {$\mu_1 - \mu_{i_t}$};

\draw[acadBlue, line width=1.5pt, densely dashed] (0,0) -- (6.65, 0);
\draw[acadRed, line width=1.5pt] (6.65, 0) -- (6.9, 0);

\draw[thick, black] (6.8, 0.1) -- (6.8, -0.1);
\node[below, font=\normalsize] at (6.8, -0.1) {$T$};

\draw[thick, acadBlue, decoration={brace, mirror, amplitude=4pt}, decorate] 
    (0.05, -0.7) -- (6.6, -0.7) node[midway, below=4pt, font=\small] {no regret is accumulated};

\draw[unaccounted] (0.2, 0) rectangle (0.45, 3.0);
\draw[unaccounted] (0.55, 0) rectangle (0.8, 2.7);
\draw[unaccounted] (0.9, 0) rectangle (1.15, 3.0);
\draw[unaccounted] (1.25, 0) rectangle (1.5, 2.8); 
\draw[unaccounted] (1.6, 0) rectangle (1.85, 2.5);
\draw[unaccounted] (1.95, 0) rectangle (2.2, 2.55);
\draw[unaccounted] (2.3, 0) rectangle (2.55, 2.35);
\draw[unaccounted] (2.65, 0) rectangle (2.9, 2.2);
\draw[unaccounted] (3.0, 0) rectangle (3.25, 2.6); 
\draw[unaccounted] (3.35, 0) rectangle (3.6, 1.9);
\draw[unaccounted] (3.7, 0) rectangle (3.95, 1.85);
\draw[unaccounted] (4.05, 0) rectangle (4.3, 1.75);
\draw[unaccounted] (4.4, 0) rectangle (4.65, 2.05); 
\draw[unaccounted] (4.75, 0) rectangle (5.0, 1.55);
\draw[unaccounted] (5.1, 0) rectangle (5.35, 1.45);
\draw[unaccounted] (5.45, 0) rectangle (5.7, 1.8);
\draw[unaccounted] (5.8, 0) rectangle (6.05, 1.9); 
\draw[unaccounted] (6.15, 0) rectangle (6.4, 1.3);

\draw[bar] (6.65, 0) rectangle (6.9, 0.2);

\begin{scope}[shift={(2.5, 3.4)}]
    \draw[bar] (0, 0) rectangle (0.25, 0.2);
    \node[right, font=\small, text=black] at (0.3, 0.1) {accumulated regret};
    \draw[unaccounted] (0, -0.5) rectangle (0.25, -0.3);
    \node[right, font=\small, text=black] at (0.3, -0.4) {unaccounted regret};
\end{scope}

\end{tikzpicture}
    }
}
\caption{Comparisons between classical \ac{rm}, our \ac{fe}+\ac{ra} formulation, and pure exploration. All setups aim to minimize the accumulated regret (in orange). Our formulation (b) sits between  classical \ac{rm} and pure exploration. When regret is not penalized, a more aggressive exploration strategy can be applied; this is reflected in the  blue bars in panels (b) and (c) being taller than the corresponding orange ones in panel (a).}
\label{fig:comparison_formulation}
\end{figure}

\subsection{Motivating Examples}\label{sec:examples}
\textbf{Robotics:}
Consider a mobile robot that is eventually deployed to perform a repetitive task, such as warehouse navigation, infrastructure inspection, or package delivery along a fixed route. Engineers must choose among several candidate control strategies that differ in their speed, robustness, and safety characteristics. At the outset, however, it is unclear which controller performs best in the target environment, since performance depends on factors such as sensor noise, wheel slippage, and unmodeled dynamics.

Before deployment, the robot typically undergoes a {\em calibration phase} (corresponding to the \ac{fe} phase in our formulation) in a laboratory or simulated environment. During this stage, engineers are free to test all candidate controllers, including poorly performing ones, in order to understand their behavior and estimate their reliability. The objective of this phase is purely informational rather than operational: inefficient or unstable behavior during calibration does not affect customers, production throughput, or safety-critical operations, and is therefore not counted as a loss. Consequently, the robot can explore aggressively and gather informative samples across all controllers without incurring regret.

Once the robot is deployed in the real world, however, controller choices have immediate operational consequences. Selecting a suboptimal controller may lead to slower task completion, increased energy consumption, or greater mechanical wear. At this {\em deployment stage}---the \ac{ra} phase in our formulation---the robot must continue selecting among the same set of controllers over time, but now the objective shifts to minimizing cumulative regret relative to the best controller in hindsight. The robot therefore exploits the information gathered during calibration to favor the most promising controllers, while still performing limited additional exploration to account for environmental changes or inaccuracies in the calibration process.

\noindent
\textbf{A/B Testing:}
Another example can be the widely deployed A/B testing pipelines. In modern large-scale systems, candidate variants (e.g., ranking algorithms or UI designs or advertisements) are first evaluated on a small scale (Beta versions) before being exposed to users of larger populations~\citep{mao2019newsheadline,bubeck2011pure}. Within this trial phase, a total budget of $\Tfe$ samples can be used to collect information without being counted in the primary performance objective. This practice aligns with the bandit literature~\citep{bubeck2011pure,zhang2023fast,qin2024optimizing}, where adaptively collected data is used to estimate the performance of different options. Crucially, these pre-deployment evaluations do not affect the experience of the majority of users and can therefore be treated as having zero or negligible cost. Consequently, they naturally serve as an \ac{fe} phase during which informative samples can be gathered without incurring any regret.

\begin{table*}[t]
\centering
\scriptsize
\renewcommand{\arraystretch}{1.15}
\setlength{\tabcolsep}{3.5pt}
\begin{tabular}{p{2.0cm} p{1.5cm} p{1.5cm} p{1.7cm} p{1.5cm} p{2.5cm} p{2.2cm}}
\toprule
\textbf{Setting} &
\makecell[l]{\textbf{Early} \\ \textbf{regret?}} &
\textbf{History?} &
\makecell[l]{\textbf{History} \\ \textbf{source}}&
\makecell[l]{\textbf{Adaptive} \\ \textbf{\ac{fe}}} &
\textbf{Objective} &
\makecell[l]{\textbf{Representative} \\ \textbf{methods}}  \\
\midrule

Classical RM (Fig.~\ref{fig:comparison_formulation}(a)) &
Yes &
No &
N/A &
N/A &
Cumulative regret minimization &
UCB,  KL-UCB, TS \\
\midrule

Warm-start RM &
No &
Yes &
\makecell[l]{Fixed/ \\  exogeneous }  &
No &
\makecell[l]{\ac{rm} with static \\  history} &
\makecell[l]{HUCB,  OTO,   \\      MIN-UCB} \\
\midrule

Explore-then-commit \ac{rm} &
Yes &
No &
N/A &
Yes &
Finite exploration, then exploitation &
ETC,  EOCP \\
\midrule

\makecell[l]{Pure explore./ \\\ac{bai}  (Fig.~\ref{fig:comparison_formulation}(c)) }&
No &
No &
N/A &
Yes &
\makecell[l]{Identify best arm/ \\ Min.\ simp.\ reg.} &
\makecell[l]{LUCB, UCB($\alpha$)  \\   SH, Track \& Stop }\\
\midrule

\rowcolor{blue!10}
\ac{fe}+\ac{ra} (ours  Fig.~\ref{fig:comparison_formulation}(b)) &
Not during FE &
Yes &
Adaptive/ endogenous &
Yes &
\ac{rm} after controllable FE &
\ouralg{} \\
\bottomrule
\end{tabular}
\caption{Comparison of related bandit formulations.}
\label{tab:comparison}
\end{table*}

Once the system or advertisement is deployed and exposed to the public, however, decisions directly impact engagement and revenue, and regret is incurred relative to the best-performing candidate in hindsight. The learning problem is thus to leverage the collected pre-deployment information to minimize online regret during deployment.

These examples highlight a practically important regime that is not captured by existing bandit models. Traditional bandit algorithms~\citep{auer2002finite,agrawal2017near,lattimore2020bandit} assume that regret accumulates from the outset, so all exploration is inherently costly. On the other hand, warm-start bandit models leverage a \textit{static} historical dataset collected prior to learning~\citep{shivaswamy2012multi,cheung2024leveraging,flore2025balancing}, but do not explicitly model a controllable exploration phase in which the learner actively chooses which arms to sample before regret begins to accrue. Our formulation captures this underexplored setting, which naturally arises in many real-world systems involving calibration, offline experimentation, and staged deployment. See Table~\ref{tab:comparison} for a detailed comparison of the different bandit formulations.

\subsection{Main Contributions} 
Our main contributions are as follows.

\begin{itemize}[leftmargin=*] 
\item  We formulate the \emph{\ac{rm} with \ac{fe}} problem in which 
an agent seeks to collect information via interactions with the environment during the \ac{fe} phase (initial $\Tfe$ time steps) with the target of minimizing the cumulative regret during the \ac{ra} phase (subsequent $T-\Tfe$ time steps).
This framework, illustrated in Figure~\ref{fig:comparison_formulation}, interpolates between standard \ac{rm} ($\Tfe = 0 $) and simple regret minimization ($\Tfe =T-1$). We discover an interesting time regime $\Tfe = O(\ln T)$ that cannot be solved by current fixed-budget \ac{bai} algorithms, e.g., Sequential Halving~\citep{karnin2013optimal}.
Akin to consistent policies used to establish instance-dependent lower bounds~\citep{lattimore2020bandit}, we  introduce the class of {\em $(\alpha,\beta)$-probably saving policies}. Policies in this class   save an $\alpha$-fraction of the maximum possible regret with high probability, more precisely with probability at least $1-O(T^{-\beta})$.

\item 
We devise \ouralg{}, consisting of an \ac{fe} policy \ouralgfe{} (Algorithm~\ref{algo:USE}) and an \ac{ra} policy \klucbh{} (Algorithm~\ref{algo:KLUCBH}). \ouralgfe{} explores the bad arms judiciously to collect sufficient salient information, so that \klucbh{} will sample   suboptimal arms much less often in the \ac{ra} phase. Together, \ouralg{} further reduces  the (instance-dependent) cumulative regret below the standard \ac{rm} problem~\citep{lai1985asymptotically,lattimore2020bandit} which is 
\begin{align}\label{equ:reg_lowbd}
        \sum_{i:\Delta_{i}>0} \frac{\Delta_i\cdot\ln T}{\kl(\nu_i,\nu_1)} +o(\ln T)\quad\mbox{as}\quad T\to\infty,
\end{align}
thanks to the \ac{fe} budget. 
Additionally, \klucbh{} can be applied given any type of historical observations, even in the case where they are not adaptively collected~\citep{shivaswamy2012multi}.
Given the provided historical data, \klucbh{} is asymptotically optimal for Bernoulli bandits.

\item 
We derive  instance-dependent lower bounds for any  $(\alpha,\beta)$-probably-saving policy based on novel multi-instance perturbation arguments tailored to the \ac{fe} setting.  
We show that even with \ac{fe}, any suboptimal arm $i$ must be pulled roughly $\bbE[T_{i,T}]\approx\frac{\ln T}{\kl(\nu_i,\nu_1)}$ times~\citep{lattimore2020bandit}. Using information-theoretic tools, we prove a tight upper bound  $\bbE[T_{i,\Tfe}]$ in the \ac{fe} phase (whose form differs in  three different regimes of $\Tfe$), which yields a lower bound on $\bbE[T_{i,T}-T_{i,\Tfe}]$ and thus on the regret in the \ac{ra} phase. Comparing these bounds on specific instances shows \ouralg{} almost attains the lower bound.

\item 
We conduct simulation studies comparing \ouralg{} with several baselines. The results show that more aggressive exploration is needed during the \ac{fe} phase, and that instance-adaptive strategies can use the \ac{fe} budget more efficiently.

\end{itemize}

\subsection{Related Works} 

The balance between exploration and exploitation has been studied extensively in stochastic bandits with a series of seminal works on \ac{rm} and \ac{bai}.

In traditional \ac{rm}, an agent minimizes regret (gap to the best cumulative reward) without prior data. Common algorithms include UCB and Thompson Sampling~\citep{auer2002finite,agrawal2017near}; KL-UCB is asymptotically optimal for Bernoulli arms and performs well more broadly~\citep{garivier2011KLUCB}.  \citet{lai1985asymptotically} provide an instance-dependent regret lower bound, which was later generalized by \citet{burnetas96}.

Beyond classic \ac{rm}, some works study \ac{rm} with static (i.e., non-adaptively acquired) historical data. \citet{shivaswamy2012multi} proposes HUCB1, adapting UCB1 to incorporate static observations; \citet{cheung2024leveraging} proposes MIN-UCB to leverage static data while allowing distribution shift between the \ac{fe} and \ac{ra} phases. \citet{flore2025balancing} devises OTO to balance the optimism and pessimism during the \ac{ra} phase with the static historical data. \citet{kausik2025leveraging} addresses the linear bandit setting with correlated rewards and proposed the  SOLD algorithm to minimize \ac{ra} regret using static historical data.
In our setting, the agent can adaptively use an \ac{fe} phase to reduce regret during the \ac{ra} phase.

In the \ac{bai} problem, there are two complementary setups. In the fixed-budget setup, given a horizon $T$, the agent aims to maximize the probability of finding the optimal  item in  $T$ steps~\citep{audibert2010best,karnin2013optimal,zhong2021probabilistic,carpentier16tight}; under the fixed-confidence setup, given $\delta > 0$, the agent aims to find the optimal item with  probability at least $1 - \delta$ in the smallest number of time steps~\citep{bubeck2013multiple,kaufmann13information}.
Besides, in  {\em simple regret minimization}, which straddles between \ac{rm} and \ac{bai}, an agent aims to maximize the expected reward of the identified outcome~\citep{carpentier2015simple,zhao2023revisiting}.  \citet{yang25} study fixed-budget BAI with historical data.

While most work treats \ac{rm} or \ac{bai}  separately, some study both simultaneously. \citet{degenne2019bridging}, \citet{zhong2023achieving}, and \citet{yang2026} characterize the \ac{rm}-\ac{bai} Pareto frontier, showing \ac{rm}  favors exploitation while \ac{bai} favors exploration. In the  linear bandit \ac{rm} setting, \citet{kirschner2021asymptotically} uses information-directed sampling to schedule exploration based on past data. Other papers propose algorithms that first explore then commit: \citet{qin2024optimizing} proposes a unified adaptive-experiment framework with adaptive stopping, characterizing the Pareto tradeoff between experiment length and cumulative regret via cost-aware Top-Two Thompson Sampling; \citet{zhang2023fast} introduces EOCP, EOCP-UG and KL-EOCP to jointly minimize exploration stopping time and post-stop regret.
\section{Problem Formulation}\label{sec:prob_form}
\paragraph{Notations:} For any $n \in \mathbb{N}$, we denote the set $\{1, \ldots , n\}$ as $[n]$. Let there be $L \in \mathbb{N}$ items indexed as $[L]$.
A random variable $X $ (or its distribution) is $\sigma$-sub-Gaussian  ($\sigma$-SG) if
$\mathbb{E}\big[ \exp( \lambda (X -\bbE X)) \big] \leq \exp( {\lambda^{2} \sigma^{2}}/{2}).$
In particular, a random variable supported on $[0,1]$ is $\sigma$-SG with $\sigma=1/2$.
Each item $i \in [L]$ is associated with a
reward distribution $\nu_i$ supported on $[0,1]$, with
mean $\mu_i$ and variance $\sigma_i^2$. We consider distributions from the single-parameter exponential family, i.e., each $\nu_i$ is parameterized by its mean $\mu_i$ when the distribution family is specified~\citep{garivier2016optimal}. 
We also assume there exists a lower bound $\sigma_{\min}^2$ on all the  variances $\sigma^2_i$'s.
The distributions
$\nu=\{\nu_i  \}_{i\in[L]}$, means $\{\mu_i   \}_{i\in[L]}$, and variances $\{\sigma_i^2   \}_{i\in[L]}$ are unknown to the agent.  
We let $\{ X_{i,t}  \}_{t=1}^T$ be an i.i.d.\ sequence of rewards associated with item $i$ during the $T$ time steps where
each $X_{i,t}$ is an independent sample from~$\nu_i$. 
Let $\kl(P,Q)$ denote the Kullback--Leibler (KL) divergence between two distributions $P$ and $Q$. We set $\klnu(\mu_i,\mu_j):=\kl(\nu_i,\nu_j)$ when the distribution family is specified. Let
$\rmd(a,b):=\kl(\Ber(a),\Ber(b))$ be the KL  divergence between two Bernoulli distributions with means $a$ and $b$ and $\rmd^+(a,b):=\rmd(a,b)\mathbbm{1}\{a<b\}.$

There are $L$ items whose distributions $\{\nu_i\}_{  i\in [L]}$ have means that satisfy 
$\mu_1 \geq \mu_2  \ge \ldots \ge \mu_L $, so item $i^*:=1$ is {\em optimal} and there may be multiple optimal items.
We denote $\Delta_{i,j} := \mu_i -\mu_j $ for any $i<j$ and $\Delta_{i} := \Delta_{1,i} $ as the {\em optimality gap} of item $ i $.
The agent uses an {\em online policy} $\pi$ to decide which item $i_t^{\pi }$ to pull at time  $t$ for all $1\le t\le T$.
More formally, a {\em policy} 
$\pi := (\pi_t)_{t=1}^{T}$, 
where the {\em sampling rule} $\pi_t$ determines the item $i_t^{\pi }$ to pull at time step $t$ based on the observation history $\calH_{t-1}:=(i_1^{\pi },  X_{ i_1^{\pi },1} , \ldots,  i_{t-1}^{\pi },   X_{  i_{t-1}^{\pi },{t-1}})$. That is, the random variable $i_t^{\pi }$ is $\calF_{t-1} $-measurable, where     
    $\calF_t  :=  \sigma ( \calH_t ) $.
The policy $\pi$ aims to minimize the {\em pseudo-regret after time $\Tfe$}
\begin{align*}
     R_{\Tfe+1:T} (\pi;\nu ) :
     = (T-\Tfe) \cdot \mu_1 -  \bbE \Bigg[ \sum_{t=\Tfe+1}^T  \mu_{i_t^\pi} \Bigg].
\end{align*}
The first \(\Tfe\) time steps form the {\em  \ac{fe} phase}, with the corresponding \ac{fe} policy denoted as \(\pi_{\FE}=(\pi_t)_{t=1}^{\Tfe}\). The remaining time steps constitute the  {\em \ac{ra} phase}, with the \ac{ra} policy \(\pi_{\RA}=(\pi_t)_{t=\Tfe+1}^{T}\). 
Since there are infinitely many possible \ac{fe} policies, we restrict attention to the class of policies that can reliably reduce a non-trivial fraction of the regret across a given class of instances. This class of policies is defined in the next section. 

\section{$(\alpha, \beta)$-Probably Saving Policies: A Novel Class of Policies for the \ac{rm} with \ac{fe} Problem} \label{sec:alpha_beta}
Akin to standard \ac{rm}, we aim to design policies with small regrets on a class of bandit instances $\Lambda$. This motivates the definition of the class of \emph{consistent} policies (cf.\ \citet{lai1985asymptotically}) specialized to our formulation in which the \ac{fe} period is available.
\begin{definition}
    A policy $\pi$ is called \emph{consistent} over a class of bandit instances $\Lambda$ if for all $\nu\in\Lambda$ and $p>0$,
    $
         R_{\Tfe+1:T} (\pi;\nu )= o(T^p). 
    $
\end{definition}
We show that, under our setting, for any consistent policy (see Lemma~\ref{lem:lwbd_whole}), any suboptimal item $i$ satisfies
\begin{align}\label{equ:itemi_lwbd}
    \liminf_{T\to\infty}\frac{\bbE[T_{i,T}]}{\ln T}\geq \frac{1}{\kl(\nu_i,\nu_1)} .
\end{align}
As we  consider the asymptotic regime where $T\to\infty$, if $\Tfe$ is a fixed constant, in contrast to the pulls of the suboptimal items which scale as $O(\ln T)$, then the fixed \ac{fe} phase will not make a difference to the regret results in the standard \ac{rm} problem. 
Therefore, we consider the case $\Tfe = g(T)$, where $g:\bbN\to\bbN$ is an increasing function that grows without bound. In other words, more \ac{fe} budget will be given when $T$ is larger.
Moreover, if $\Tfe$ grows faster than $\ln T$ (i.e., $g(T) = \omega( \ln T )$), then implementing Sequential Halving (SH) \citep{karnin2013optimal} in the \ac{fe} phase, followed by sampling the empirically best  in the \ac{ra} phase, yields constant regret. Specifically, using the result in \citet[Theorem~4.1]{karnin2013optimal}, we deduce that
\begin{align}
    \underbrace{ (3\log_2 L)\cdot \exp\Big(-\frac{\Tfe}{8H_2\log_2 L}\Big)}_{\text{failure prob.\ of fixed-budget BAI using SH}}\cdot \underbrace{\vphantom{\frac{-\Tfe}{8H_2\log_2 L}} T\Delta_{L}}_{\text{max regret}} = o(1),
\end{align}
where $H_2 = \max_{i>1}i/\Delta_i^{2}$ is a hardness parameter. 
Hence, we focus on the \textbf{non-trivial regime} in which $\Tfe$ grows without bound but not faster than $\ln T$, i.e., 
$$\Tfe = O(\ln T)\cap\omega(1).$$

Given  the true means $\{\mu_i\}_{i\in[L]}$, the optimal deterministic \ac{fe} strategy $\pi_{\FE}^*$ samples the items in   decreasing order of the gaps $\{\Delta_i\}_{i\in[L]}$, i.e., from $L$ to $1$, until each bad item has been sufficiently explored. Specifically, 
Lemma~\ref{lem:oracle_regret} shows
$\pi_{\FE}^*$ yields the \emph{oracle  saved regret}  
\begin{align}
    &\Saved^*(\nu,\Tfe) :=\Delta_{\ife} \bigg(\Tfe - \sum_{i=\ife+1}^L \frac{\ln T}{\kl(\nu_i,\nu_1)}\bigg) 
    +
    \sum_{i=\ife+1}^L \frac{\Delta_{i}\ln T}{\kl(\nu_i,\nu_1)}, \label{equ:oracle_regret}
    \\
    &\quad\mbox{where}\quad
    \ife := \max\bigg\{j \in [K]: \sum_{i=j}^L \frac{\ln T }{\kl(\nu_i,\nu_1)}\geq \Tfe\bigg\},
\end{align}
and where $L-\ife$ represents the maximum number of bad items that can possibly be ``eliminated'' in the \ac{fe} phase. 
\begin{figure}[t]
    \begin{center}
            \begin{tikzpicture}[
    xscale=1.08, yscale=0.96,
    >=Stealth,
    curve/.style={thick},
    label node/.style={font=\footnotesize} 
]

    \draw[->, natureBlack] (-0.3, 0) -- (6.5, 0) node[right, font=\footnotesize] {$\frac{T}{\ln T}$};
    \draw[->, natureBlack] (0, -0.3) -- (0, 3.8) node[right, font=\footnotesize] {$R_{\Tfe+1:T}(\pi;\nu)$};

    \coordinate (TFE) at (1.5, 0);
    \draw[natureBlack] (TFE) -- ++(0, +0.08) node[below=4pt, font=\scriptsize] {$\frac{\Tfe}{\ln T}$};
    

    \draw[curve, natureBlack, dotted, domain=0:6, samples=200] 
        plot (\x, {1.8 * ln(\x + 1)}) 
        node[right,yshift=+5pt , label node, xshift=-2pt] {$\sum_i \frac{\Delta_i \ln T}{\kl(\nu_i, \nu_1)}$};

    \draw[curve, natureGreen, domain=1.5:6, samples=200, thin] 
        plot (\x, {1.8 * ln(\x - 1.5 + 1)}) 
        node[right, yshift=-5pt, label node, color=natureGreen, xshift=-2pt] {$\sum_i \frac{\Delta_i \ln (T - \Tfe)}{D_{\text{kl}}(\nu_i, \nu_1)}$};

    \draw[curve, natureVermillion, domain=1.5:6, samples=200, line width=1.4pt] 
        plot (\x, {1 * ln(\x - 1.5 + 1)}) 
        node[right, yshift=5pt,label node, color=natureVermillion, font=\footnotesize\bfseries, xshift=-4pt] {Upper bd. (Sec.~\ref{sec:upbd})};

    \draw[curve, natureOrange, dash dot, domain=1.5:6, samples=200, line width=1.4pt] 
        plot (\x, {0.8 * ln(\x - 1.5 + 1)})
        node[right, label node, font=\footnotesize\bfseries, color=natureOrange ,xshift=-2pt] {Lower bd. (Sec.~\ref{sec:lwbd_tightness})};

    \draw[curve, natureSkyBlue, dashed, domain=1.5:6, samples=200] 
        plot (\x, {0.5 * ln(\x - 1.5 + 1)})
        node[right, yshift=-4pt, label node, xshift=-2pt] {$\pi^* = (\pi^*_{\text{FE}},\pi_{\RA}^*)$};

    \def\arrowX{3.8}
    \coordinate (ArrowTop) at (\arrowX, {1.8 * ln(\arrowX  + 1)});
    \coordinate (ArrowBot) at (\arrowX, {0.5 * ln(\arrowX - 1.5 + 1)});
    
    \draw[{Stealth[scale=0.7]}-{Stealth[scale=0.7]}, natureSkyBlue, semithick] (ArrowTop) -- (ArrowBot) 
        node[midway, right, yshift=6pt, color=natureSkyBlue, font=\tiny\itshape] {$\Saved^*(\nu,\Tfe)$};

    \draw[thick, natureBlack, decoration={brace, mirror, amplitude=3pt}, decorate] 
        (0.02, -0.1) -- (1.5, -0.1) node[midway, below=2pt, font=\footnotesize] {\FE};

    \draw[thick, natureBlack, decoration={brace, mirror, amplitude=3pt}, decorate] 
        (1.5, -0.1) -- (6.0, -0.1) node[midway, below=2pt, font=\footnotesize] {\RA};

\end{tikzpicture}
        \caption{Comparisons between different regret bounds. We set the ratio of $\Tfe$ to $ \ln T$  to be a   constant.
        The regret of \ouralg{} is smaller than that of the standard setup, and is close to the lower bound.
        }
        \label{fig:optimal_FE}
    \end{center}
\end{figure}
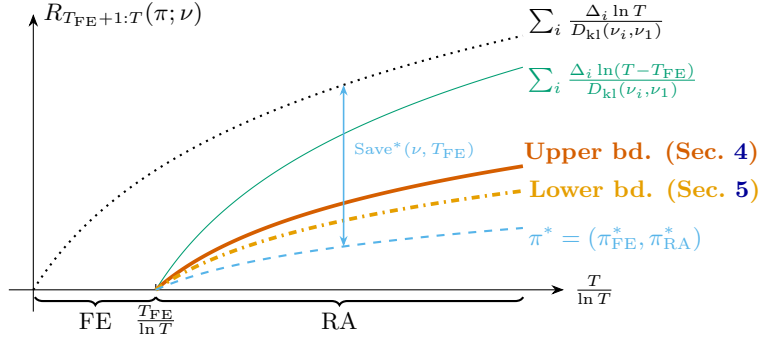
We illustrate the relationships between the regrets in Figure~\ref{fig:optimal_FE}. 
The top dotted line is the lower bound of the regret without \ac{fe}.
The green line is the lower bound of the regret if the policy starts at  $\Tfe$ without history, which can be achieved by running asymptotically optimal algorithms, and it will approach the dotted line as $T\to\infty$.
The blue dashed line is the regret of an optimal algorithm $\pi^*$, which adopts the optimal deterministic \ac{fe} policy $\pi_{\FE}^*$ and an asymptotically optimal history-aware \ac{ra} policy $\pi_{\RA}^*$, e.g., \klucbh{} for Bernoulli bandits.
No algorithm equipped with a deterministic \ac{fe}  policy can achieve smaller asymptotic regret than $\pi^*$.

In addition, with $\pi_{\FE}^*$ as the \ac{fe} policy, the instance-dependent lower bound of the expected regret of any consistent policy $\pi$ satisfies 
\begin{align}
    \underline{\mathfrak{R}}(\pi,\nu) &:= \liminf_{T\to\infty}\frac{R_{\Tfe+1:T} (\pi;\nu )}{\ln T}
    \geq 
    \sum_{i:\Delta_{i}>0}
    \frac{\Delta_i}{\kl(\nu_i,\nu_1)}
    -
    \limsup_{T\to\infty}\frac{\Saved^*(\nu,\Tfe)}{\ln T}.
    \label{equ:lowbd_decomp}
\end{align}
By adopting the optimal \ac{fe} strategy $\pi_{\FE}^*$ as the baseline, we define the \emph{empirical saved regret} of policy $\pi$ as
\begin{align}\label{equ:empi_saved}
    \widehat{\Saved}(\pi;\nu,\Tfe):=
    \sum_{i:\Delta_{i}>0} \Delta_{i}\cdot \min\bigg\{ T_{i,\Tfe},\frac{\ln T}{\kl(\nu_i,\nu_1)}\bigg\}.
\end{align}
Due to \eqref{equ:itemi_lwbd}, any number of pulls that exceeds $\frac{\ln T}{\kl(\nu_i,\nu_1)}$ for item $i$ is ``wasted'' and these excess pulls should have been allocated to other items. This motivates the $\min$ operator in $\widehat{\Saved}(\pi;\nu,\Tfe)$.

As the agent has no access to the instance $\nu$ and needs to collect information in an online fashion, it is unlikely that $\widehat{\Saved}(\pi;\nu,\Tfe)=\Saved^*(\nu,\Tfe).$ 
For example, if a policy always samples a highly suboptimal item $i$ with large gap $\Delta_{\mathrm{large}}$ in the \ac{fe} phase, it could save  a large amount (namely, $\Tfe\cdot\Delta_{\mathrm{large}}$) of regret if  $\Tfe\leq \frac{\ln T}{\kl(\nu_i,\nu_{i^*})}  + o(\ln T)$. However, without sampling other items, this policy cannot save too much regret when this item turns out to be optimal {\em on other instances}. To quantify how well the  policy $\pi$ can save  regret, we introduce an original  class of policies which we call {\em $\alpha$-saving policies}.
Recall that $\Tfe=O(\ln T)\cap \omega(1)$ as $T\to\infty$.
\begin{definition}[$\alpha$-Saving Policies]\label{def2}
    Given $\alpha\in(0,1)$ and free exploration budget $\Tfe$,
    a policy $\pi$ is an {\em $\alpha$-saving policy} over a class of instance $\Lambda$ if 
    for any instance $\{\nu_i\}_{i\in[L]}\in\Lambda$,
    \begin{align}\label{equ:alpha_saving}
            \liminf_{T\to\infty}\frac{\bbE[\widehat{\Saved}(\pi;\nu,\Tfe)]}{\Saved^*(\nu,\Tfe) }\geq\alpha.
    \end{align}
\end{definition}
While the defining property of the $\alpha$-saving policy is appealing, we remark that it also includes ``undesirable'' policies whose empirical saved regret has large variances.
More specifically, consider a $2$-item instance and a policy decides which item to sample uniformly at random and solely samples that selected item in the whole \ac{fe} phase. This policy is a $\frac{1}{2}$-saving policy, but the variance of the saved regret is as large as $(\Saved^*(\nu,\Tfe))^2/4 $.
We provide another example and more discussions in Appendix~\ref{app:discussions}.
These examples demonstrate that the family of $\alpha$-saving algorithms may not be entirely satisfactory. 

To overcome this limitation, we would like to additionally impose that policies designed are \emph{reliable}, in the sense that $\alpha\cdot\Saved^*(\nu,\Tfe)$ can be saved \emph{with high probability}, e.g., with probability $\ge 1-\delta$. 
The policy adopted in the two-item example can save $\frac{1}{2}\cdot \Saved^*(\nu,\Tfe)$ with probability $\frac{1}{2}$.
Naturally, when $T$ is large, we expect that the confidence level $1-\delta$ can approach $1$. Thus, we consider the setting in which  $\delta=\delta_T$ vanishes with~$T$. Inspired by the fixed-confidence \ac{bai} literature~\citep{kaufmann2016complexity,garivier2016optimal,karnin2013optimal,jamieson2014best} in which  $O(\ln \frac{1}{\delta})$ samples are usually required to guarantee a probability of success of $1-\delta$, we consider the case $\delta_T = O(T^{-\beta})$ for some $\beta\in(0,1)$ given the $\Tfe=O(\ln T)$ budget. This motivates the following definition.
\begin{definition}[$(\alpha, \beta)$-Probably Saving Policies]\label{def:absaving}
    Given $\alpha,\beta\in(0,1)$ and free exploration budget $\Tfe$,
    a policy $\pi$ is an {\em $(\alpha, \beta)$-probably-saving  policy} if 
    for any instance $\{\nu_i\}_{i\in[L]}$ in the instance class $\Lambda$, as $T\to\infty$,
    \begin{align}
        \bbP\big(
            \widehat{\Saved}(\pi;\nu,\Tfe)\ge\alpha \cdot\Saved^*(\nu,\Tfe)
        \big)
        \ge
        1 -O(T^{-\beta}). \label{eqn:alpha_sav}
    \end{align}
\end{definition}
\begin{remark} An $(\alpha,\beta)$-probably saving policy (for any $\beta \in (0,1)$) is an $\alpha$-saving policy because, by applying Markov's inequality to \eqref{eqn:alpha_sav}, we obtain 
\begin{align}
    \bbE\big[\widehat{\Saved}(\pi;\nu,\Tfe)\big]\ge
    \big(1 -O(T^{-\beta}) \big)\cdot\alpha\cdot\Saved^*(\nu,\Tfe),
\end{align}
which  clearly implies~\eqref{equ:alpha_saving}. 
\end{remark} 

\section{Upper Bound on the Regret}\label{sec:upbd}
We introduce our policy \ouralg{}, which consists of a simple and powerful \ac{fe} policy for the \ac{fe} phase which we call \textsc{Uniform sampling with Forced Elimination} (\ouralgfe{}). This is  followed by \klucbh{}  for the \ac{ra} phase.

\begin{algorithm}[t]
    \caption{\textsc{Uniform sampling with Forced Exploration} (\ouralgfe{})}\label{algo:USE}
    \textbf{inputs:} Free-exploration budget $\Tfe$, time horizon $T$, scaling parameter $d$, stopping parameter $c$, error parameter $\beta$.\\
    \textbf{procedures:}
    \begin{algorithmic}[1]
    \STATE \textbf{initialize:} Set $t=0$, $t(k,\delta) = 2 \sigma^2 d^{2k}\ln\frac{2}{\delta}$, $\delta_\beta = T^{-\beta}$,  and $\calA(0)=[L]$.\alglabel{alglineuse:init}
    \FOR{$k=1,2,3,\ldots$}  \addtocounter{algline}{1} 
    \STATE Pull $i\in\calA(k-1)$ until it has been pulled a total of $T(k):=t(k,\frac{\delta_\beta}{(k+1)^2})$ times.\alglabel{alglineuse:sample}
    \STATE Update the statistics for $i\in\calA(k-1)$ \alglabel{alglineuse:update}
    \begin{align}
        t = t + |\calA(k-1)|\cdot \big(T(k)-T(k-1)\big)
        \quad\mbox{and}\quad
        \hatmu_{i}(k) =\hatmu_{i,t} 
    \end{align}
    \STATE Update the empirically best item $i^*(k)$, the bad  set of items  $\calB(k)$ and the  active   sets of items $\calA(k)$ as in~\eqref{equ:use_active}.\alglabel{alglineuse:update_activeset}
    \STATE Pull each  item $i\in \calB(k)$ until it has been pulled a total of $\frac{\frac{c^2}{(c-2)^2}\ln T}{\klnu(\hatmu_i(k),\hatmu_{i^*(k)}(k))}$ times.\alglabel{alglineuse:forced_explore}
    \STATE \textbf{if $t=\Tfe$} \textbf{break}.
    \ENDFOR
    \end{algorithmic}
\end{algorithm}

\subsection{Free Exploration Policy: \ouralgfe{}}
Given the \ac{fe} budget $\Tfe$ and the error parameter $\beta$, we aim to design an \ac{fe} policy that can save as much regret as possible, i.e., maximize $\alpha$.
Our policy, \ouralgfe{}, presented in Algorithm~\ref{algo:USE},  acts in phases.
It first initializes the  number of pulls for each phase $t(k,\delta) = 2 \sigma^2 d^{2k}\ln\frac{2}{\delta}$ and the active item set $\calA(0)$ (Line~\ref{alglineuse:init}).
In phase~$k$, it samples each item $i\in\calA(k-1)$ in a round-robin manner for a total of $T_{i,t}=t(k,\frac{\delta_\beta}{(k+1)^2})=:T(k)$ times 
(Line~\ref{alglineuse:sample}), and updates the empirical means after $k$ phases $\hatmu_{i}(k)=\hatmu_{i,t}$ (Line~\ref{alglineuse:update}), where 
\begin{align}\label{equ:statistics}
     T_{i,t} =\sum_{s=1}^t\mathbbm{1}\{i_s = i\}
    \quad\mbox{and}\quad
    \hatmu_{i,t}=\frac{\sum_{s=1}^{t}X_{i,s}\mathbbm{1}\{i_s = i\}}{T_{i,t}}.
\end{align}
Next, \ouralgfe{} identifies the empirically best item $i^*(k):= \argmax_i \hatmu_i(k)$, the bad item set~$\calB(k)$ and maintains the {\em   active set} $\calA(k)$ (Line~\ref{alglineuse:update_activeset}) defined respectively as
\begin{align}
    \calB(k) &= \big\{i\in\calA(k-1): \hatmu_{i^*(k)}(k)-\hatmu_{i}(k)\geq c\cdot d^{-k}\big\}\quad\mbox{and}\quad
    \calA(k) = \calA(k-1)\setminus \calB(k) \label{equ:use_active}.
\end{align}
For each bad item $i\in\calB(k)$, forced exploration (by round-robin) is applied to guarantee  that the bad items are sampled sufficiently  (Line~\ref{alglineuse:forced_explore}).
The algorithm  terminates whenever the budget is used up, 
even when a phase has not been completed (in Lines~\ref{alglineuse:sample} or \ref{alglineuse:forced_explore}).
We assume it ends after a complete phase for ease of understanding.

Intuitively, as the definition of the empirical saved regret in~\eqref{equ:empi_saved} suggests, maximizing savings resembles \ac{rm} and a policy that chooses the item with the smallest lowest confidence bound  should be employed. While this is plausible, we highlight that the $\min$ operator in~\eqref{equ:empi_saved} hints that a good policy should stop sampling the bad items in a timely manner after sufficient exploration, as  excess pulls are not counted. To stop promptly, a good estimate of $\kl(\nu_i,\nu_1)$ (or $\mu_1$) is required. Therefore, both the good and worst items need to be sampled to provide suboptimality information and to save regret, respectively. 
  \ouralgfe{} achieves this by allocating $O(\beta \ln T)$ pulls to the good items, which is   necessary to satisfy Definition~\ref{def:absaving}, as suggested by the lower bound (see Section~\ref{sec:lwbd_tightness}). This is   followed by forced explorations of the bad items to approach $\frac{\ln T}{\kl(\nu_i,\nu_1)}$. 
The $\frac{c^2}{(c-2)^2}$ factor in Line~\ref{alglineuse:forced_explore} takes the uncertainty of the empirical means into account. \ouralgfe{}  samples conservatively to guarantee arm $i\in\calB(k)$ is sampled sufficiently many times.

\begin{remark}
    While \ouralgfe{}, Sequential Halving (SH)~\citep{karnin2013optimal}, Successive Elimination (SE) and the Na\"ive algorithm~\citep{Eyal2006action} all sample items---except for those already identified as suboptimal---for the same number of times. By contrast,  \ouralgfe{} \textbf{enforces additional exploration} of suboptimal items for more times (Line 6 of Algorithm~\ref{algo:USE}). This unique design is essential for an efficient \ac{fe} policy because the SH, SE and  Na\"ive algorithms will stop pulling an item once it is identified as suboptimal with certain confidence level. However, unique to our \ac{rm} with \ac{fe} problem, such an item needs to be pulled more times so that it is not likely to be subsequently pulled and generate (large) regret during the \ac{ra} phase.
\end{remark}

To quantify the fluctuations of the KL divergence under perturbations, for fixed $c>0$, we define 
\begin{align}
    \rho_{i,\min}(\varepsilon) := 
        \min_{\substack{
            \hat\mu_i \in [\mu_i-\varepsilon,\;\mu_i+\varepsilon] \\
            \hat\mu_1 \in [\mu_1-\varepsilon,\;\mu_1+\varepsilon] \\
            \hat\mu_1-\hat\mu_i\geq c\varepsilon
            }}
            \frac{\klnu(\mu_i,\mu_1)}{\klnu(\hat\mu_i,\hat\mu_1)}
    \quad\mbox{and}\quad
     \rho_{i,\max}(\varepsilon) := 
             \max_{\substack{
                 \hat\mu_i \in [\mu_i-\varepsilon,\;\mu_i+\varepsilon] \\
                 \hat\mu_1 \in [\mu_1-\varepsilon,\;\mu_1+\varepsilon] \\
                 \hat\mu_i-\hat\mu_i \geq c\varepsilon
                 }}
                 \frac{\klnu(\mu_i,\mu_1)}{\klnu(\hat\mu_i,\hat\mu_1)}.
\end{align}
In addition, we define the \emph{critical phases} $k_i^{\pm}$ and corresponding \emph{critical pull counts} $T_i^\pm$ for each item $i$ as
\begin{align}
    &k_i^+ :=\Big\lceil\log_d \frac{c+2}{\Delta_{i}}\Big\rceil  ,\qquad
    T_i^+ := \frac{c^2}{(c-2)^2}\cdot\frac{ \rho_{i,\max}(\frac{\Delta_i}{c-2})\cdot\ln T}{\kl(\nu_i,\nu_1)},
    \\
    &k_i^- :=  \Big\lfloor\log_d \frac{c-2}{\Delta_{i}}\Big\rfloor,\qquad
    T_i^- :=\frac{c^2}{(c-2)^2}\cdot\frac{\rho_{i,\min}(\frac{\Delta_i}{c-2})\cdot \ln T}{\kl(\nu_i,\nu_1)} .
\end{align}
We show in Appendix~\ref{app:proof_thmuse} that item $i$ will be in the bad item set $\calB(k)$ for some $k\in \{k_i^-,k_i^-+1,\ldots, k_i^+\}$, and  will be sampled for a total of $T_{i,\Tfe}\in\{T_i^-,T_i^-+1,\ldots, T_i^+\}$ times with high probability.

Therefore, let $\kappa$ denote the termination phase of \ouralgfe{}. 
If $\kappa>k_i^+$, then item $i$ has been sufficiently explored; 
if $\kappa<k_i^-$, then item $i\notin\calB(k)$ for any $k\leq\kappa$; 
if $\kappa\in \{k_i^-,k_i^-+1,\ldots, k_i^+\}$, item $i$ can be forced explored or is not in $\calB(k)$ yet.
Given these insights, we let $\hat{R}(\kappa;\nu,T,\Tfe)$ be the minimum value to the following optimization problem,
\begin{align}
    &\min_{ \{ T_{i,\Tfe}\}_{i\in[L]}} \sum_{i:\Delta_i>0} \Delta_i\cdot \min\Big\{ T_{i,\Tfe}, \frac{\ln T}{\kl(\nu_i,\nu_1)} \Big\}
    \\*
    &
    \begin{aligned}
      \quad \mbox{s.t.} \qquad  
        T_i^- &\leq T_{i,\Tfe} \leq T_i^+,\quad\;\; \kappa > k_i^+ 
    \\*
        T(\kappa-1) &\leq T_{i,\Tfe} \leq T_i^+,\quad\;\; k_i^-\leq \kappa\leq k_i^+
    \\*
         T(\kappa-1)& \leq T_{i,\Tfe}\leq T(\kappa),\quad  \kappa < k_i^-
    \\*
        \sum_{i\in[L]} T_{i,\Tfe} &=\Tfe.
    \end{aligned}\label{equ:opti_prob}
\end{align}
By convention, we set the minimum to be $\infty$ when the optimization problem~\eqref{equ:opti_prob} is infeasible. 
\begin{theorem}\label{thm:use}
    Given an instance $\nu$, error parameter $\beta$, free exploration budget $\Tfe$ and time horizon $T$, $\pi=(\pi_{\FE},\pi_{\RA})$ with $\pi_{\FE}=$\ \ouralgfe{} saves 
    \begin{align}\label{equ:use_save}
        \widehat{\Saved}(\pi;\nu,\Tfe)\geq\min_{\kappa\in[T]} \hat{R}(\kappa;\nu,T,\Tfe)
    \end{align}
    with probability $1-L\cdot T^{-\beta}$, where $\hat{R}(\kappa;\nu,T,\Tfe)$ is the solution to \eqref{equ:opti_prob}.
\end{theorem}
The proof of Theorem~\ref{thm:use} is postponed to Appendix~\ref{app:proof_thmuse}.
As the saved regret depends on the specific instance, the value of $\alpha$ can be deduced from \eqref{equ:use_save}. We will derive the explicit saved regret for specific instances in Section~\ref{sec:tightness}.

\subsection{Regret Accumulation Policy: \klucbh{}}
We present the history-aware algorithm, \klucbh{}, in Algorithm~\ref{algo:KLUCBH}. It uses the pulled items and observed history $\calH_{\Tfe}$ in the \ac{fe} phase as the warm-up data to facilitate the subsequent \ac{rm} process. 
The policy firstly initializes the item pulls and empirical means based on the history $\calH_{\Tfe}$ with $t=\Tfe$ by \eqref{equ:statistics} (Line~\ref{algline:init}).
At time steps $t>\Tfe$, it pulls the item with the largest KL-UCB index which is given by (Line~\ref{algline:update})
\begin{align}
    u_{i,t}:=
        \max\Big\{
           q\in  [0,1]: T_{i,t} \,  \rmd^+ (\hatmu_{i,t},q)\le\ln t+ 3\ln\ln t
        \Big\}.
\end{align}
Lastly, it updates the statistics via \eqref{equ:statistics}.
\begin{algorithm}[t]
    \caption{\klucbh{}}\label{algo:KLUCBH}
    \textbf{Inputs:} Time horizon $T$, history $\calH_{\Tfe} = \{(i_s,X_{i_s,s})\}_{s=1}^{\Tfe}$\\
    \textbf{Procedures:}
    \begin{algorithmic}[1]
    \STATE Initialize the statistics $T_{i,\Tfe}$ and $\hatmu_{i,\Tfe}$ for $i\in[L]$.\alglabel{algline:init}
    \FOR{$t = \Tfe+1,\ldots,T$} \addtocounter{algline}{1}
    \STATE Pull $i_t = \argmax_{i\in[L]} u_{i,t}$ and observe reward $X_{i_t,t}$.\alglabel{algline:klucb_index}
    \STATE Update the statistics $T_{i,t}$ and $\hatmu_{i,t}$ for $i\in[L]$.\alglabel{algline:update}
    \ENDFOR
    \end{algorithmic}
\end{algorithm}
We highlight that the history can be generated by {\em any} \ac{fe} policy and  \klucbh{} is still applicable. 
Let 
\begin{align}
    t_i:= \bigg\lfloor \frac{1}{\rmd(\mu_i,\mu_1)} \Big(\ln T + 3\ln\ln T \Big)\bigg\rfloor.
\end{align}
We present the following guarantee on Algorithm~\ref{algo:KLUCBH}.
\begin{theorem}[Upper Bound on the Regret]\label{thm:kulcbh}
    Given a bandit instance $\nu$,
    for any policy $\pi=(\pi_{\FE},\pi_{\RA})$ with $\pi_{\RA}=$\ \klucbh{} and
    let $T_{i,\Tfe}^\pi$ be the (random) number of pulls of item $i$ in the \ac{fe} phase, we have 
    \begin{align}\label{equ:klucbh_upbd}
      \!\!\!\!\overline{\mathfrak{R}}(\pi,\nu)\!:=\!\limsup_{T\to\infty}  \frac{R_{\Tfe+1:T} (\pi;\nu )}{\ln T}
        \!\le\!\sum_{i:\Delta_i>0} \!
          \Delta_{i}
             \bigg(
               \frac{1}{\rmd(\mu_i,\mu_1)}
               \!-\!
                \liminf_{T\to\infty}\frac{\bbE\big[\min\{t_i,T_{i,\Tfe}^\pi\} \big]}{\ln T}
             \bigg).
    \end{align}
    Furthermore, for any $(\alpha,\beta)$-probably-saving policy (with $\alpha,\beta\in (0,1)$), 
    \begin{align}\label{equ:klucbh_upbd_ab}
       \overline{\mathfrak{R}}(\pi,\nu)\le \sum_{i:\Delta_i>0} 
               \frac{\Delta_{i}}{\rmd(\mu_i,\mu_1)}
               -
                \alpha\cdot\liminf_{T\to\infty}\frac{\Saved^*(\nu,\Tfe)}{\ln T}.
    \end{align}
    For Bernoulli bandits, the following cleaner expression holds: 
    \begin{align}
         \overline{\mathfrak{R}}(\pi,\nu)\le\sum_{i:\Delta_i>0,i\leq\ife} 
           \frac{\Delta_{i}}{\rmd(\mu_i,\mu_1)}
           +
           \sum_{i:i>\ife} 
           \frac{(1-\alpha)\Delta_{i}}{\rmd(\mu_i,\mu_1)} .
    \end{align}
\end{theorem}
The proof of Theorem~\ref{thm:kulcbh} is provided in Appendix~\ref{app:proof_klucbh}.
The first term on the right-hand side of~\eqref{equ:klucbh_upbd} represents the upper bound in the standard \ac{rm} problem, whereas the second term is the saved regret thanks to the \ac{fe} phase. As \textsc{KL-UCB}  is asymptotically optimal for Bernoulli bandits~\citep{garivier2011KLUCB}, the first term is nearly optimal. 
Thus, the overall regret hinges on the performance of the \ac{fe} policy which is expected to maximize $\sum_{i:\Delta_i>0} \Delta_i \bbE\big[\min\{t_i,T_{i,\Tfe}^\pi\} \big]\approx \bbE[\widehat{\Saved}(\pi;\nu,\Tfe)]$ up to $\ln\ln T$ and $\kl(\nu_i,\nu_1)$ terms.
In particular, we show that $\alpha\cdot \Saved^*(\nu,\Tfe) $ regret can be saved in expectation for $(\alpha,\beta)$-probably-saving policies, where the term involving $\beta$ is $o(\ln T)$ and thus vanishes.
From this result, we demonstrate that the \ac{fe} budget is beneficial, as the regret is strictly smaller than the lower bound established under the standard \ac{rm} scenario in~\eqref{equ:reg_lowbd} when $\alpha$ is  large.
Additionally, the comparison between \eqref{equ:klucbh_upbd} and \eqref{equ:lwbd} suggests \klucbh{} is asymptotically optimal for Bernoulli instances when the \FE{} policy or history is given.

\begin{remark}
    Because the benefit brought by the \ac{fe} phase is of order $\ln T$, we seek to obtain a tight instance-dependent factor that pre-multiplies $\ln T$, i.e., asymptotic optimality is sought. 
    Therefore, we adopt KL-UCB~\citep{garivier2011KLUCB} as the backbone due to its asymptotic optimality (for Bernoulli instances).
    For UCB1-type policies~\citep{auer2002finite,shivaswamy2012multi,cheung2024leveraging}, the resulting bounds are looser by additional constant factors, which are comparable in magnitude to the saved regret. 
    We compare \klucbh{} with these two algorithms in Appendix~\ref{app:klucb_compare}. We leave the extension to other asymptotically optimal policies like Thompson Sampling~\citep{agrawal2017near} as future work.
\end{remark}

Combining Theorems~\ref{thm:use} and \ref{thm:kulcbh}, we can derive an upper bound on total regret of \ouralg{}, i.e., $\pi=(\pi_{\FE},\pi_{\RA})$ with $\pi_{\FE}=$\ \ouralgfe{} and $\pi_{\RA}=$\ \klucbh{}.
\begin{corollary}
    Given an instance $\nu$,  
    the expected regret of $\pi =$\ \ouralg{} satisfies
    \begin{align}
         \overline{\mathfrak{R}}(\pi,\nu)&\le \sum_{i:\Delta_i>0} 
           \frac{\Delta_{i}}{\rmd(\mu_i,\mu_1)}  
         -
         \liminf_{T\to\infty}
         \frac{\min_{\kappa\in[T]} \hat{R}(\kappa;\nu,T,\Tfe)}{\ln T},
    \end{align}
    where $\hat{R}(\kappa;\nu,T,\Tfe)$ is the optimal value of the optimization problem in~\eqref{equ:opti_prob}.
\end{corollary}
We further derive an  upper bound on the regret of \ouralg{} under the following specific instance, which we call the {\em two-valued instance}.
We let $\beta\in(0,1)$ be small and the Bernoulli instance $\nu=\{\nu_i\}_{i\in[L]}$ with $L_{\rmg}\in \{L/2,L/2+1,\ldots,L\}$ {\em good items} and $L_{\rmb}=L-L_{\rmg}$ {\em bad items} whose means are
\begin{align}\label{equ:two_value_instance}
    \!\!\mu_i = \frac{1}{2}+\varepsilon, \;\; i\in[L_{\rmg}]
    \quad\mbox{and}\quad
    \mu_i = \frac{1}{2},\;\; i\notin [L_{\rmg}],
\end{align}
where we set $\varepsilon=(c+2)d^{-n_1}=(c-2)d^{-n_2}$ for some large $n_1,n_2\in\bbN$, so that $\varepsilon$ is small and we can safely discard the floor and ceiling operators in $k_i^+$ and $k_i^-$.
Denote $\Tfe = \frac{\gamma L_{\rmb}\ln T}{\kl(\nu_L,\nu_1)}$ so that $\Tfe$ is parametrized solely by $\gamma$.
\begin{corollary}\label{cor:upbd_two_value}
   Let the switching points $\overline{r}_1 = \frac{L}{L_{\rmb}} (c+2)^2\beta + O(\varepsilon)$ and $\overline{r}_2 = 1 +  \frac{L_{\rmg}}{L_{\rmb}} (c+2)^2\beta\!+\!O(\varepsilon)$.  Define the following three intervals of $\gamma$:  
    \begin{align}
             \mathcal{R}_1   = [0,\overline{r}_1] ,\quad\mathcal{R}_2 = (\overline{r}_1, \overline{r}_2], \quad\mathcal{R}_3= (\overline{r}_2,\infty).
    \end{align}
    For the two-valued instance $\nu$ in \eqref{equ:two_value_instance}, the expected regret of $\pi=$\ \ouralg{} and its corresponding $\alpha$  satisfy
    \begin{align}
        \overline{\mathfrak{R}}(\pi,\nu)\le\left\{\! \begin{array}{ll}
            \frac{L_{\rmb}\cdot\Delta_{L}}{\rmd(\mu_L,\mu_1)}(1-\frac{L_{\rmb}\cdot\gamma}{L}), & \gamma \in \mathcal{R}_1, \\
               \frac{L_{\rmb}\cdot \Delta_{L}}{\rmd(\mu_L,\mu_1)}\big(\overline{r}_2-\gamma\big),& \gamma \in \mathcal{R}_2,\\
               \frac{L_{\rmb}\cdot \Delta_{L}}{\rmd(\mu_L,\mu_1)}\cdot O(\varepsilon), & \gamma \in \mathcal{R}_3,
        \end{array}  \right. 
    \; \mbox{and}\quad
        \alpha\ge \left\{ \begin{array}{ll}
              \frac{L_{\rmb}}{L},& \gamma \in \mathcal{R}_1 , \\
              1-\frac{L_{\rmg}\cdot (c+2)^2\beta}{L_{\rmb}\cdot\gamma}+O(\varepsilon),&\gamma \in \mathcal{R}_2, \\ 
            1-O(\varepsilon) ,&\gamma \in \mathcal{R}_3.
        \end{array} \right  .\label{equ:alpha_two_value_maintext}
    \end{align}
\end{corollary}
We plot the upper bound on $\overline{\mathfrak{R}}(\pi,\nu)$ as the orange line in Figure~\ref{fig:tight_bounds} and  provide results for more instances in Appendix~\ref{app:tightness_two_value}.  From Figure~\ref{fig:tight_bounds}, we observe that regret decreases as the \ac{fe} parameter $\gamma$ increases, as expected. For large $\gamma$ (i.e., $\gamma\in\mathcal{R}_3$), nearly all regret—up to a negligible~$O(\varepsilon)$ term—can be eliminated. For very small $\gamma$ (i.e., $\gamma\in\mathcal{R}_1$), the regret falls more slowly with $\gamma$ because the limited \ac{fe} budget is insufficient to estimate the means accurately and all arms are sampled. By contrast, for moderate~$\gamma$ (i.e., $\gamma\in\mathcal{R}_2$) the regret decays linearly in $\gamma$ as the bad items are sampled. The conclusions for $\alpha$ are analogous: a larger $\gamma$ yields a greater fraction of regret saved.

\begin{figure}[t]
    \begin{center}
            \begin{tikzpicture}[
    xscale=0.8, yscale=0.7,
    >=Stealth,
    thick,
    label node/.style={font=\footnotesize},
    tick label/.style={font=\scriptsize, below=6pt}
]

\def\xStart{0}
\def\xBetaLb{0.6}      
\def\xLKinkOne{1.2}    
\def\xOne{6.2}         
\def\xgreenEnd{6.7}         
\def\xBlueEnd{6.2}     
\def\xLKinkTwo{7.2}    
\def\xEnd{8.5}

\def\yMax{5.0}         
\def\yBlackKinkOne{4.8} 
\def\yGreenStart{4.82}   
\def\yFlat{0.05}         

\draw[->] (-0.2, 0) -- (\xEnd+0.4, 0) node[right] {$\gamma$};
\draw[->] (0, -0.2) -- (0, \yMax + 0.8) node[right, align=center] {$\frac{R_{T_{\text{FE}}+1:T}}{\ln T}$};

\draw[dotted, thin] (0, \yMax) -- (\xEnd, \yMax);
\node[above, font=\footnotesize, align=center] at (5, \yMax) {$\frac{L_b \Delta_L}{\rmd(\mu_L, \mu_1)}$ (standard \ac{rm} setup)};


\draw[natureVermillion, very thick] 
    (0, \yMax) -- 
    (\xLKinkOne, \yBlackKinkOne) -- 
    (\xLKinkTwo, \yFlat) -- 
    (\xEnd, \yFlat);

\draw[natureOrange, dash dot] 
    (\xBetaLb, \yGreenStart) -- 
    (\xgreenEnd, \yFlat) --
    (\xEnd, \yFlat);

\draw[natureGreen, thin, samples=100, smooth] 
    plot[domain=\xLKinkOne+0.25:\xOne] (\x, {\yMax - 0.755*\x + 0.7/\x +0.1});

\draw[natureSkyBlue, dashed, samples=100, smooth] 
    plot[domain=0:\xOne] (\x, {\yMax - \x/\xOne *\yMax + \yFlat}) --
    (\xEnd, \yFlat);


\draw[dotted, thin] (\xBetaLb, 0) -- (\xBetaLb, \yGreenStart);
\draw (\xBetaLb, 0) -- ++(0, +0.1) node[tick label,xshift=-5pt,yshift=14pt ]{$\underline{r}_1$};

\draw[dotted, thin] (\xLKinkOne, 0) -- (\xLKinkOne, \yBlackKinkOne);
\draw (\xLKinkOne, 0) -- ++(0, +.1) node[tick label,xshift=5pt,yshift=15pt ] {$\overline{r}_1$};

\draw[dotted, thin] (\xOne, 0) -- (\xOne, 0.9); 
\draw (\xOne, 0) -- ++(0, +0.1) node[tick label, xshift=-9pt,yshift=14pt] {$1$};

\draw[dotted, thin] (\xgreenEnd, 0) -- (\xgreenEnd, 0.45); 
\draw (\xgreenEnd, 0) -- ++(0, +0.1) node[tick label, xshift=0pt,yshift=2pt] {$\underline{r}_2$};

\draw (\xLKinkTwo, 0) -- ++(0, +0.1) node[tick label,xshift=3pt,yshift=15 ] {$\overline{r}_2$};

\draw[thick, natureBlack, decoration={brace, mirror, amplitude=3pt}, decorate] 
        (0.02, -0.1) -- (\xLKinkOne, -0.1) node[midway, below=2pt, font=\footnotesize] {$\calR_1$};

\draw[thick, natureBlack, decoration={brace, mirror, amplitude=3pt}, decorate] 
    (\xLKinkOne, -0.1) -- (\xLKinkTwo, -0.1) node[midway, below=2pt, font=\footnotesize] {$\calR_2$};

\draw[thick, natureBlack, decoration={brace, mirror, amplitude=3pt}, decorate] 
    (\xLKinkTwo, -0.1) -- (\xEnd+0.1, -0.1) node[midway, below=2pt, font=\footnotesize] {$\calR_3$};

\begin{scope}[shift={(5.0, 4.5)}]
    \draw[fill=white, draw=gray!20, opacity=0.9, text opacity=1] 
        (-0.05,0.1) rectangle (5.8,-2.15);

    \draw[natureVermillion, very thick] 
        (0.2, -0.3) -- (0.7, -0.3)
        node[right, black, font=\tiny] 
        {Upper bd (Cor.~\ref{cor:upbd_two_value})};

    \draw[natureSkyBlue, dashed, thick] 
        (0.2, -0.8) -- (0.7, -0.8)
        node[right, black, font=\tiny] 
        {Oracle $\pi^*$(Eqn.~\eqref{equ:oracle_regret})};

    \draw[natureOrange, dash dot, thick] 
        (0.2, -1.3) -- (0.7, -1.3)
        node[right, black, font=\tiny] 
        {Lower bd (Thm.~\ref{thm:lwbd_2value}(i))};

    \draw[natureGreen, thin] 
        (0.2, -1.8) -- (0.7, -1.8)
        node[right, black, font=\tiny] 
        {Improved lower bd (Thm.~\ref{thm:lwbd_2value}(ii))};
\end{scope}

\end{tikzpicture}
        \caption{Illustrations of the bounds on the two-valued instance in~\eqref{equ:two_value_instance} and Corollary~\ref{cor:tightness}. We recall that $\Tfe = \frac{\gamma L_{\rmb}\ln T}{\rmd(\mu_L,\mu_1)}$ is parametrized by $\gamma$. The green line is the improved lower bound when the alternative instance $\nu^{(2)}$ exists, and it approaches the more generic lower bound when $\gamma$ is large (Theorem~\ref{thm:lwbd_2value}).
        }
        \label{fig:tight_bounds}
    \end{center}
\end{figure}
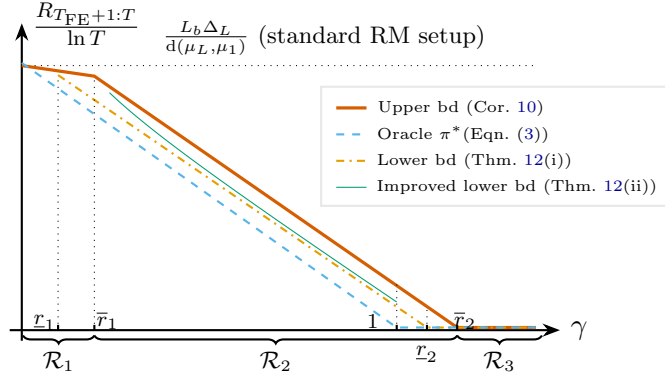

\section{Lower Bound and Asymptotic Optimality}\label{sec:lwbd_tightness}
\subsection{Information-Theoretic Lower Bound}\label{sec:lwbd}
We provide an instance-dependent lower bound on the regret in this section.
Define
\begin{align}
    \widetilde{\Saved}(\pi;\nu,\Tfe):=\sum_{i:\Delta_i>0} \Delta_i \min\bigg\{\frac{\ln T}{\kl(\nu_i,\nu_1)},\bbE[T_{i,\Tfe}]\bigg\},
\end{align}
a quantity that is similar to the empirical saved regret in~\eqref{equ:empi_saved}. We  first  show that the number of pulls for each item in the \textit{whole horizon} can be lower bounded similar to  the standard \ac{rm} setting  (i.e., without \ac{fe}).
\begin{lemma}\label{lem:lwbd_whole}
    For any consistent algorithm $\pi$, given instance $\nu$ and free exploration budget $\Tfe=o(T)$, the expected number of  pulls of item $i$ satisfies
    \begin{align}
        \liminf_{T\to\infty}\frac{\bbE[T_{i,T}]}{\ln T}
        \geq
        \frac{1}{\kl(\nu_i,\nu_1)}.
    \end{align}
    Furthermore, 
 \begin{align}
    \underline{\mathfrak{R}}(\pi,\nu)
      \ge  \sum_{i:\Delta_i>0}\!\!\frac{\Delta_i}{\kl(\nu_i,\nu_1)}
        -
        \limsup_{T\to\infty}
        \frac{\widetilde{\Saved}(\pi;\nu,\Tfe)}{\ln T}.
\end{align}
\end{lemma}
The proof is presented in Appendix~\ref{app:lwbd_whole}.
This lemma shows the total number of pulls for each item is still lower bounded by $\frac{\ln T}{\kl(\nu_i,\nu_1)}$ asymptotically as $T\to\infty$. The performance of the algorithm depends on how well  $\pi_{\FE}$ behaves. 
In the following, we   upper bound the second term  involving $\widetilde{\Saved}(\pi;\nu,\Tfe) $ for the class of $(\alpha, \beta)$-probably saving  algorithms.

Compared to standard \ac{rm} or \ac{bai}, the main difficulty in our problem lies in constructing meaningful alternative instances. In   \ac{rm}, one can simply replace a suboptimal item $i$ by an item with mean $\mu_1+\varepsilon$ (with the rest of the items unchanged), where $\varepsilon$ can be arbitrarily small~\citep{lai1985asymptotically,burnetas96}.   This  ubiquitous technique unfortunately fails here: changing the mean of a single item may have a negligible effect on the empirical saved regret $\widehat{\Saved}(\pi;\nu,\Tfe)$. For example, even if we exclude the regret incurred by  pulling the worst item $L$, the saved regret can be larger than $\alpha\cdot\Saved^*(\nu,\Tfe)$. This phenomenon becomes more pronounced when there are more bad items. Thus, the construction of  appropriate alternative instances, in general, requires {\em replacing more than one item}  and thus is highly instance-dependent.

Below we state a lower bound for the two-valued case. This simpler setting suffices to convey the key insights and to highlight the subtleties of the construction. 
Other instances are discussed in Appendix~\ref{app:more_lwbd}. Consider the two-valued instance $\nu$ with $L_{\rmg}$ good items and $L_{\rmb}=L-L_{\rmg}$ bad items (but with general means $\mu_1>\mu_L$ instead of the means specified in~\eqref{equ:two_value_instance}). We construct alternative instances $\nu^{(j)}$ with means $\{\mu_i^{(j)}\}_{i\in[L]},j=1,2,3$ satisfying
\begin{align}
    &\mu_i^{(1)}=
        \frac{\mu_1+\mu_L-\xi_1}{2}\mathbbm{1}\{i\in[L_{\rmg}]\}+\frac{\mu_1+\mu_L+\xi_1}{2}\mathbbm{1}\{i\notin[L_{\rmg}]\},
        \\
    &\mu_i^{(2)}=\mu_1\mathbbm{1}\{i\in[L_{\rmg}^\prime]\}
            +(\mu_L-\xi_2)\mathbbm{1}\{i\in(L_{\rmg}^\prime,L_{\rmg}]\}
        +\mu_L\mathbbm{1}\{i\notin[L_{\rmg}]\},\label{eqn:two_valued}
        \\
    &\mu_i^{(3)}=(\mu_L-\xi_3)\mathbbm{1}\{i\in[L_{\rmg}]\}
        +\mu_L\mathbbm{1}\{i\notin[L_{\rmg}]\}, 
\end{align}
where $\xi_1,\xi_2, \xi_3,L_{\rmg}^\prime<L_{\rmg}$ are subject to the conditions  
\begin{equation}
    \xi_2\geq\frac{1-\alpha}{2\alpha-1}\Delta_L,\quad \frac{(L_{\rmg}-L_{\rmg}^\prime)\ln T}{\kl(\nu_{L_{\rmg}}^{(2)},\nu_1^{(2)})}\geq \Tfe\quad\mbox{and}\quad \frac{L_{\rmg}\ln T}{\kl(\nu_{1}^{(j)},\nu_L^{(j)})}\geq \Tfe \quad  \mbox{for } j=1,3 .
\end{equation}
 Under these alternative instances, item $i\notin[L_{\rmg}]$ is not the worst item. We illustrate the instances in Figure~\ref{fig:lwbd_two_value} and elaborate  on their roles after Theorem~\ref{thm:lwbd_2value}.

 \begin{figure}[t]
    \begin{center}
            \begin{tikzpicture}[>=stealth, scale=0.75]


    \tikzstyle{dot} = [circle, fill, inner sep=1.1pt]

    \draw[-, thick] (0, 0) -- (10.0, 0);
    \draw[->, thick] (0.5, -0.2) -- (0.5, 4.5) node[above] {Mean};

    \draw[dotted, thin] (0.5, 3.5) -- (10.0, 3.5); 
    \draw[dotted, thin] (0.5, 2.5) -- (10.0, 2.5); 
    \draw[dotted, thin] (0.5, 1.5) -- (10.0, 1.5); 

    \node[left] at (0.5, 3.5) {$\mu_1$};
    \node[left] at (0.5, 2.5) {$\frac{\mu_1+\mu_L}{2}$};
    \node[left] at (0.5, 1.5) {$\mu_L$};

    \draw[thick] (2.8, 0) -- (2.8, 4.2);
    \draw[thick] (5.3, 0) -- (5.3, 4.2);
    \draw[thick] (7.8, 0) -- (7.8, 4.2);

    \node[natureVermillion, font=\large] at (1.8, 4.0) {$\nu$};
    \node[dot, natureVermillion] at (0.8, 3.5) {}; 
    \node[dot, natureVermillion] at (1.1, 3.5) {}; 
    \node[dot, natureVermillion] at (1.4, 3.5) {};
    \node[dot, natureVermillion] at (1.7, 3.5) {};
    \node[natureVermillion, font=\small, below] at (0.9, 3.5) {$\nu_1$};
    \node[natureVermillion, font=\small, below,xshift=3pt] at (1.8, 3.5) {$\nu_{L_{\rmg}}$};
    \node[dot, natureVermillion] at (2.0, 1.5) {}; 
    \node[dot, natureVermillion] at (2.3, 1.5) {}; 
    \node[dot, natureVermillion] at (2.6, 1.5) {};
    \node[natureVermillion, font=\small, below] at (2.5, 1.5) {$\nu_L$};

    \node[natureSkyBlue, font=\large] at (4.1, 4.1) {$\nu^{(1)}$};
    \node[dot, natureSkyBlue] at (4.3, 3.0) {}; 
    \node[dot, natureSkyBlue] at (4.6, 3.0) {}; 
    \node[dot, natureSkyBlue] at (4.9, 3.0) {};
    \node[above, natureSkyBlue,font=\small] at (4.9, 3.0) {$\nu_{L}^{(1)}$};
    \node[dot, natureSkyBlue] at (3.1, 2.0) {}; 
    \node[dot, natureSkyBlue] at (3.4, 2.0) {}; 
    \node[dot, natureSkyBlue] at (3.7, 2.0) {}; 
    \node[dot, natureSkyBlue] at (4.0, 2.0) {};
    \node[natureSkyBlue, font=\small, right] at (4.0, 2.0) {$\nu_{L_{\rmg}}^{(1)}$};
    \draw[<->] (4.3, 2.5) -- (4.3, 3.0) node[midway, left] {$\xi_1$};
    \draw[<->] (3.1, 2.5) -- (3.1, 2.0) node[midway, left] {};

    \node[natureGreen, font=\large] at (6.6, 4.1) {$\nu^{(2)}$};
    \node[dot, natureGreen] at (5.5, 3.5) {}; 
    \node[dot, natureGreen] at (5.8, 3.5) {};
    \node[natureGreen, font=\small, below, xshift=3pt] at (5.8, 3.5) {$\nu_{L_{\rmg}^\prime}^{(2)}$};
    \node[dot, natureGreen] at (6.7, 1.5) {}; 
    \node[dot, natureGreen] at (7.0, 1.5) {}; 
    \node[dot, natureGreen] at (7.3, 1.5) {};
    \node[natureGreen, font=\small, above] at (7.3,1.5) {$\nu_{L}^{(2)}$};
    \node[dot, natureGreen] at (6.1, 0.9) {}; 
    \node[dot, natureGreen] at (6.4, 0.9) {};
    \node[natureGreen, font=\small, right] at (6.4, 0.9) {$\nu_{L_{\rmg}}^{(2)}$};
    \draw[<->] (6.1, 1.5) -- (6.1, 0.9) node[midway, left] {$\xi_2$};

    \node[natureOrange, font=\large] at (9.2, 4.1) {$\nu^{(3)}$};
    \node[dot, natureOrange] at (8.9, 1.5) {}; 
    \node[dot, natureOrange] at (9.2, 1.5) {}; 
    \node[dot, natureOrange] at (9.5, 1.5) {};
    \node[natureOrange, font=\small, above] at (9.5, 1.5) {$\nu_{L}^{(3)}$};
    \node[dot, natureOrange] at (8.0, 0.9) {}; 
    \node[dot, natureOrange] at (8.2, 0.9) {}; 
    \node[dot, natureOrange] at (8.4, 0.9) {}; 
    \node[dot, natureOrange] at (8.6, 0.9) {};
    \node[natureOrange, font=\small, right] at (8.6, 0.9) {$\nu_{L_{\rmg}}^{(3)}$};
    \draw[<->] (8.6, 1.5) -- (8.6, 0.9) node[midway, left] {$\xi_3$};

\end{tikzpicture}
        \caption{Illustrations of the alternative instances for the two-valued instance in \eqref{eqn:two_valued} in which  $L=7,L_{\rmg}=4,L_{\rmg}^\prime=2$. Each dot represents an item and the vertical axis represents the mean.
        }
        \label{fig:lwbd_two_value}
    \end{center}
\end{figure}
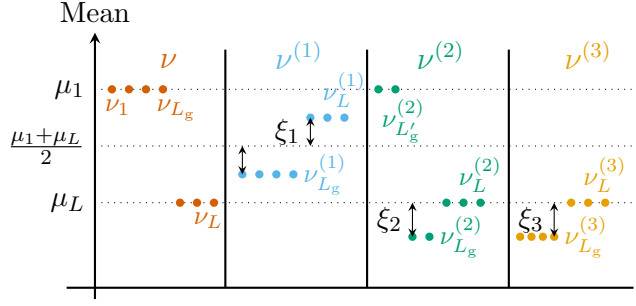

Note that while instances $\nu^{(1)}$ and $\nu^{(3)}$ (i.e., $\xi_1, \xi_3$) surely exist, the existence of $\nu^{(2)}$ depends on various other parameters, e.g., the original instance $\nu$ ($L_{\rmg}$ and $\Delta_L$), the \ac{fe} budget $\Tfe$ and the time horizon $T$. When $L_{\rmg}$ is large compared to $L_{\rmb}$ or $\Tfe$ is small,  $\nu^{(2)}$ exists and a stronger claim can be made.

Recall that for the two-valued case, we parametrize the \ac{fe} budget $\Tfe = \frac{\gamma L_{\rmb} \ln T}{\kl(\nu_L,\nu_1)}$ by~$\gamma$.
We define the switching  points for $\gamma$
\begin{align}
      \!\!\underline{r}_1\!:=\! \frac{\beta\kl(\nu_L,\nu_1)}{L_{\rmb}\max\big\{\kl(\nu_1,\nu_1^{(1)}),\kl(\nu_L,\nu_L^{(1)})\big\}}, \quad 
     \underline{r}_2\!:=\! \min\bigg\{ 1 \!+\! \frac{\beta\kl(\nu_L,\nu_1)}{L_{\rmb}\kl(\nu_1,\nu_1^{(3)})},\frac{\alpha}{1-\alpha}\bigg\}.\label{equ:lwbd_two_value_switching}
\end{align}
\begin{theorem}\label{thm:lwbd_2value}
    (i) If $\pi$ is an $(\alpha, \beta)$-probably-saving  algorithm on instances $\nu$ and $\{\nu^{(j)} \}_{j\in[3] }$ with $\alpha>1/2$, then 
    \begin{align}
         &\overline{\mathfrak{S}}(\pi,\nu) :=  \limsup_{T\to\infty}
        \frac{\widetilde{\Saved}(\pi;\nu,\Tfe)}{\ln T}  
    \le\begin{cases}
        \mbox{N.A.},& \gamma  < \underline{r}_1,
        \\
        \Delta_L\Big(
                    \frac{\gamma L_{\rmb}}{\kl(\nu_L,\nu_1)}-\frac{\beta\ln T}{\kl(\nu_1,\nu_1^{(3)})}
                \Big),& \underline{r}_1\leq \gamma  \leq \underline{r}_2,\!\!
        \\
        \Delta_L \cdot  \frac{L_{\rmb}\ln T}{\kl(\nu_L,\nu_1)},& \gamma  > \underline{r}_2,
    \end{cases}
\end{align}
where $N.A.$ indicates $\pi$ is not an $(\alpha, \beta)$-probably-saving algorithm for the given $\Tfe$ (or equivalently $\gamma$).

\noindent (ii) Furthermore, when $\nu^{(2)}$ exists, then 
for $\underline{r}_1\leq \gamma  \leq 1$, 
\begin{align}\label{equ:two_value_less}
    \overline{\mathfrak{S}}(\pi,\nu)\le\Delta_L \cdot \bigg(
             \frac{\gamma L_{\rmb}}{\kl(\nu_L,\nu_1)}-
            \frac{\zeta\cdot L_{\rmg}\cdot \beta}{\kl(\nu_{L_{\rmg}},\nu_{L_{\rmg}}^{(2)})}
        \bigg),
\end{align}
where $\zeta=\min\Big\{\frac{\kl(\nu_L,\nu_1)}{\gamma L_{\rmb}\cdot \kl(\nu_{L_{\rmg}}^{(2)},\nu_{1}^{(2)}) },1\Big\}$.
\end{theorem}
The proof of Theorem~\ref{thm:lwbd_2value} is in Appendix~\ref{app:proof_lwbd_2value}.
The saved regret exhibits different behaviors across three time regimes of $\Tfe$ (see Figure~\ref{fig:tight_bounds} for a Bernoulli instance). 
Firstly, the role of $\nu^{(1)}$ to show that when $\gamma <\underline{r}_1$, the \ac{fe} budget $\Tfe$ is too small, then no  $(\alpha,\beta)$-probably-saving algorithm  exists.
Secondly, $\nu^{(3)}$ is constructed to show that when $\underline{r}_1\leq \gamma \leq\underline{r}_2$, we need $\frac{\beta \ln T}{\kl(\nu_1,\nu_1^{(3)})}$ pulls to ``identify'' items in $[L_{\rmg}]$ as the better items and apply the rest of the available pulls to sample bad items to save regret. The saved regret increases linearly in $\gamma$.
Lastly, when $\Tfe$ is large, satisfying  $\gamma >\underline{r}_2$, all regret can be saved.

In addition, when the alternative instance $\nu^{(2)}$ exists, the lower bound can be improved, as shown in~\eqref{equ:two_value_less} and the green curve in Figure~\ref{fig:tight_bounds}. Here, we see that 
the number of good items $L_{\rmg}$ influences the saved regret: when $L_{\rmg}$ is large, it is harder to identify and save regret from the bad items.
Additionally, the \ac{fe} budget $\Tfe$ also plays an important role: when $\gamma$ increases in the interval of interest $[\underline{r}_1,1]$, $\zeta$ in~\eqref{equ:two_value_less} is  reduced. Intuitively, when $\Tfe$ is large, a significant amount of information about the instance can be learned, and most suboptimal items can be  sampled sufficiently often to save regret. Technically, there is ``less room'' to construct an alternative instance in which $\alpha$ fraction of the oracle saved regret cannot be achieved.

It can be seen that 
to construct a meaningful alternative instance, {\em multiple items} need to be replaced (as in Figure~\ref{fig:lwbd_two_value}), in contrast to the standard \ac{rm} case where replacing only {\em   one item} suffices. Thus, the construction is rather subtle.

\subsection{Asymptotic Optimality}\label{sec:tightness}

We demonstrate the tightness of \ouralg{} for the instance in \eqref{equ:two_value_instance}. For this instance, the switching points in \eqref{equ:lwbd_two_value_switching} can be simplified as $\underline{r}_1 = \frac{\beta}{L_{\rmb}}+O(\varepsilon)$ and $\underline{r}_2 = 1\!+\!\frac{\beta}{4L_{\rmb}} + O(\varepsilon)$.
\begin{corollary}\label{cor:tightness}
    (i) Let $\pi=$\ \ouralg{} and $\pi^\prime$ be any consistent and $(\alpha,\beta)$-probably saving algorithm with $\alpha$ specified in Corollary~\ref{cor:upbd_two_value}. Then, the gap between the upper and lower bounds 
   \begin{align}
        &\overline{\mathfrak{R}}(\pi,\nu)- \underline{\mathfrak{R}}(\pi^\prime,\nu)
        \le\left\{\! \begin{array}{ll}
               \frac{L_{\rmb}\cdot \Delta_{2}}{\rmd(\mu_L,\mu_1)}\big(\frac{(L_{\rmg}(c+2)^2-\frac{1}{4})\beta}{L_{\rmb}}+O(\varepsilon)\big),& \gamma \in \mathcal{R}_2\\
                 \frac{\Delta_L}{\rmd(\mu_L,\mu_1)}\cdot O(\varepsilon), & \gamma \in \mathcal{R}_3
        \end{array}  \right. 
    \end{align}
    where $\calR_2$ and $\calR_3$ are defined in Corollary~\ref{cor:upbd_two_value}.
   
\noindent   (ii) Furthermore, when  $\frac{3L_{\rmg}(c+2)^2\beta}{L_{\rmb}} + O(\varepsilon)\leq \gamma\leq \min\big\{1,\frac{L_{\rmg}-2}{4L_{\rmb}}\big\}$, the gap can be bounded as 
    \begin{align}\label{equ:tightness_two_value_nu2}
      \overline{\mathfrak{R}}(\pi,\nu)- \underline{\mathfrak{R}}(\pi^\prime,\nu)\le \frac{L_{\rmb}\cdot \Delta_{2}}{\rmd(\mu_L,\mu_1)}\bigg(\frac{((c+2)^2-\zeta)L_{\rmg}\beta}{L_{\rmb}}+O(\varepsilon)\bigg),
    \end{align}
    where $\zeta = \min\{ \frac{1+O(\varepsilon)}{4\gamma L_{\rmb}},1\}$.
\end{corollary}
The gap is plotted in Figure~\ref{fig:tight_bounds}.
When $\gamma$ is small, the lower bound on $\widetilde{\Saved}(\pi;\nu,\Tfe)$ is not applicable, we thus only compare the results for $\gamma$ in the intervals $\calR_2$ and $\calR_3$.
We remark that $\beta \in (0,1)$ is a small parameter, so in the interesting case $\gamma<1$, the gap is also commensurately  small. 
Nevertheless, from \eqref{equ:tightness_two_value_nu2}, this gap is smaller when there are more bad items and fewer good items.
When we have sufficient \ac{fe} budget ($\gamma>1$), almost all regret can be saved.

Corollary~\ref{cor:tightness} indicates \ouralg{} is almost optimal, in the sense that there is minimal   improvement (reduction) in the regret if other consistent and $(\alpha,\beta)$-probably saving algorithms are employed. 
\section{Numerical Experiments}\label{sec:exp}

We conduct numerical experiments to compare the efficacy of the proposed \ouralg{} to some baseline algorithms. 

\paragraph{Experimental design:}
First, we introduce the following baselines.
\begin{itemize}[leftmargin=*]
    \item \textsc{KL-UCB} (w/o \FE): the vanilla \textsc{KL-UCB} with effective horizon $T-\Tfe$. Any improvement over this baseline reflects the benefit of incorporating a \ac{fe} phase.
    \item \textsc{KL-UCB} (w/ \FE): \textsc{KL-UCB} runs from time step 1, with regret accumulating from $\Tfe+1$. It serves as a baseline, where the effect of the exploration-exploitation trade-off as in standard \ac{rm} is  retained. However, there is no theoretical guarantee on its regret under our \ac{rm} with \ac{fe} setting.
    \item \textsc{Uniform-KLUCB-H}: \textsc{Uniform Sampling (US)} in the \ac{fe} phase and \textsc{KLUCB-H} in the \ac{ra} phase. We adopt the devised history-aware \textsc{KLUCB-H} for the \ac{ra} phase to demonstrate its efficacy. We chose \textsc{US} for the \FE\ phase due to its simplicity. However, we highlight that \textsc{US} is \textit{not adaptive} to the instance. The amount of regret that can be saved by \textsc{US} highly depends on how uniform the means of the arms are. We provide a concrete comparison below (see \textbf{Exp 2}).
    \item \ouralg{}: our proposed algorithm. \ouralgfe{} \textit{adapts} to each instance to save more regret than \textsc{US}, and \klucbh{} leverages the observations from the \FE{} phase.
    \item \textsc{Oracle-KLUCB-H}: the oracle algorithm $\pi^*$ in Figure~\ref{fig:optimal_FE}, which has access to the means. This regret is unattainable in practice and serves only as a reference. Moreover, the lower bound is greater than this curve and difficult to express explicitly, thus it is not plotted.
\end{itemize}
We design two experiments:
In the first experiment \textbf{(Exp 1)}, we illustrate the benefit of the \ac{fe} phase and we shows more aggressive exploration is beneficial.
The instance $\nu$ has $L=10$ Bernoulli arms with means $\mu_i=0.9-0.02\times 1.5^{i-1}$  for $i\in[10]$. 
In the second experiment \textbf{(Exp 2)}, we demonstrate that \ouralg{} is generalizable to other instances beyond Bernoulli ones and the adaptivity of \ouralgfe{} can use the \ac{fe} budget more efficiently. 
The two instances are composed by $10$ Gaussian arms with unit variance:
For instance A, $\mu_1=0.9$ and $\mu_i=0.1$ for $i\in[10]\setminus\{1\}$; for instance B, $\mu_1=0.9$ and $\mu_i=0.85$ for $i\in[5]\setminus\{1\}$ and $\mu_j=0.1$ for $i\in[10]\setminus[5]$.
We also present more results for some randomly generated instances with different numbers of arms $L=5,10,20$ to show the practicality of \ouralg{} in Appendix~\ref{app:exp}. 
In each experiment, we set $\Tfe=0.8H\ln T$ with $H=\sum_{i:\Delta_i>0}1/\kl(\nu_i,\nu_1)$ and $\beta=0.1$. Results are averaged over $100$ runs and reported with $\pm 1$ standard deviation across those runs.

\paragraph{Experimental Results:}
The results for Exp 1 are illustrated in Figure~\ref{fig:exp1}.
\begin{figure}[t]
    \centering
    \includegraphics[width=0.6\linewidth]{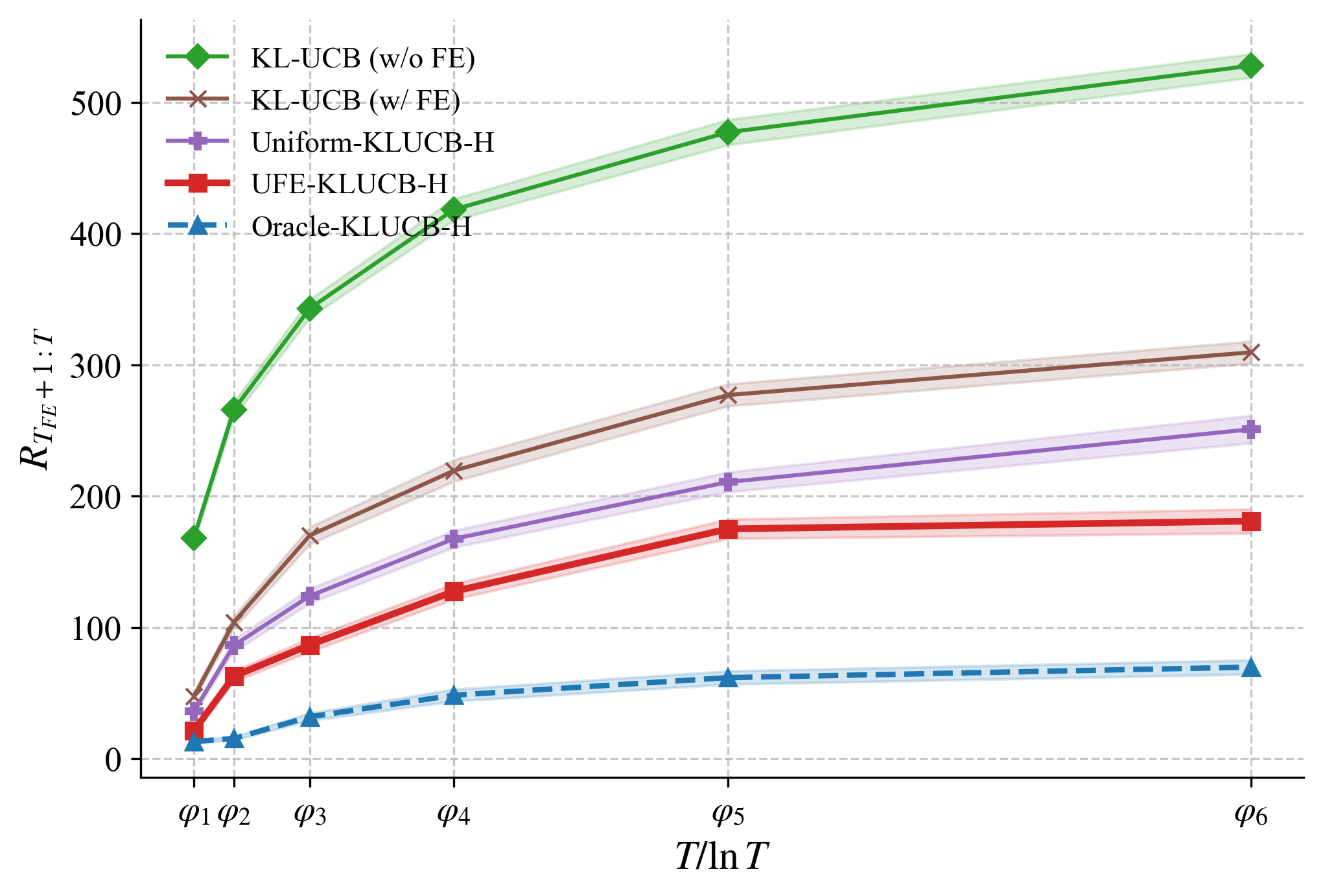}
    \caption{The regret of different methods in Exp 1, evaluated at horizons $x_i = T_0 \times 2^{i-1}$ for $i \in [6]$ where $T_0 = 40000$ and $\varphi_i = x_i / \ln x_i$. \ouralg{} performs well compared to the baselines.}
    \label{fig:exp1}
\end{figure}
Firstly, the results showcase the benefit of the \ac{fe} phase: an \ac{fe} phase of $\Theta(\ln T)$ steps can reduce the regret even if we directly apply \textsc{KL-UCB}. 
Furthermore, if we only explore without any exploitation during the \ac{fe} phase via \textsc{US}, there can be more reduction in the regret.
Lastly, \ouralg{} can explore the arms more efficiently and adaptively, saving more regret compared to the baselines.

The results for Exp 2 are plotted in Figure~\ref{fig:exp2}.
It shows that \ouralg{} remains effective for distributions other than Bernoulli.
Moreover, the comparisons to the benchmark algorithms highlight that \textsc{Uniform-KLUCB-H} is not adaptive and its performance highly hinges on the instance: in Instance A, all suboptimal arms are the same so that \textsc{US} performs well, whereas in Instance B, it requires adaptivity where \textsc{US} is similar to \textsc{KL-UCB}(w/ \FE), and \ouralgfe{} adapts to the instance and saves more regret.
\begin{figure}[t]
    \centering
    \subfigure[Regret under Instance A.]{
        \includegraphics[width=0.47\textwidth]{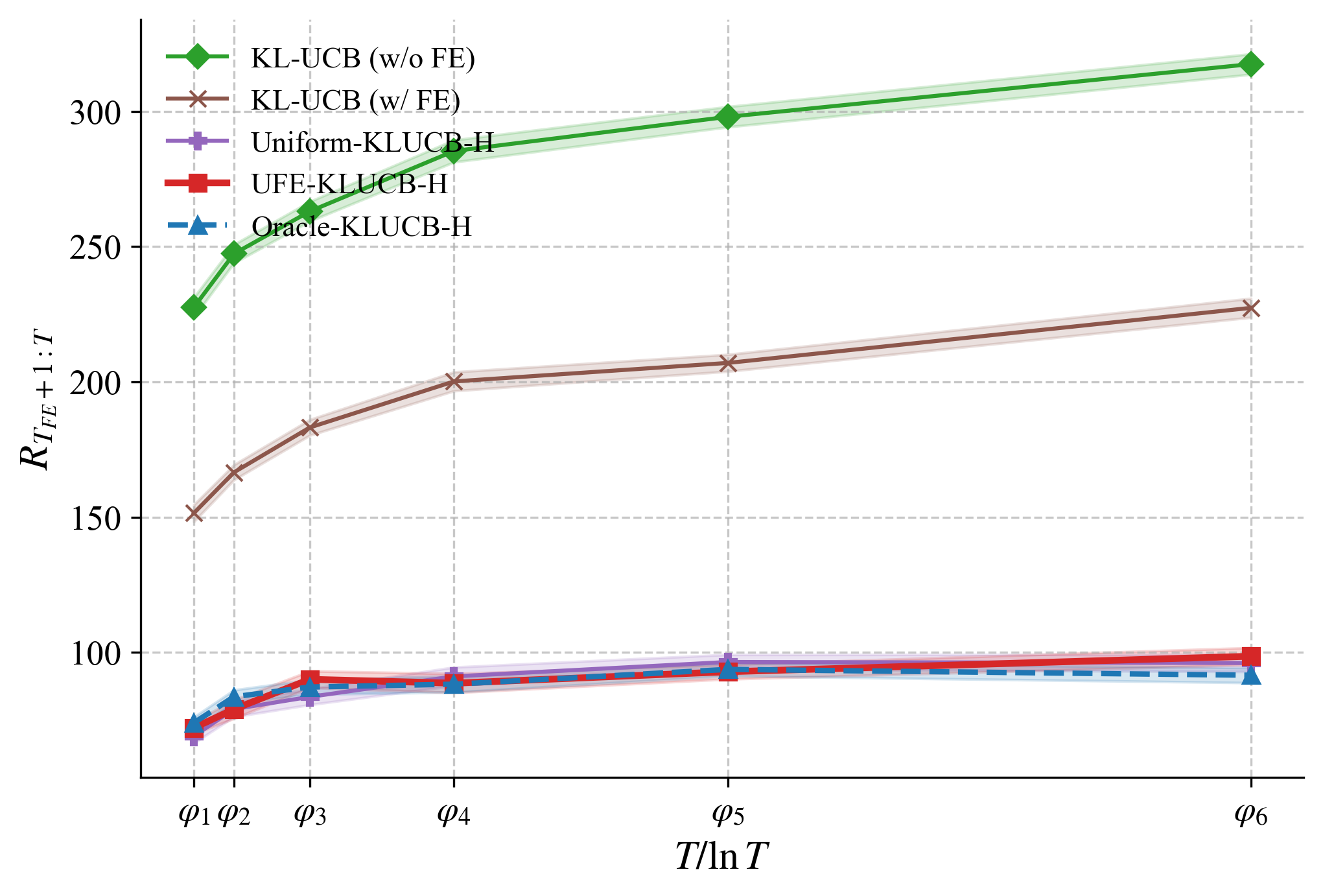}
        \label{fig:exp2_1}
    }
    \hfill
    \subfigure[Regret under Instance B.]{
        \includegraphics[width=0.47\textwidth]{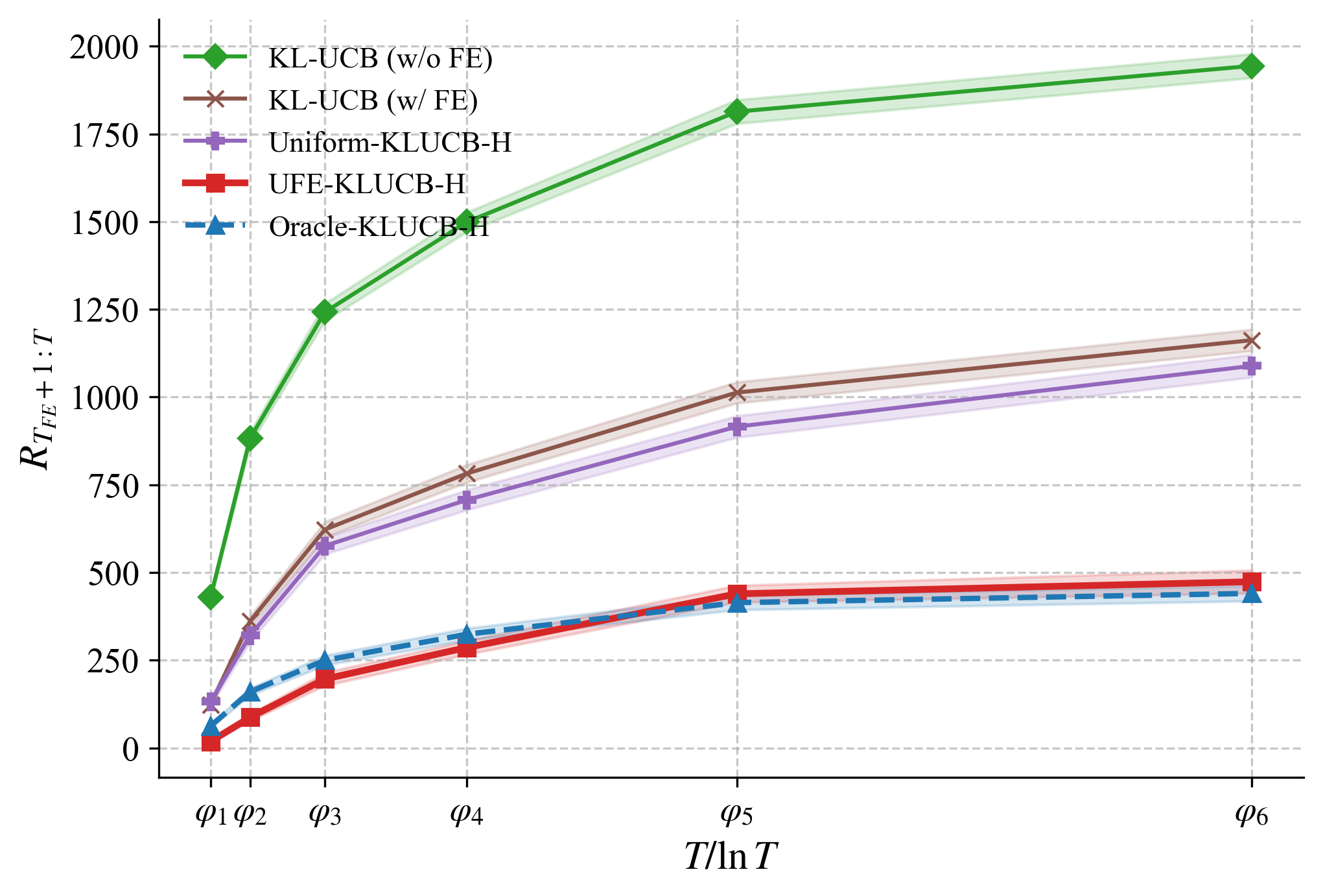}
        \label{fig:exp2_2}
    }
    \caption{The regret of different methods in Exp 2. The performance of \textsc{Uniform-KLUCB-H} depends on the instance, whereas \ouralg{} is more adaptive and saves more regret.}
    \label{fig:exp2}
\end{figure}
\section{Conclusions and Future Works}\label{app:concl}

\noindent
\textbf{Conclusions}
We formulate the \ac{rm} with \ac{fe} problem, motivated by practical scenarios (Section \ref{sec:examples}), and introduce a novel class of policies which we call {\em $(\alpha,\beta)$-probably-saving policies} (Section \ref{sec:alpha_beta}). We present \ouralg{} (Section \ref{sec:upbd}), which achieves substantial regret savings. We  derive matching instance-dependent upper and lower bounds  for specific instances  which can be computed explicitly in terms of the parameters of the instance (Section \ref{sec:lwbd_tightness}). As such, we can demonstrate that \ouralg{} is nearly asymptotically optimal for the class of two-valued instances (Section~\ref{sec:tightness}). This problem offers new insights into the exploration–exploitation trade-off in the bandit literature: with an additional exploration budget, an agent can devote more effort to exploitation during the \ac{ra} phase. 

\noindent
\textbf{Future Works}
In this work, \ouralg{} adopts uniform sampling among active items, which proves to be simple and effective. It contributes $O(\beta\ln T)$ pulls which is small, as $\beta$ is instantiated as a  small parameter. Nevertheless, it is still of interest whether we can improve the factor by designing better exploration policies when only rare information about the reward distributions is available. Promising algorithms include variants of LUCB~\citep{Kalyanakrishnan2012PAC}.

Furthermore,
as Thompson Sampling~\citep{agrawal2017near} performs well in practice and is also asymptotically optimal for certain instances (e.g., Gaussian), it would be interesting to see whether the history-aware version of Thompson Sampling can be derived and how the regret is reduced with the assistance of the historical data.

Moreover,
we assume the distribution family of the reward is given in this work. While this is common in the literature~\citep{kaufmann2016complexity,garivier2011KLUCB,garivier2016optimal}, the non-parametric bandit problem may be more practical when such information is unavailable~\citep{Riou2020bandit}. The extension to the non-parametric bandits is a promising direction for future research.

We study $K$-armed bandits  in this work. It would be interesting to extend the \ac{fe} formulation to structured bandits, e.g., linear bandits~\citep{Abbasi2011improved}, Lipschitz bandits~\citep{magureanu2014lipschitz}, unimodal bandits~\citep{yu2011unimodal,Combes2014unimodal}, etc.
Under such setups, instead of sampling individual items, we expect that the agent can use the \ac{fe} budget to explore the \emph{structure} of the instance, which can help to   eliminate suboptimal items more efficiently. Additionally, the extension to the constrained setup is also of interest, where the agent needs to take the constraint into consideration, beyond the exploration-and-exploitation trade-off~\citep{Amani2019Linear,hou2024probably}. The balance among these factors  requires a more delicate algorithm design.

Finally, in this work, the reward distributions of the arms do not shift from the \ac{fe} phase to the \ac{ra} phase. In practice, there can be potential distribution shifts between these two phases. Because the \ac{fe} \& \ac{rm} setting has not been studied systematically, we first address the simpler no-shift case to build insights before tackling distribution shifts.
At a high level, the observations gathered during the \ac{fe} phase are most useful when there is no distribution shift between phases; our algorithm is most effective in that setting. As the shift increases, the value of \ac{fe} data diminishes. \citet{cheung2024leveraging} shows that FE data is unhelpful when the shift is unbounded or exceeds $2c\Delta_i$ for arm $i$ (for some $c>0$). Intuitively, if the relationship between arms in the \ac{fe} and \ac{ra} phases can be revealed or learned, then \ac{fe} information can be leveraged to reduce regret in the \ac{ra} phase. For example: (1) when the distribution shift is small, or (2) in switching bandits where arms evolve via a learnable transition matrix $P$ at certain change points, our formulation could be extended to incorporate that structure.





\appendix

\section{More Discussions}

\subsection{Discussions on $\alpha$-Saving Algorithms}\label{app:discussions}
As mentioned, since the \ac{fe} budget $\Tfe=O(\ln T)$, a na\"ive algorithm such as uniform sampling can save $\alpha\cdot\Saved^*(\nu,\Tfe)$ with constant probability.
However, when $T$ and $\Tfe$ grow, we expect the regret can be saved more reliably. Therefore, while this na\"ive approach can achieve \eqref{equ:alpha_saving}, it can still be improved.
This is illustrated in the following example. 
\begin{example}\label{exmp1}
    Consider a two arm Gaussian instance $\nu=\{\nu_i=\calN(\mu_i,1)\}_{i=1,2}$. We apply the non-asymptotic law of iterated logarithms (LIL)~\citep[Lemma 3]{jamiesonlilucb2014} to the sequence  $(X_{1,t}-X_{2,t})_{t\in\bbN}$ with
    $\sigma, \varepsilon$ fixed and $\delta$ sufficient small. Let $b_t = \Theta\big( \sqrt{t^{-1} \ln\ln t} \big)$ and $ \hat{\Delta}_t := \frac{1}{t}\sum_{s=1}^t ( X_{1,s}-X_{2,s} )$
    be the empirical gap.
    Then $\bbP\big(\forall\, t \in \bbN: |\hat{\Delta}_t-\Delta_{1,2}|\geq b_t \big)\leq 0.99.$ 
    We uniformly sample the two arms until $b_t\leq c|\hat{\Delta}_t|$ with $c=\frac{1-\sqrt{\alpha}}{1+\sqrt{\alpha}}$. Then, we sample the empirically better arm  $\frac{2}{(1+c)^2\hat{\Delta}_t^2}\ln T$ times or until the \ac{fe} budget is exhausted. With probability at least $0.99$, it holds that
    $$
        \alpha\cdot \frac{2}{\Delta_{1,2}^2}\ln T
        \leq 
        \frac{2}{(1+c)^2\hat{\Delta}_t^2}\ln T
        \leq 
        \frac{2}{\Delta_{1,2}^2}\ln T. 
    $$
    The expected saved regret is thus lower bounded by $0.99\alpha\cdot \Saved^*(\nu,\Tfe)$ (We can set $\delta$ or $c$ smaller to guarantee exactly $\alpha\cdot \Saved^*(\nu,\Tfe)$ is saved).
    Furthermore, it can be proved that $b_t\leq c|\hat{\Delta}_t| $ occurs at most $\tilde{O}(\frac{\sqrt{\alpha}}{1-\sqrt{\alpha}}\cdot\frac{1}{\Delta}\ln\frac{1}{\delta})$ times.
    We highlight that this stopping time is \emph{independent} of $T$.
    Therefore, by using uniform sampling with a \emph{finite} number of pulls, with an additional forced exploration, this simple algorithm guarantees that~\eqref{equ:alpha_saving} holds. But the guarantees on the empirical saved regret hold with constant probability, which does not vanish as $T\to\infty$.
\end{example}

\subsection{Comparison between \klucbh{} and Other Algorithms} \label{app:klucb_compare}

During the \ac{ra} phase, in contrast to  our \klucbh{} algorithm,
existing algorithms that are designed for \ac{rm} with {\em static} (not adaptively chosen) historical dataset can also be applied.
Table~\ref{tab:compare_rm_history} compares the performance of \klucbh{} with MIN-UCB~\citep{cheung2024leveraging}, HUCBV~\citep{shivaswamy2012multi} and OTO~\citep{flore2025balancing} algorithms.
\begin{table}[H]
\centering
\small
\setlength{\tabcolsep}{6pt}
\renewcommand{\arraystretch}{1.35}
\scalebox{0.77}{
\begin{tabular}{ll}
\toprule
\textbf{Algorithm} & \textbf{Upper bound} \\
\midrule

\klucbh{} (Theorem~\ref{thm:kulcbh})
&
$\displaystyle
\sum_{i:\Delta_i>0}
\Delta_i
\left(
\frac{\ln T}{\rmd(\mu_i,\mu_1)}
-
\bbE\!\left[\min\{t_i,T_{i,\Tfe}\}\right]
\right)
+
O(\ln\ln T)
$
\\
\midrule

HUCBV~\citep{shivaswamy2012multi}
&
$\displaystyle
\sum_{i:\Delta_i>0}
\Delta_i
\left(
\max\left\{
9.6
\left(
\frac{\sigma_i^2}{\Delta_i^2}
+
\frac{2}{\Delta_i}
\right)
\ln T
-
\bbE\!\left[T_{i,\Tfe}\right],
\,2
\right\}
+
O(1)
\right)
$
\\
\midrule

MIN-UCB~\citep{cheung2024leveraging}
&
$\displaystyle
\frac{\pi^2}{6}\Delta_{\max}
+
\sum_{i:\Delta_i>0}
\max\left\{
\frac{32\ln\!\left(4L(T-\Tfe)^4\right)}{\Delta_i}
-
\bbE\!\left[T_{i,\Tfe}\right]\Delta_i,
\,
\Delta_i
\right\}
$
\\
\midrule

OTO~\citep{flore2025balancing}
&
$\displaystyle
\sum_{i=1}^{K}
\Delta_i
\left(
\frac{4\log(K/\delta)}{\Delta_i^2}
-
\bbE\!\left[T_{i,\Tfe}\right]
\right)_{+}
+
\frac{12K\log(K/\delta)}{\alpha\beta}
+
K
$
\\

\bottomrule
\end{tabular}
}
\caption{Comparison between \klucbh{} and existing \ac{rm} algorithms with historical data.}
\label{tab:compare_rm_history}
\end{table}
Note that MIN-UCB algorithm originally allows distribution shift between the \ac{fe} and \ac{ra} phases. We adapt the bounds to our problem in Table~\ref{tab:compare_rm_history}.
When \klucbh{}, HUCBV, MIN-UCB and OTO are all accompanied with an instance-dependent bound,  \klucbh{}'s regret bound is the smallest one.

\section{Lower Bounds for the Two-item and General Cases}\label{app:more_lwbd}
First, we present the lower bound for the simplest two-item case in Appendix~\ref{app:two_arm_result}. Next, we extend the lower bound to the general case in Appendix~\ref{app:general_result}.

\subsection{Two-item Case}\label{app:two_arm_result}
We first consider the case where there are   two items with $\mu_1>\mu_2$.
For brevity, we instantiate the failure probability in Definition~\ref{def:absaving} to be $T^{-\beta}/2.4=O(T^{-\beta}).$
Given a two-item instance $\nu=\{\nu_i\}_{i=1,2}$, $\alpha\in(\frac{1}{2},1),\beta\in(0,1)$, construct instances $\nu^{(j)},j=1,2,$ satisfying
\begin{align}
    &\bbE_{\nu_i^{(1)}}[X]=
        \frac{\mu_1+\mu_2-\varepsilon_1}{2}\mathbbm{1}\{i=1\}
        +\frac{\mu_1+\mu_2+\varepsilon_1}{2}\mathbbm{1}\{i=2\},
        \\
    &\bbE_{\nu_i^{(2)}}[X]=
        (\mu_2-\varepsilon_2)\mathbbm{1}\{i=1\}+\mu_2\mathbbm{1}\{i=2\},
\end{align}
where $\varepsilon_1,\varepsilon_2>0$ and $\varepsilon_2$ satisfies
$\frac{\ln T}{\kl(\nu_1^{(2)},\nu_2^{(2)})}\geq \Tfe$. Define the \emph{switching time steps}
\begin{align}\label{equ:two_item_switch_time}
    &\tau_1:=\frac{\beta\ln T}{ \max\big\{\kl(\nu_1,\nu_1^{(1)}),\kl(\nu_2,\nu_2^{(1)})\big\}},
    \\
    &
    \tau_2:=\min\bigg\{\frac{\ln T}{\kl(\nu_2,\nu_1)}+\frac{\beta\ln T}{\kl(\nu_1,\nu_1^{(2)})}, \frac{\frac{\alpha}{1-\alpha}\ln T}{\kl(\nu_2,\nu_1)}\bigg\}.
\end{align}
\begin{theorem}\label{thm:lwbd_2arm}
    If algorithm $\pi$ is $(\alpha, \beta)$-probably-saving on the instances $\{\nu,\nu^{(i)},i=1,2\}$ with $\alpha>1/2$, then the saved regret $\widetilde{\Saved}(\pi;\nu,\Tfe)$ is upper bounded as
    \begin{align}
        &\Delta_2\cdot\min\bigg\{\frac{\ln T}{\kl(\nu_2,\nu_1)},\bbE[T_{2,\Tfe}]\bigg\}
        \leq
        \begin{cases}
            \mbox{N.A.},& \Tfe < \tau_1,
            \\
            \Delta_2\cdot \big(\Tfe-\frac{\beta \ln T}{\kl(\nu_1,\nu_1^{(2)})}\big),& \tau_1\leq \Tfe\leq \tau_2,
            \\
            \Delta_2\cdot \frac{\ln T}{\kl(\nu_2,\nu_1)},& \Tfe\geq \tau_2.
        \end{cases}
    \end{align}
\end{theorem}
The proof of Theorem~\ref{thm:lwbd_2arm} is postponed to Appendix~\ref{app:proof_lwbd_2arm}.
The upper bound shows different  behaviors across the three time intervals, similar to the two-valued case as in Theorem~\ref{thm:lwbd_2value}.

We now derive the tightness result for the two-item instance:
\begin{align}
    \mu_1 = \frac{1}{2}+\varepsilon 
    \quad\mbox{and}\quad
    \mu_2 = \frac{1}{2}.
\end{align}
We parameterize $\Tfe = \frac{\gamma\ln T}{\rmd(\mu_2,\mu_1)}$ by the single parameter $\gamma$.
\begin{corollary}\label{cor:tightness_two_item}
    Let $\pi=$\ \ouralg{} and $\pi^\prime$ be any consistent and $(\alpha,\beta)$-probably saving algorithm with $\alpha$  in \eqref{equ:tight_alpha_two_item}. Then, the gap between the upper and lower bounds is
   \begin{align}
        &\overline{\mathfrak{R}}(\pi,\nu)- \underline{\mathfrak{R}}(\pi^\prime,\nu)
        \\*
        &
        \le\left\{\! \begin{array}{cc}
               \frac{\Delta_{2}}{\rmd(\mu_2,\mu_1)}\big(((c+2)^2-\frac{1}{4})\beta+O(\varepsilon)\big),& 2 (c+2)^2\beta + O(\varepsilon) \leq \gamma \leq
            1 + (c+2)^2\beta + O(\varepsilon),\\
                 \frac{\Delta_2\cdot O(\varepsilon)}{\rmd(\mu_2,\mu_1)}, & \gamma> 1 + (c+2)^2\beta + O(\varepsilon)
        \end{array}  \right. .
    \end{align}
\end{corollary}
The proof of Corollary~\ref{cor:tightness_two_item} is in Appendix~\ref{app:tigntness_two_item}.
As $\beta$ is small, this gap, which is at most linear in $\beta$, is  also  small.

\subsection{General Case}\label{app:general_result}
Before we present the results for the general case, we introduce some notations. 
Given an instance  $\nu = \{\nu_i\}_{i\in[L]}$ and the \ac{fe} budget $\Tfe$, we denote $\barDelta$ as the average saved regret such that $\barDelta\cdot \Tfe = \Saved^*(\nu,\Tfe) $.
Let $\calL_g := \{i:\Delta_i<\barDelta-\frac{1-\alpha}{\alpha}\Delta_L\}$ and we will choose a subset of items $\calL_{\rmc}$ to construct the alternative instance from
\begin{align}
    &\calS_c:=\Big\{
        \calL_{\rmc}\subset  [i_{FE}-1]: \alpha \geq \frac{\Delta_{L}}{\Delta_{L}+\barDelta-\Delta_{\calL_{\rmc}}},
            \Delta_{\calL_{\rmc}}<\alpha\cdot\barDelta,\;
            \frac{|\calL_{\rmc}|\cdot \ln T}{\kl(\nu_L,\nu_{i^\dagger})} > \Tfe
    \Big\},
\end{align}
where $\Delta_{\calL_{\rmc}}:=\max_{i\in\calL_{\rmc}}\Delta_{i}$ and 
    $\idagger:=\min_{i\notin\calL_{\rmc}}\; i$.
Given $\calL_{\rmc}\in\calS_c$, let $\nu_\varepsilon$ be the distribution with mean $\mu_L-\varepsilon$ and we define the permissible $\varepsilon$ that will be used to determine the mean gap in the alternative instance
\begin{align}
    &\calS_{\calL_{\rmc}}
    :=
    \Big\{\varepsilon>\varepsilon_{\calL_{\rmc}}: \frac{|\calL_{\rmc}|\cdot\ln T}{\kl(\nu_\varepsilon,\nu_{i^\dagger})} \geq \Tfe\Big\}
    \quad \mbox{where}\quad \varepsilon_{\calL_{\rmc}}:=\frac{(1-\alpha)\Delta_{\idagger,L}\cdot(\Delta_{L}-\Delta_{\calL_{\rmc}})}{\alpha(\barDelta+\Delta_{L}-\Delta_{\calL_{\rmc}})-\Delta_{1,L}}.
\end{align}
When $\calS_c$ or $\calS_{\calL_{\rmc}}$ is empty, the bound on the saved regret is vacuous, i.e., there is no alternative instance that can confuse the agent.
When both are nonempty, we construct instances $\nu^{(j)},j=1,2$ as
\begin{align}
    &\bbE_{\nu_i^{(1)}}[X]=
        \frac{\mu_1+\mu_L-\varepsilon_1}{2}\mathbbm{1}\{i\in\calL_g\}+\frac{\mu_1+\mu_L+\varepsilon_1}{2}\mathbbm{1}\{i\notin\calL_g\},
        \\*
    &\bbE_{\nu_i^{(2)}}[X]=(\mu_L-\varepsilon)\mathbbm{1}\{i\in\calL_{\rmc}\}
            +\mu_i\mathbbm{1}\{i\notin\calL_{\rmc}\},
\end{align}
where $\varepsilon,\varepsilon_1$ satisfy $\frac{|\calL_g|\cdot \ln T}{\kl(\nu_i^{(1)},\nu_{L}^{(1)})} >\Tfe$ and $\calL_{\rmc}\in\calS_c$, $\varepsilon\in\calS_{\calL_{\rmc}}$ if they are not empty.
We illustrate the construction of $\nu_i^{(2)}$ in Figure~\ref{fig:lwbd_general}.

We   group the good items in several non-overlapping sets, as shown in Figure~\ref{fig:lwbd_general}. Specifically, let $i_c=\max_{i\in\calL_{\rmc},\calL_{\rmc}\in\calS_c} i$ and let $k\in\bbN$ be the minimum integer such that $\{i_c-k+i:i\in[k]\}\in\calS_c$ and let $J\in\bbN$ be the integer such that $i_c-Jk\leq 1$ and $i_c-(J-1)k+1\geq 2$. We construct good item sets as follows, each containing consecutive items:
\begin{align}\label{equ:lwbd_general_partition}
    \calL_j:=\Big\{i_c-(j-1)k+i:i\in[k]\Big\}\quad\forall j\in[J].
\end{align}

\begin{figure}[t]
\centering

\subfigure[Construction of the alternative instance.]{
    \resizebox{0.48\textwidth}{!}{%
        \begin{tikzpicture}[>=stealth, scale=0.8]

    \tikzstyle{dot} = [circle, fill=natureVermillion, inner sep=1.2pt]
    \tikzstyle{groupbox} = [draw, dotted, thick, rectangle, minimum width=0.8cm, minimum height=0.8cm]

    \draw[->, thick] (-0.5, 0) -- (7, 0) node[below] {$i$};
    \draw[->, thick] (0, -0.2) -- (0, 4.5) node[above] {Mean};
    

    \draw[dashed] (0,4.0) -- (0.7, 4.0);
    \node[left] at (0.0, 4.0) {$\mu_1$};

    \node[dot] at (0.7, 4.0) {};
    \node[dot] at (1.4, 3.6) {};

    \node[dot] at (2.1, 3.2) {};
    \node[dot] at (2.8, 2.8) {};

    \node[groupbox] (L1) at (3.9, 2.3) {};
    \node[above=2pt] at (L1.north) {$\mathcal{L}_c$};
    \node[dot] at (3.5, 2.4) {};
    \node[dot] at (4.2, 2.0) {};

    \node[dot, yshift=-2pt,label=right:{$i_{\FE}$}] at (4.8, 1.5) {};

    \draw[thick, dashed] (0, 1.1) -- (6.5, 1.1);
    \node[left] at (0.0, 1.1) {$\mu_1 - \bar{\Delta}$};

    \draw[thick, dash dot] (0, 1.7) -- (6.5, 1.7);
    \node[left] at (0.0, 1.7) {$\mu_1-\barDelta+\frac{1-\alpha}{\alpha}\Delta_L$};

    \node[dot] at (5.6, 0.8) {};
    \node[dot, label=right:{$L$}] at (6.3, 0.6) {};

    \draw[red, thick, ->] (3.8,1.8) -- (3.8,0.5);
    \node[circle, fill=natureGreen, inner sep=1.2pt] at (3.5, 0.2) {};
    \node[circle, fill=natureGreen, inner sep=1.2pt] at (4.2, 0.2) {};

    \node[draw, dotted, thick, rectangle, minimum width=0.8cm, minimum height=0.3cm] (L1) at (3.9, 0.25) {};
    
    \draw[natureGreen, thick, <->] (3.2,0.6) -- (3.2,0.2) node[midway, left] {$\xi$};
    \draw[natureGreen, dashed] (3.2,0.6) -- (6.3, 0.6);
    \draw[natureGreen, dashed] (3.2,0.2) -- (3.5, 0.2);

\end{tikzpicture}
    }
}
\subfigure[Grouping of the good items.]{
    \resizebox{0.4\textwidth}{!}{%
        \begin{tikzpicture}[>=stealth, scale=0.8]

    \tikzstyle{dot} = [circle, fill=natureVermillion, inner sep=1.2pt]
    \tikzstyle{groupbox} = [draw, dotted, thick, rectangle, minimum width=0.8cm, minimum height=0.8cm]

    \draw[->, thick] (-0.5, 0) -- (7, 0) node[below] {$i$};
    \draw[->, thick] (0, -0.2) -- (0, 4.5) node[above] {Mean};
    

    \draw[dashed] (0,4.0) -- (0.7, 4.0);
    \node[left] at (0.0, 4.0) {$\mu_1$};
    
    \node[dot] at (0.7, 4.0) {};
    \node[dot] at (1.4, 3.6) {};

    \node[groupbox] (L2) at (2.4, 3.1) {};
    \node[above=2pt] at (L2.north) {$\mathcal{L}_2$};
    \node[dot] at (2.1, 3.2) {};
    \node[dot] at (2.8, 2.8) {};

    \node[groupbox] (L1) at (3.9, 2.2) {};
    \node[above=2pt] at (L1.north) {$\mathcal{L}_1$};
    \node[dot] at (3.5, 2.4) {};
    \node[dot, label=right:{$i_c$}] at (4.2, 2.0) {};


    \node[dot, label=right:{$i_{\FE}$}] at (4.9, 1.5) {};

    \draw[thick, dashed] (0, 1.1) -- (6.5, 1.1);
    \node[left] at (0.0, 1.1) {$\mu_1 - \bar{\Delta}$};

    \node[dot] at (5.6, 0.8) {};
    \node[dot, label=right:{$L$}] at (6.3, 0.6) {};

\end{tikzpicture}
    }
}

\caption{Illustration of the instance. (a): Those items above the dash dot line $\mu_1-(\barDelta-\frac{(1-\alpha)\Delta_L}{\alpha})$ are in $\calL_g$. We choose the permissible $\calL_{\rmc}\in\calS_c,\varepsilon\in\calS_{\calL_{\rmc}}$, and move them to $\mu_L-\varepsilon$ to construct the alternative instance $\nu^{(2)}$. (b): We group the good items in non-overlapping permissible groups with $k=2, J=2$, which allows us to obtain the lower bound in~\eqref{equ:lwbd_general}.}
\label{fig:lwbd_general}
\end{figure}
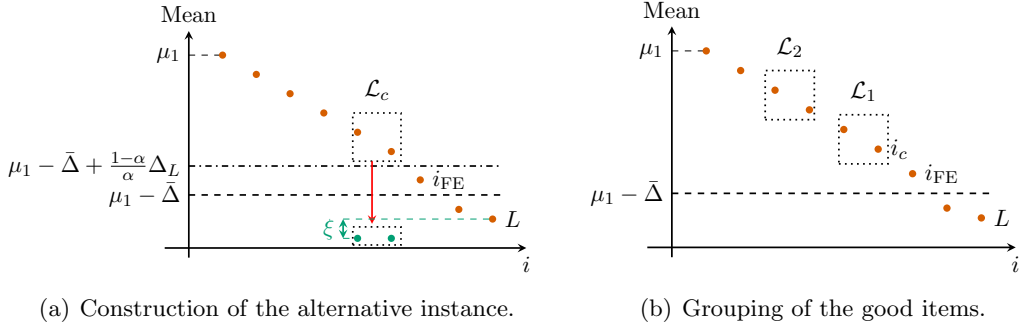

\begin{theorem}\label{thm:lwbd_general}
    If algorithm $\pi$ is $(\alpha,\beta)$-probably saving on instances $\nu$ and $\{ \nu^{(i)}\}_{  i=1,2} $, then
    \begin{itemize}[leftmargin=*]
    \item 
    if $\Tfe< \frac{\beta\ln T}{\max_i \kl(\nu_i,\nu_i^{(1)}) }$, the algorithm cannot be an $(\alpha, \beta)$-probably saving  algorithm for the given $\Tfe$. 
    \item 
    otherwise, $\widetilde{\Saved}(\pi;\nu,\Tfe)$ is upper bounded as
    \begin{align}\label{equ:lwbd_general}
             &\limsup_{T\to\infty} \frac{\widetilde{\Saved}(\pi;\nu,\Tfe)}{\ln T}
    \leq 
        \sum_{j\in[J]}\frac{\beta\cdot \max_{i\in\calL_j} \Delta_i}{\max_{i\in\calL_j} \kl(\nu_i,\nu_i^{(2)})}
        +
        \lim_{T\to\infty}\frac{\Saved^*(\nu,\Tfe^\prime)}{\ln T},
        \\
        &\quad \mbox{where}\quad
        \Tfe^\prime = \Tfe -  \sum_{j\in[J]}\frac{\beta\ln T}{\max_{i\in\calL_j} \kl(\nu_i,\nu_i^{(2)})}.
    \end{align}
    \end{itemize}
\end{theorem}
The proof of Theorem~\ref{thm:lwbd_general} is postponed to Appendix~\ref{app:lwbd_general_case}. 
The first term in \eqref{equ:lwbd_general} indicates the item pulls that need to be allocated to the good items, whereas the second term is the optimal saved regret with the remaining \ac{fe} budget.
The bound on the saved regret shows that
\begin{itemize}[leftmargin=*]
    \item When the \ac{fe} budget $\Tfe$ is small, any algorithm cannot be an $(\alpha, \beta)$-probably saving  algorithm.
    \item Otherwise, when $\Tfe$ is moderate and there exists permissible $\calL_{\rmc}\in\calS_c$ and $\varepsilon\in\calS_{\calL_{\rmc}}$, we show that approximately $\frac{\beta\ln T}{k\cdot \kl(\nu_i,\nu_i^{(2)}) }$ pulls are required by each good item $i$, where $k$ is the cardinality of the set we constructed in~\eqref{equ:lwbd_general_partition}.
    \item When $\Tfe$ is large enough, we have $\calS_c=\emptyset$ (as the last condition therein cannot be satisfied), all regret $\Saved^*(\nu,\Tfe)$ can be saved.
\end{itemize}
Therefore, the lower bound for the general case also exhibits different behaviors across three time regimes of $\Tfe$, similar to the two-valued case.
\section{Proof of the Upper Bounds}\label{app:proof_main}
\subsection{Proof of Theorem~\ref{thm:use}}\label{app:proof_thmuse}

Define the good event
\begin{align}
    \calG:= \bigcap_{i\in[L]}\bigcap_{k\in\bbN} \{ |\mu_{i}(k)-\mu_i|\geq d^{-k}\}.
\end{align}
By applying Gaussian concentration inequalities and a union bound, we can show that
\begin{align}
    \bbP(\calG) \geq 1 - L\delta_\beta.
\end{align}

\begin{lemma}[Guarantee of Eliminated Arms]\label{lem:elmn_arm}
    Conditioned on the good event $\calG$, if $i\in \calB(k)$ and it is pulled $\frac{1}{c^2}t(k,\frac{1}{T})$ times for some $k$, then $\Delta_{i}\geq (c-2)d^{-k}$ and $T_{i,\Tfe}\geq  \frac{\ln T}{\kl(\nu_i,\nu_1)}\cdot \frac{c^2}{(c-2)^2}\cdot \rho_{i,\min}(\frac{\Delta_i}{c-2})
        \geq
        \frac{\ln T}{\kl(\nu_i,\nu_1)}\cdot \frac{ \sigma_{\min}^2}{\sigma^4}.$
\end{lemma}
\begin{lemma}[Necessary Exploration of the Arms in $\calB(k)$]\label{lem:nec_USE}
    Conditioned on the good event $\calG$, if $i\in \calB(k)$, then $k\geq \log_d \frac{c-2}{\Delta_{i}}$.
\end{lemma}
\begin{lemma}[Sufficient Condition for the Arms in $\calB(k)$]\label{lem:suf_USE}
    Conditioned on the good event $\calG$, we have $i\in \calB(k)$ for some $k \leq \lceil\log_d \frac{c+2}{\Delta_{i}}\rceil.$
\end{lemma}
The proofs of the above lemmas are postponed to Appendix~\ref{app:tech_lem}.
For each arm $i\in[L]$, we recall 
\begin{align}
    &k_i^+ :=\left\lceil\log_d \frac{c+2}{\Delta_{i}} \right\rceil  ,\quad
    T_i^+ := \frac{\ln T}{\kl(\nu_i,\nu_1)}\cdot \frac{c^2}{(c-2)^2}\cdot \rho_{i,\max}\left(\frac{\Delta_i}{c-2}\right)
    \\
    &k_i^- :=\left\lfloor\log_d \frac{c-2}{\Delta_{i}}\right\rfloor,\quad
    T_i^- :=\frac{\ln T}{\kl(\nu_i,\nu_1)}\cdot \frac{c^2}{(c-2)^2}\cdot \rho_{i,\min}\left(\frac{\Delta_i}{c-2}\right).
\end{align}
In addition, let $k_i$ be the stopping time where arm $i$ is eliminated in phase $k_i$. By Lemmas~\ref{lem:nec_USE} and \ref{lem:suf_USE}, $k_i\in[k_i^-,k_i^+]$ almost surely, given the good event $\calG$.

Conditional on the good event $\calG$, if we denote $\kappa$ as the termination phase,
then the minimum saved regret is the solution to the following optimization problem:
\begin{align}
    &\min_{\kappa,T_{i,\Tfe},i\in[L]} \sum_{i:\Delta_i>0} \Delta_i\cdot\min\bigg\{T_{i,\Tfe},\frac{\ln T}{\kl(\nu_i,\nu_1)}\bigg\}
    \\
    &
    \begin{aligned}
        \quad \mbox{s.t.} \qquad
        T_i^- &\leq T_{i,\Tfe} \leq T_i^+,\quad k_i^+\leq \kappa
    \\
        T(\kappa-1) &\leq T_{i,\Tfe} \leq T_i^+,\quad k_i^-\leq \kappa<k_i^+
    \\
         T(\kappa-1) &\leq T_{i,\Tfe}\leq T(\kappa),\quad  \kappa < k_i^-
    \\
        \sum_{i\in[L]} T_{i,\Tfe} &=\Tfe
    \end{aligned}
\end{align}
where we recall $T(k)= t(k,\frac{\delta_\beta}{(k+1)^2})$ in Algorithm~\ref{algo:USE}.
If we fix $\kappa$, and let $\hat{R}(\kappa;\nu,T,\Tfe)$ be the minimum value to the following problem:
\begin{align}
    &\min_{T_{i,\Tfe},i\in[L]} \sum_{i:\Delta_i>0} \Delta_i\cdot\min\bigg\{T_{i,\Tfe},\frac{\ln T}{\kl(\nu_i,\nu_1)}\bigg\}
    \\
    &
    \begin{aligned}
        \quad \mbox{s.t.} \qquad
        T_i^- &\leq T_{i,\Tfe} \leq T_i^+,\quad k_i^+\leq \kappa
    \\
        T(\kappa-1) &\leq T_{i,\Tfe} \leq T_i^+,\quad k_i^-\leq \kappa<k_i^+
    \\
         T(\kappa-1) &\leq T_{i,\Tfe}\leq T(\kappa),\quad  \kappa < k_i^-
    \\
        \sum_{i\in[L]} T_{i,\Tfe} &=\Tfe
    \end{aligned}
\end{align}
The minimum saved regret is thus
$ 
    \min_\kappa \hat{R}(\kappa;\nu,T,\Tfe).
$

\begin{remark}
A more explicit but conservative lower bound on the saved regret can be obtained if we relax the above optimization problem further.
\begin{itemize}[leftmargin=*]
    \item for the (bad) arm $i\in\{i:k_i^+\leq \kappa\}$, it will be sampled at most $T_i^+$ times
    and saves regret at least
     $T_i^-\Delta_i=\frac{\Delta_i\cdot \ln T}{\kl(\nu_i,\nu_1)}\cdot \rho_{i,\min}(\frac{\Delta_i}{c-2})$.
     \item for the (medium) arm $i\in\{i:k_i^-\leq \kappa <k_i^+\}$, it could be in $\calB(k)$ for some $k\in[k_i^-,\kappa]$, or it has not been in any $\calB(k)$. To compute the least saved regret,
     these arms will be sampled at most $T_i^+$ times and saves regret at least $\Delta_i  T(\kappa-1)$.
     \item for the (good) arm $i\in\{i:\kappa \leq k_i^+\}$, it has been sampled at most $T(\kappa)$ and saves regret $\Delta_i T(\kappa-1)$.
\end{itemize}
Denote the maximum pulls up to phase $\kappa$ as
\begin{align}
    \calT(\kappa) &:= 
        \sum_{i\in\{i:k_i^+\leq \kappa\}}T_i^+
        +
        \sum_{i\in\{i:k_i^-\leq \kappa <k_i^+\}}T_i^+
        +
        \sum_{i\in\{i:\kappa \leq k_i^+\}}T(\kappa)
    \\
    &=\sum_{i\in\{i: k_i^-\leq \kappa\}}\frac{\ln T}{\kl(\nu_i,\nu_1)}\cdot \frac{(c+2)^2}{(c-2)^2}\cdot \rho_{i,\max}\bigg(\frac{\Delta_i}{c-2}\bigg)
    +
   \sum_{i\in\{i:\kappa \leq k_i^+\}}2d^{2\kappa}\sigma^2\ln\big(2\kappa^2 T^\beta\big)
\end{align}
and the minimum regret saved after phase $\kappa$ as
\begin{align}
    \tilde{R}(\kappa): = 
    \sum_{i\in\{i:k_i^+\leq \kappa\}}\frac{\ln T}{\kl(\nu_i,\nu_1)}\cdot\frac{(c+2)^2}{(c-2)^2}\cdot \rho_{i,\min}\bigg(\frac{\Delta_i}{c-2}\bigg)
    +
    \sum_{i\in\{i:\kappa \leq k_i^+\}}\Delta_{i}2d^{2\kappa}\sigma^2\ln\big(2\kappa^2 T^\beta\big).
\end{align}
As the arm pull count $\calT(\kappa)$ and the saved regret $\tilde{R}(\kappa)$ are increasing in $\kappa$, the smallest $\kappa$ yields the least saved regret, which is 
$ 
    \tilde{R}\big(\min_{ \kappa\in[T]: \calT(\kappa)>\Tfe } \kappa\big).
$

\end{remark}

\subsection{Proof of Theorem~\ref{thm:kulcbh}}\label{app:proof_klucbh}

We firstly decompose the expected item pulls as follows:
\begin{align}    \label{equ:upbd_decomp}
    \bbE[T_{i,T}-T_{i,\Tfe}]
    &=\bbE\bigg[ \sum_{t=\Tfe+1}^T\mathbbm{1}\{i_t=i\}\bigg]
    \\
    &\leq \bbE\bigg[ \sum_{t=\Tfe+1}^T\mathbbm{1}\{\mu_1>u_{1,t}\}\bigg]
    +\bbE\bigg[ \sum_{t=\Tfe+1}^T\mathbbm{1}\{i_t=i,\mu_1\leq u_{1,t}\}\bigg]
    \\*
    &\leq
    \sum_{t=\Tfe+1}^T \bbP( \mu_1>u_{1,t})
    +\bbE\bigg[ \sum_{s=T_{i,\Tfe}}^T \mathbbm{1}\{s\cdot \rmd^+(\hatmu_i^s,\mu_1)\leq \ln T+3 \ln\ln T\}\bigg],
\end{align}
where the last inequality adopts Lemma~\ref{lem:KLUCBH_lem1} and $\hatmu_i^s$ is the empirical mean of the first $s$ samples from item $i$.

The first term in \eqref{equ:upbd_decomp} can be bounded using~\citet[Theorem 10]{garivier2011KLUCB}
\begin{align}
    \sum_{t=\Tfe+1}^T \bbP( \mu_1>u_{1,t})
    \leq 
    \sum_{t=\Tfe+1}^T \frac{e\lceil(\ln t)^2 + 3 \ln t\cdot \ln\ln t\rceil}{t (\ln t)^3}
    \leq
    c_1 \cdot ( \ln\ln T -  \ln\ln \Tfe),
\end{align}
where $c_1$ is an absolute constant.

The second term in \eqref{equ:upbd_decomp} can be bounded as
\begin{align}
    &\bbE\bigg[ \sum_{s=T_{i,\Tfe}}^T \mathbbm{1}\{s\cdot \rmd^+(\hatmu_i^s,\mu_1)\leq \ln T+3 \ln\ln T\}\bigg]
    \\
    &\quad\leq
    \bbE\bigg[ \sum_{s=t_i}^T \mathbbm{1}\{s\cdot \rmd^+(\hatmu_i^s,\mu_1)\leq \ln T+3 \ln\ln T\}\bigg]
    \\
    &\qquad+\bbE\bigg[\mathbbm{1}\{T_{i,\Tfe}<t_i\} \sum_{s=T_{i,\Tfe}}^{t_i} \mathbbm{1}\{s\cdot \rmd^+(\hatmu_i^s,\mu_1)\leq \ln T+3 \ln\ln T\}\bigg]
    \\
    &\quad\leq
    \bbE\Big[\max\{t_i-T_{i,\Tfe},0\} \Big] + \frac{c_2(\varepsilon)}{T^{f(\varepsilon)}}
    \\
    &\quad=t_i-
    \bbE\Big[\min\{t_i,T_{i,\Tfe}\} \Big] + \frac{c_2(\varepsilon)}{T^{f(\varepsilon)}}
\end{align}
where the second last inequality uses~\citet[Lemma 8]{garivier2011KLUCB}.

Therefore, the number of pulls of item $i$ in the \ac{rm} phase can be bounded as
\begin{align}
    &\bbE[T_{i,T}-T_{i,\Tfe}]
    \\
    &\quad\leq  
        t_i-
        \bbE\Big[\min\{t_i,T_{i,\Tfe}\} \Big] + \frac{c_2(\varepsilon)}{T^{f(\varepsilon)}}
        +
        c_1 \cdot ( \ln\ln T -  \ln\ln \Tfe)
    \\
    &\quad =
        \frac{1+\varepsilon}{\rmd(\mu_i,\mu_1)} \Big(\ln T + 3\ln\ln T \Big) - \bbE\Big[\min\{t_i,T_{i,\Tfe}\} \Big] 
           + \frac{c_2(\varepsilon)}{T^{f(\varepsilon)}}
            +
            c_1 \cdot ( \ln\ln T -  \ln\ln \Tfe)
    \\
    &\quad =
    \frac{1+\varepsilon}{\rmd(\mu_i,\mu_1)} \ln T  - \bbE\Big[\min\{t_i,T_{i,\Tfe}\} \Big]  + O( \ln\ln T ).
\end{align}
By the regret decomposition lemma~\citep[Lemma 4.5]{lattimore2020bandit},  and taking $\varepsilon\to0$, we obtain
\begin{align}
    \limsup_{T\to\infty}  \frac{R_{\Tfe+1:T} (\pi;\nu )}{\ln T}\le\sum_{i:\Delta_i>0} 
          \Delta_{i}
             \Big(
               \frac{1}{\rmd(\mu_i,\mu_1)}
               -
                \liminf_{T\to\infty}\frac{\bbE\big[\min\{t_i,T_{i,\Tfe}\} \big]}{\ln T}
             \Big).
\end{align}
Lastly, by using the property of any $(\alpha,\beta)$-probably saving algorithm in the \ac{fe} phase, i.e.,
\begin{align}
        &\bbP\Big(
            \widehat{\Saved}(\pi;\nu,\Tfe)>\alpha\cdot \Saved^*(\nu,\Tfe)
        \Big)
        \geq
        1 - O(T^{-\beta}),
        \\
        &\mbox{where}\quad
        \widehat{\Saved}(\pi;\nu,\Tfe) = 
    \sum_{i:\Delta_{i}>0} \Delta_{i}\cdot \min\bigg\{ T_{i,\Tfe},\frac{\ln T}{\klnu(\mu_i,\mu_1)}\bigg\},
\end{align}
and $t_i\geq \frac{\ln T}{\klnu(\mu_i,\mu_1)}$,
we obtain 
\begin{align}
    \sum_{i:\Delta_i>0} \Delta_i\cdot \bbE\Big[\min\{t_i,T_{i,\Tfe}\} \Big]
    &\geq
    \sum_{i:\Delta_{i}>0} \Delta_{i}\cdot \bbE\bigg[\min\bigg\{ T_{i,\Tfe},\frac{\ln T}{\klnu(\mu_i,\mu_1)}\bigg\}\bigg]
    \\
    &
    \geq (1- O(T^{-\beta})) \cdot \alpha\cdot \Saved^*(\nu,\Tfe)
    \\
    &= \alpha\cdot \Saved^*(\nu,\Tfe) - O\bigg(\frac{\ln T}{T^{\beta}}\bigg).
\end{align}
So the regret is upper bounded as
\begin{align}
    R_{\Tfe+1:T} (\pi;\nu )
    \leq
    \sum_{i:\Delta_i>0} \Delta_i\cdot
    \frac{1}{\rmd(\mu_i,\mu_1)} \ln T
    -
    \alpha\cdot \Saved^*(\nu,\Tfe) 
    +
    O( \ln\ln T ), \quad\mbox{as}\quad T\to\infty.
\end{align}
In particular, for Bernoulli bandits, the regret is upper bounded as
\begin{align}
    \limsup_{T\to\infty}  \frac{R_{\Tfe+1:T} (\pi;\nu )}{\ln T}
    &\leq
    \sum_{i:\Delta_i>0}\frac{\Delta_i}{\rmd(\mu_i,\mu_1)}-
    \alpha\cdot \lim_{T\to\infty}\frac{\Saved^*(\nu,\Tfe)}{\ln T}
    \\
    &\leq \sum_{i:\Delta_i>0,i\leq \ife}\frac{\Delta_i}{\rmd(\mu_i,\mu_1)}
    +
    \sum_{i:\Delta_i>0,i>\ife}\frac{(1-\alpha)\Delta_i}{\rmd(\mu_i,\mu_1)}.
\end{align}

\section{Proof of the Lower Bounds}\label{app:lwbd_proof}
\subsection{Proof of Lemma~\ref{lem:lwbd_whole}}\label{app:lwbd_whole}
The proof extends the standard argument of~\citet[Theorem~16.2]{lattimore2020bandit} to account for the additional \ac{fe} phase.

Fix any suboptimal item $j$ and   any $\xi>0$. We construct the alternative instance as  
\begin{align}
\nu^{(j)} 
&= \bigl\{ \nu_i^{(j)} \bigr\}_{i\in[L]}
\quad\text{where}\quad
\bbE_{X\sim\nu_i^{(j)}}[X]=
\mu_i^{(j)} =
\begin{cases}
\mu_i, & i \neq j, \\
\mu_j^\prime, & i=j,
\end{cases}
\end{align}
where $\mu_j^\prime>\mu_1$ and $\kl(\nu_j,\nu_j^{(j)})\leq\kl(\nu_j,\nu_1)+\xi $.

By using the divergence decomposition lemma \citep[Lemma 15.1]{lattimore2020bandit} and Bretagnolle--Huber inequality~\citep[Theorem 14.2]{lattimore2020bandit}, we can show that for any  event $A\in\sigma(\calH_T)$,
\begin{align}
    \bbP_{\nu}(A) + \bbP_{\nu^{(j)}}(A^c)  
    \geq 
    \frac{1}{2}
        \exp
        \Big(
            -\bbE[T_{i,T}]\big(\kl(\nu_i,\nu_1)+\xi\big)
        \Big).
\end{align}
In particular, we set $A=\{T_{j,T}>\frac{T}{2}\}$ and we have
\begin{align}
    &R_{\Tfe+1:T} (\pi;\nu)+R_{\Tfe+1:T} (\pi;\nu^{(j)} )
    \\
    &\quad\geq 
        \frac{T-2\Tfe}{2} 
        \Big( 
            \bbP_{\nu}(A)\cdot\Delta_j + \bbP_{\nu^{(j)}}(A^c)\cdot (\mu_j^{j}-\mu_1)
        \Big)
    \\&
    \quad\geq
        \frac{T-2\Tfe}{4} 
        \exp
        \Big(
            -\bbE[T_{i,T}]\big(\kl(\nu_i,\nu_1)+\xi\big)
        \Big)
        \min\{\Delta_j,\mu_j^{j}-\mu_1\},
\end{align}
which further indicates that
\begin{align}
    \frac{\bbE[T_{i,T}]}{\ln T}
    \geq
    \frac{
        \ln\frac{T-\Tfe}{4}+\ln \min\{\Delta_j,\mu_j^{j}-\mu_1\}-\ln \big(R_{\Tfe+1:T} (\pi;\nu)+R_{\Tfe+1:T} (\pi;\nu^{(j)} )\big)
    }{\big(\kl(\nu_i,\nu_1)+\xi\big)\cdot \ln T}.
\end{align}
Because the algorithm $\pi$ is consistent and $\Tfe=o(T)$, we obtain 
\begin{align}
    \liminf_{T\to\infty}\frac{\bbE[T_{i,T}]}{\ln T}
    \geq
    \frac{1}{\kl(\nu_i,\nu_1)+\xi}.
\end{align}
As the above inequality holds for any $\xi>0$, we send $\xi\to0$ and we get the desired result
\begin{align}
    \liminf_{T\to\infty}\frac{\bbE[T_{i,T}]}{\ln T}
    \geq
    \frac{1}{\kl(\nu_i,\nu_1)}.
\end{align}
For any consistent algorithm, we have
\begin{align}
    &R_{\Tfe+1:T} (\pi;\nu ) 
    = 
        \bbE\bigg[
            \sum_{i:\Delta_i>0} (T_{i,T}-T_{i,\Tfe})\cdot \Delta_i
        \bigg]
    = 
        \bbE\bigg[
            \sum_{i:\Delta_i>0} (T_{i,T}-T_{i,\Tfe})_+\cdot \Delta_i
        \bigg].
\end{align}
 By Jensen's inequality applied to the convex function $(\cdot)_+:=\max\{\cdot,0\}$,
\begin{align}
    \liminf_{T\to\infty} \frac{R_{\Tfe+1:T} (\pi;\nu) }{\ln T}
    \geq 
        \sum_{i:\Delta_i>0}\Delta_i\bigg(\frac{1}{\kl(\nu_i,\nu_1)}
        -
        \limsup_{T\to\infty}
        \frac{\min\big\{\frac{\ln T}{\kl(\nu_i,\nu_1)},\bbE[T_{i,\Tfe}]\big\}}{\ln T}
        \bigg).
\end{align}

\subsection{Proof of Theorem~\ref{thm:lwbd_2arm}}\label{app:proof_lwbd_2arm}
We divide the whole proof into two cases: 
\begin{align}
{\bf Case~1}:    \Tfe\leq \frac{\ln T}{\kl(\nu_2,\nu_1)} \quad\mbox{and}\quad {\bf Case~2}:\frac{\ln T}{\kl(\nu_2,\nu_1)}<\Tfe<\frac{\alpha}{1-\alpha} \cdot \frac{\ln T}{\kl(\nu_2,\nu_1)}.   
\end{align}
We consider alternative instances for each case in which achieving an $\alpha$-fraction of the optimal saved regret is unlikely.
Specifically,
\begin{itemize}[leftmargin=1em,parsep=1pt,topsep=1pt]
    \item Alternative instance 1 is used to show when $\Tfe$ is too small, no algorithm can be $(\alpha,\beta)$-probably saving.
    \item Alternative instance 2 is used to show in the moderate $\Tfe$ budget regime, in order to sample the suboptimal item $2$ sufficiently many times, we need to sample item $1$ some time to identify the suboptimality of item $2$.
    \item Alternative instance 3 is used to show in the rich $\Tfe$ budget regime, we still need to sample the good item $1$ to identify item $2$ as suboptimal.
\end{itemize}
Recall that the failure probability in Definition~\ref{def:absaving} is instantiated to be $ T^{-\beta}/2.4.$
We show the constructed alternative instances in Figure~\ref{fig:lwbd_two_item}.
\begin{figure}[ht]
    \begin{center}
            \begin{tikzpicture}[>=stealth, scale=0.8]

    \definecolor{myOrange}{RGB}{204, 102, 0}
    \definecolor{myGreen}{RGB}{34, 139, 34}
    \definecolor{myBlue}{RGB}{0, 76, 153}

    \tikzstyle{dot} = [circle, fill, inner sep=1.2pt]

    \draw[-, thick] (0, 0) -- (8, 0);
    \draw[->, thick] (0.5, -0.2) -- (0.5, 4.5) node[above] {Mean};

    \draw[dashed] (0.5, 3.5) -- (7.5, 3.5); 
    \draw[dashed] (0.5, 2.5) -- (7.5, 2.5); 
    \draw[dashed] (0.5, 1.5) -- (7.5, 1.5); 

    \node[left] at (0.5, 3.5) {$\mu_1$};
    \node[left] at (0.5, 2.5) {$\frac{\mu_1+\mu_2}{2}$};
    \node[left] at (0.5, 1.5) {$\mu_2$};

    \draw[thick] (2.5, 0) -- (2.5, 4.2);
    \draw[thick] (4.8, 0) -- (4.8, 4.2);

    \node[natureVermillion, font=\large] at (1.5, 4.0) {$\nu$};
    \node[dot, natureVermillion, label={[natureVermillion]below:$\nu_1$}] at (1.0, 3.5) {};
    \node[dot, natureVermillion, label={[natureVermillion]below:$\nu_2$}] at (2.0, 1.5) {};

    \node[natureGreen, font=\large] at (3.6, 4.1) {$\nu^{(1)}$};
    \node[dot, natureGreen, label={[natureGreen]right:$\nu_2^{(1)}$}] (v21) at (3.8, 3.0) {};
    \node[dot, natureGreen, label={[natureGreen]right:$\nu_1^{(1)}$}] (v11) at (3.0, 2.0) {};
    
    \draw[<->,natureBlack] (3.8, 3.0) -- (3.8, 2.5) node[midway, left,natureBlack] {$\xi_1$};
    \draw[<->,natureBlack] (3.0, 2.5) -- (3.0, 2.0) node[midway,natureBlack, right] {};

    \node[natureOrange, font=\large] at (6.5, 4.1) {$\nu^{(j)},j=2,3$};
    \node[dot, natureOrange, label={[natureOrange]above:$\nu_2^{(j)}$}] at (6.5, 1.5) {};
    \node[dot, natureOrange, label={[natureOrange]right:$\nu_1^{(j)}$}] (v1j) at (5.5, 0.5) {};
    
    \draw[<->,natureBlack] (5.5, 1.5) -- (5.5, 0.5) node[midway, left,natureBlack] {$\xi_j$};

\end{tikzpicture}
        \caption{Illustration of the alternative instances for the two-item case.
        }
        \label{fig:lwbd_two_item}
    \end{center}
\end{figure}
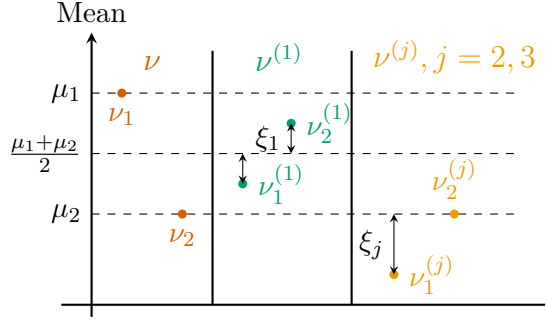

\paragraph{\textbf{\underline{Case 1}}:} $\Tfe\leq \frac{\ln T}{\kl(\nu_2,\nu_1)}.$
Let 
\begin{align}
    \calE_1 :=\bigg\{\min\Big\{T_{2,\Tfe},\frac{\ln T}{\kl(\nu_2,\nu_1)}\Big\}\geq \alpha\cdot \Tfe\bigg\}.
\end{align}

\paragraph{Alternative Instance $1$:}
We consider the alternative instance 
\begin{align}
\nu^{(1)} 
&= \big\{ \nu_i^{(1)} \big\}_{i=1,2} 
\quad\text{where}\quad \bbE_{X\sim \nu_i^{(1)}}[X]=
\mu_i^{(1)} =
\begin{cases}
\frac{\mu_1+\mu_2-\xi_1}{2}, & i = 1, \\
\frac{\mu_1+\mu_2+\xi_1}{2}, & i = 2,
\end{cases}
\end{align}
where $\xi_1>0$ satisfies $\frac{\ln T}{\kl(\nu_1^{(1)},\nu_2^{(1)})}\geq \Tfe$.
When $\alpha> 1/2$, we have
\begin{align}
    \bbP_{\nu}(\calE_1)\geq 1- \frac{T^{-\beta}}{2.4},\quad
    \bbP_{\nu^{(1)}}(\calE_1)\leq \frac{T^{-\beta}}{2.4}.
\end{align}
Therefore, by using~\citet[Lemma~1]{kaufmann2016complexity}, we have
\begin{align}\label{equ:lowbd_twoarm3}
    \Tfe \geq \frac{\beta\ln T}{ \max\big\{\kl(\nu_1,\nu_1^{(1)}),\kl(\nu_2,\nu_2^{(1)})\big\}} .
\end{align}
When $\Tfe$ is smaller than this lower bound, there does not exist any $(\alpha,\beta)$-probably saving algorithm.

\paragraph{Alternative Instance $2$:}
We consider the alternative instance 
\begin{align}
\nu^{(2)} 
&= \big\{ \nu_i^{(2)} \big\}_{i=1,2} 
\quad\text{where}\quad \bbE_{X\sim \nu_i^{(2)}}[X]=
\mu_i^{(2)} =
\begin{cases}
\mu_2 - \xi_2, & i = 1, \\
\mu_2, & i = 2.
\end{cases}
\end{align}
where we require $\xi_2>0$ to satisfy $\frac{\ln T}{\kl(\nu_1^{(2)},\nu_2^{(2)})}\geq \Tfe$.
Under the original instance $\nu$, we have
\begin{align}
    \bbP_\nu (\calE_1) 
    \geq 1 -\frac{T^{-\beta}}{2.4}.
\end{align}
Under the alternative instance $\nu^{(2)}$, we have 
\begin{align}
    \bbP_{\nu^{(2)}} (\calE_1) 
    = \bbP_{\nu^{(2)}} \big(T_{1,\Tfe}<(1-\alpha)\Tfe\big) 
    \leq 
    \bbP_{\nu^{(2)}} \big(T_{1,\Tfe}<\alpha\Tfe\big),
\end{align}
where we use 
$\alpha\geq 1/2$ in the second last inequality.
Therefore, by using~\citet[Lemma~1]{kaufmann2016complexity}
\begin{align}
    \bbE_{\nu}[T_{1,\Tfe}] \kl(\nu_1,\nu_1^{(2)}) \geq \rmd \left(1-\frac{T^{-\beta}}{2.4}, \frac{T^{-\beta}}{2.4} \right)
    \geq \beta\ln T,
\end{align}
which indicates
\begin{align}
    &\bbE_{\nu}[T_{1,\Tfe}]
    \geq 
     \frac{\beta \ln T}{\kl(\nu_1,\nu_1^{(2)})}
     \quad\mbox{and}\quad
     \bbE_{\nu}[T_{2,\Tfe}]
    \leq 
    \Tfe-\frac{\beta \ln T}{\kl(\nu_1,\nu_1^{(2)})}.
\end{align}

\paragraph{\textbf{\underline{Case 2}}:}  $\frac{\ln T}{\kl(\nu_2,\nu_1)}<\Tfe<\frac{\alpha}{1-\alpha} \frac{\ln T}{\kl(\nu_2,\nu_1)}.$
We then consider the following instance.

\paragraph{Alternative Instance $3$:} Consider,
\begin{align}
\nu^{(3)} 
&= \big\{ \nu_i^{(3)} \big\}_{i=1,2} 
\quad\text{where}\quad \bbE_{X\sim \nu_i^{(3)}}[X]=
\mu_i^{(3)} =
\begin{cases}
\mu_2 - \xi_3, & i = 1, \\
\mu_2,& i = 2,
\end{cases}
\end{align}
where we require $\xi_3>0$ to satisfy $\frac{\ln T}{\kl(\nu_1^{(3)},\nu_2^{(3)})}\geq \Tfe.$
Define 
\begin{align}
    \calE_2 :=\bigg\{\min\Big\{T_{2,\Tfe},\frac{\ln T}{\kl(\nu_2,\nu_1)}\Big\}\geq \alpha\cdot\frac{\ln T}{\kl(\nu_2,\nu_1)}\bigg\}.
\end{align}
Under $\nu^{(3)}$,
\begin{align}
    \widehat{\Saved}(\nu^{(3)},\Tfe) = \xi_3\cdot  \min\bigg\{ T_{1,\Tfe},\frac{\ln T}{\kl(\nu_1^{(3)},\nu_2^{(3)})}\bigg\}
    \quad
    \mbox{and}
    \quad
    \Saved^*(\nu^{(3)},\Tfe) = \Tfe\cdot\xi_3.
\end{align}
Under $\calE_2$, as $\Tfe< \frac{\alpha}{1-\alpha} \frac{1}{\kl(\nu_2,\nu_1)}\ln T$, this indicates $T_{1,\Tfe}<\Tfe-\alpha\cdot  \frac{1}{\kl(\nu_2,\nu_1)}\ln T<\alpha\cdot \Tfe$, and thus
\begin{align}
     \widehat{\Saved}(\nu^{(3)},\Tfe) <\alpha\cdot \Saved^*(\nu^{(3)},\Tfe).
\end{align}
Therefore,
\begin{align}
    \bbP_{\nu}(\calE_2)\geq 1- \frac{T^{-\beta}}{2.4}
    \quad\mbox{and}\quad
    \bbP_{\nu^{(3)}}(\calE_2)\leq \frac{T^{-\beta}}{2.4},
\end{align}
and we have 
\begin{align}
    &\bbE_{\nu}[T_{1,\Tfe}] \geq \frac{\beta \ln T}{\kl(\nu_1,\nu_1^{(3)})}
    \quad\mbox{and}\quad
    \bbE_{\nu}[T_{2,\Tfe}] \leq \Tfe-\frac{\beta \ln T}{\kl(\nu_1,\nu_1^{(3)})}.
\end{align}
The saved regret is thus
\begin{align}
    \Delta_2\!\cdot\!\min\bigg\{\frac{\ln T}{\kl(\nu_2,\nu_1)},\bbE[T_{2,\Tfe}]\bigg\}
    \!\leq
    \begin{cases}
        \Delta_2\Big(
                \Tfe-\frac{\beta\ln T}{\kl(\nu_1,\nu_1^{(3)})}
            \Big), & \!\!\Tfe\leq \frac{\ln T}{\kl(\nu_2,\nu_1)}+\frac{\beta\ln T}{\kl(\nu_1,\nu_1^{(3)})}, 
        \\
        \Delta_2  \frac{\ln T}{\kl(\nu_2,\nu_1)},& \!\!\mbox{otherwise}.
    \end{cases}
\end{align}
\newline
\textbf{In summary}, 
let 
\begin{align}\label{equ:lwbd_tau_two_item}
    &\tau_1:=\frac{\beta\ln T}{ \max\big\{\kl(\nu_1,\nu_1^{(1)}),\kl(\nu_2,\nu_2^{(1)})\big\}},
    \\
    &
    \tau_2:=\min\bigg\{\frac{\ln T}{\kl(\nu_2,\nu_1)}+\frac{\beta\ln T}{\kl(\nu_1,\nu_1^{(3)})},\frac{\frac{\alpha}{1-\alpha} \ln T}{\kl(\nu_2,\nu_1)}\bigg\}.
\end{align}
we have
\begin{align}
    \Delta_2\cdot\min\bigg\{\frac{\ln T}{\kl(\nu_2,\nu_1)},\bbE[T_{2,\Tfe}]\bigg\}
    \leq
    \begin{cases}
        \mbox{N.A.},& \Tfe < \tau_1,
        \\
        \Delta_2\cdot \big(\Tfe-\frac{\beta \ln T}{\kl(\nu_1,\nu_1^{(2)})}\big),& \tau_1\leq \Tfe\leq \frac{\ln T}{\kl(\nu_2,\nu_1)},
        \\
        \Delta_2\cdot \big(\Tfe-\frac{\beta \ln T}{\kl(\nu_1,\nu_1^{(3)})}\big),& \frac{\ln T}{\kl(\nu_2,\nu_1)}\leq \Tfe\leq \tau_2,
        \\
        \Delta_2\cdot \frac{\ln T}{\kl(\nu_2,\nu_1)},& \Tfe\geq \tau_2.
    \end{cases}
\end{align}

\subsection{Proof of Theorem~\ref{thm:lwbd_2value}}\label{app:proof_lwbd_2value}
Similar to the proof of Theorem~\ref{thm:lwbd_2arm}, construct several alternative instances:
\begin{itemize}[leftmargin=1em,parsep=1pt,topsep=1pt]
    \item Alternative Instance 1 is used to show that, when $\Tfe$ is too small, no algorithm can be $(\alpha,\beta)$-probably saving.
    \item In Alternative Instance 2, we modify the instance by altering $L_{\rmc}<L_{\rmg}$ good items and show that these items must be sampled sufficiently often to satisfy the $(\alpha,\beta)$-probably saving requirement. This instance requires more conditions than Alternative Instance 3, but yields a tighter bound as fewer items are changed.
    \item In Alternative Instance 3, we modify all $L_{\rmg}$ good items to construct the alternative instance, showing that these items also require a certain number of pulls to guarantee the $(\alpha,\beta)$-probably saving constraint.
    \item In Alternative Instance 4, we consider the case where a larger \ac{fe} budget is available, and show that the good items still require additional pulls.
\end{itemize}

\paragraph{\textbf{\underline{Case 1}}:}  $\Tfe\leq L_{\rmb}\cdot\frac{\ln T}{\kl(\nu_L,\nu_1)}$.
Let
\begin{align}
    \calE_1 :=\bigg\{
    \sum_{i\notin[L_{\rmg}]}\min\Big\{
        \frac{\ln T}{\kl(\nu_i,\nu_1)},T_{i,\Tfe}
        \Big\}
        \geq \alpha\cdot \Tfe
    \bigg\}.
\end{align}

\paragraph{Alternative Instance $1$:} Similar to Alternative Instance $2$ for the $2$-item case, there is a lower bound on $\Tfe$, below which the algorithm cannot be a  $(\alpha,\beta)$-probably saving algorithm.
Specifically, we construct the following alternative instance 
\begin{align}
\nu^{(1)} 
&= \big\{ \nu_i^{(1)} \big\}_{i\in[L]} 
\quad\text{where}\quad \bbE_{X\sim \nu_i^{(1)}}[X]=
\mu_i^{(1)} =
\begin{cases}
\frac{\mu_1+\mu_L-\xi_1}{2}, & i \in[L_{\rmg}] , \\
\frac{\mu_1+\mu_L+\xi_1}{2}, & i \in[L]\setminus[L_{\rmg}].
\end{cases}
\end{align}
When $\alpha> 1/2$, we have
\begin{align}
    \bbP_{\nu}(\calE_1)\geq 1- \frac{T^{-\beta}}{2.4},\quad
    \bbP_{\nu^{(1)}}(\calE_1)\leq \frac{T^{-\beta}}{2.4}.
\end{align}
Therefore, by using~\citet[Lemma~1]{kaufmann2016complexity},
\begin{align}
    \Tfe \cdot  \max\big\{\kl(\nu_1,\nu_1^{(1)}),\kl(\nu_L,\nu_L^{(1)})\big\} \geq \beta\ln T.
\end{align}
Thus,
\begin{align}
    \Tfe \geq  \frac{\beta\ln T}{\max\big\{\kl(\nu_1,\nu_1^{(1)}),\kl(\nu_L,\nu_L^{(1)})\big\}}.
\end{align}
When $\Tfe$ is below this threshold, the algorithm cannot be an $(\alpha,\beta)$-probably saving algorithm on instances $\nu$ and $\nu^{(1)}$.

\paragraph{Alternative Instance $2$:}
We construct an alternative instance
\begin{align}
\nu^{(2)} 
&= \bigl\{ \nu_i^{(2)} \bigr\}_{i\in[L]}
\quad\text{where}\quad
\bbE_{X\sim\nu_i^{(2)}}[X]=
\mu_i^{(2)} =
\begin{cases}
\mu_1, & i \in[L_{\rmg}^\prime], \\
\mu_L-\xi_2, & i \in [L_{\rmg}]\setminus[L_{\rmg}^\prime], \\
\mu_L, & i \in [L]\setminus[L_{\rmg}].
\end{cases}
\end{align}
where we denote $L_{\rmc} := L_{\rmg}-L_{\rmg}^\prime$ the number of changed items which is to be specified and $L_{\rmc}< L_{\rmg}$ (see Alternative Instance $3$ for $L_{\rmc}=L_{\rmg}$). 
We also require $ \frac{L_{\rmc}\cdot\ln T}{\kl(\nu_{L_{\rmg}}^{(2)},\nu_{1}^{(2)})} \geq \Tfe$.
Without loss of generality, we assume $i \in [L_{\rmg}]\setminus[L_{\rmg}^\prime]$ have the smallest $\bbE_{\nu}[T_{i,\Tfe}]$ among $i\in[L_{\rmg}].$
According to the definition of $(\alpha, \beta)$-probably saving algorithm, we have
\begin{align}
    \bbP(\calE_1)\geq 1- \frac{T^{-\beta}}{2.4}.
\end{align}
Under $\nu^{(2)}$, we have
$$ \Saved^*(\nu^{(2)},\Tfe) = \Tfe\cdot(\mu_1-\mu_L+\xi_2) $$
and $\calE_1$ indicates
\begin{align}
    \widehat{\Saved}(\nu^{(2)},\Tfe)\leq  (1-\alpha)\cdot \Tfe \cdot (\mu_1-\mu_L+\xi_2) + \alpha\cdot \Tfe\cdot (\mu_1-\mu_L),
\end{align}
where we make use of $\alpha> \frac{1}{2}$. 
We construct a contradiction by setting $\xi_2\geq \frac{1-\alpha}{2\alpha-1}\cdot(\mu_1-\mu_L)$ which implies
\begin{align}
    \widehat{\Saved}(\nu^{(2)},\Tfe)\leq \alpha\cdot \Saved^*(\nu^{(2)},\Tfe).
\end{align}
In this case, 
we have 
\begin{align}
    \bbE\bigg[\sum_{i=L_{\rmg}^\prime+1}^{L_{\rmg}} T_{i,\Tfe} \bigg] \cdot 
    \kl(\nu_{L_{\rmg}},\nu_{L_{\rmg}}^{(2)})
    \geq \beta \ln T.
\end{align}
Therefore,
\begin{align}
    \bbE\bigg[\sum_{i=1}^{L_{\rmg}} T_{i,\Tfe} \bigg]
    \geq 
    \frac{L_{\rmg}}{L_{\rmc}} \cdot \frac{\beta\ln T}{\kl(\nu_{L_{\rmg}},\nu_{L_{\rmg}}^{(2)})}.
\end{align}
In order to maximize $\frac{L_{\rmg}}{L_{\rmc}} $, subject to $\frac{L_{\rmc}\cdot\ln T}{\kl(\nu_{L_{\rmg}}^{(2)},\nu_{1}^{(2)})} \geq \Tfe$, it suffices to have 
\begin{align}\label{equ:lowbdreq}
    L_{\rmc}\geq \max\bigg\{\frac{\Tfe}{\ln T}\cdot \kl(\nu_{L_{\rmg}}^{(2)},\;\nu_{1}^{(2)}),1\bigg\}.
\end{align}
With this choice of $L_{\rmc}$,
\begin{align}\label{equ:lowbd3}
    \bbE\bigg[\sum_{i=1}^{L_{\rmg}} T_{i,\Tfe}\bigg ]
    \geq 
    \frac{L_{\rmg}}{\max\big\{\frac{\Tfe}{\ln T}\cdot \kl(\nu_{L_{\rmg}}^{(2)},\nu_{1}^{(2)}),1\big\}}
    \cdot
    \frac{\beta\ln T}{\kl(\nu_{L_{\rmg}},\nu_{L_{\rmg}}^{(2)})}.
\end{align}
The $ \frac{L_{\rmg}}{\max\big\{\frac{\Tfe}{\ln T}\cdot \kl(\nu_{L_{\rmg}}^{(2)},\nu_{1}^{(2)}),1\big\}}$ factor is interesting. If there are many good items (i.e., $L_{\rmg}$ is large), then it can be hard to find out the suboptimal items to save regret and the saved regret can be less.

The saved regret $\widetilde{\Saved}(\nu,\Tfe)$ is thus upper bounded by 
\begin{align}
    &\sum_{i\notin [L_{\rmg}]}
    \Delta_i\cdot \min\bigg\{\frac{\ln T}{\kl(\nu_i,\nu_1)},\bbE[T_{i,\Tfe}]\bigg\}
    \\
    &\quad\leq
    \Delta_L \cdot \bigg(
        \Tfe - 
        \frac{L_{\rmg}}{\max\big\{\frac{\Tfe}{\ln T}\cdot \kl(\nu_{L_{\rmg}}^{(2)},\nu_{1}^{(2)}),1\big\}}\cdot\frac{\beta\ln T}{\kl(\nu_{L_{\rmg}},\nu_{L_{\rmg}}^{(2)})}
    \bigg),
\end{align}
if $L_{\rmg}>L_{\rmc}\geq \max\big\{\frac{\Tfe}{\ln T}\cdot \kl(\nu_{L_{\rmg}}^{(2)},\nu_{1}^{(2)}),1\big\}$ can be satisfied.

\paragraph{Alternative Instance $3$:}
For the case where $L_{\rmc}=L_{\rmg}$ in the previous alternative instance, the best item changes and we discuss it in this alternative instance. We consider the alternative instance 
\begin{align}
\nu^{(3)} 
&= \big\{ \nu_i^{(3)} \big\}_{i=1,2} 
\quad\text{where}\quad \bbE_{X\sim \nu_i^{(3)}}[X]=
\mu_i^{(3)} =
\begin{cases}
\mu_L - \xi_3, & i \in[L_{\rmg}], \\
\mu_L, & i\notin [L_{\rmg}],
\end{cases}
\end{align}
where we require $\xi_3 $ to satisfy $ \frac{L_{\rmg}\cdot\ln T}{\kl(\nu_1^{(3)},\nu_L^{(3)})}\geq \Tfe$.
Therefore,
under $\nu^{(3)}$,
\begin{align}
    \Saved^*(\nu^{(3)},\Tfe) = \Tfe\cdot\xi_3.
\end{align}
Under $\calE_1$,
\begin{align}
    \widehat{\Saved}(\nu^{(3)},\Tfe) \leq \xi_3\cdot  (1-\alpha)\Tfe.
\end{align}
Thus, when $\calE_1$ holds and $\alpha> \frac{1}{2}$, we have $\widehat{\Saved}(\nu^{(3)},\Tfe) < \alpha\cdot \Saved^*(\nu^{(3)},\Tfe) $. Therefore, 
\begin{align}
    \bbP_{\nu}(\calE_1)\geq 1- \frac{T^{-\beta}}{2.4}
    \quad
    \mbox{and}
    \quad
    \bbP_{\nu^{(3)}}(\calE_1)\leq   \frac{T^{-\beta}}{2.4}.
\end{align}
Finally,
\begin{align}
    &\sum_{i\in[L_{\rmg}]}\bbE_{\nu}[T_{i,\Tfe}] \kl(\nu_i,\nu_i^{(3)}) \geq \rmd \bigg(1-\frac{T^{-\beta}}{2.4}, \frac{T^{-\beta}}{2.4}\bigg)
    \geq \beta\ln T,
    \\*
    \Rightarrow&
     \sum_{i\in[L_{\rmg}]}\bbE_{\nu}[T_{i,\Tfe}] \geq \frac{\beta\ln T}{\kl(\nu_1,\nu_1^{(3)})}.
\end{align}

\paragraph{\textbf{\underline{Case 2}}:}  $ \frac{L_{\rmb}\ln T}{\kl(\nu_L,\nu_1)} \leq \Tfe\leq \frac{\alpha}{1-\alpha}\frac{L_{\rmb}\ln T}{\kl(\nu_L,\nu_1)}$.

Akin to Alternative Instance 3 in the two-item case, the empirical saved regret will approach the optimal one. We then construct the following. 

\paragraph{Alternative Instance 4:}
\begin{align}
\nu^{(4)} 
&= \big\{ \nu_i^{(4)} \big\}_{i\in[L]} 
\quad\text{where}\quad \bbE_{X\sim \nu_i^{(4)}}[X]=
\mu_i^{(4)} =
\begin{cases}
\mu_L - \xi_4, & i\in[L_{\rmg}] , \\
\mu_L,& i\notin[L_{\rmg}],
\end{cases}
\end{align}
where we require $\xi_4$ to satisfy
\begin{align}
    \frac{L_{\rmg}\cdot\ln T}{\kl(\nu_1^{(4)},\nu_L^{(4)})}\geq \Tfe.
\end{align}
Under $\nu^{(4)}$,
\begin{align}
    &\Saved^*(\nu^{(4)},\Tfe) = \Tfe\cdot\xi_4.
\end{align}
Define the event 
\begin{align}
    \calE_2 :=\bigg\{\sum_{i=L_{\rmg}+1}^L \min\big\{T_{i,\Tfe},\frac{\ln T}{\kl(\nu_i,\nu_1)}\big\} >\alpha\cdot \frac{L_{\rmb}\cdot\ln T}{\kl(\nu_L,\nu_1)}  \bigg\}.
\end{align}
Under $\calE_2$,
\begin{align}
    &\widehat{\Saved}(\nu^{(4)},\Tfe) \leq \xi_4\cdot  \bigg(\Tfe-\alpha\cdot \frac{L_{\rmb}\cdot\ln T}{\kl(\nu_L,\nu_1)}\bigg).
\end{align}
As $\Tfe< \frac{\alpha}{1-\alpha} \frac{L_{\rmb}}{\kl(\nu_L,\nu_1)}\ln T$, this indicates 
\begin{align}
     \widehat{\Saved}(\nu^{(4)},\Tfe) <\alpha\cdot \Saved^*(\nu^{(4)},\Tfe).
\end{align}
Therefore,
\begin{align}
    \bbP_{\nu}(\calE_2)\geq 1- \frac{T^{-\beta}}{2.4}
    \quad\mbox{and}\quad
    \bbP_{\nu^{(4)}}(\calE_2)\leq \frac{T^{-\beta}}{2.4},
\end{align}
and we have 
\begin{align}
    \sum_{i\in[L_{\rmg}]}\bbE_{\nu}[T_{i,\Tfe}] \geq \frac{\beta\ln T}{\kl(\nu_1,\nu_1^{(4)})}
    \quad\mbox{and}\quad
    \sum_{i\notin[L_{\rmg}]}\bbE_{\nu}[T_{i,\Tfe}] \leq \Tfe-\frac{\beta\ln T}{\kl(\nu_1,\nu_1^{(4)})}.
\end{align}
The saved regret is thus
\begin{align}
    &\sum_{i\notin [L_{\rmg}]}
    \Delta_i\cdot \min\bigg\{\frac{\ln T}{\kl(\nu_i,\nu_1)},\bbE[T_{i,\Tfe}]\bigg\}
    \\*
    &\quad\leq
    \begin{cases}
        \Delta_L\cdot\Big(
                \Tfe-\frac{\beta\ln T}{\kl(\nu_1,\nu_1^{(4)})}
            \Big), & \Tfe\leq \frac{L_{\rmb} \ln T}{\kl(\nu_L,\nu_1)}  +\frac{\beta\ln T}{\kl(\nu_1,\nu_1^{(4)})},
        \\
        \Delta_L \cdot  \frac{L_{\rmb}\ln T}{\kl(\nu_L,\nu_1)} ,& \Tfe\geq \frac{L_{\rmb}\ln T}{\kl(\nu_L,\nu_1)}+\frac{\beta\ln T}{\kl(\nu_1,\nu_1^{(4)})}.
    \end{cases}
\end{align}
\textbf{In summary},
let 
\begin{align}\label{equ:lwbd_tau_two_value}
    &\tau_1:=\frac{\beta\ln T}{\max\big\{\kl(\nu_1,\nu_1^{(1)}),\kl(\nu_L,\nu_L^{(1)})\big\}},
    \\
    &
    \tau_2:=\min\bigg\{\frac{L_{\rmb}\ln T}{\kl(\nu_L,\nu_1)}+\frac{\beta\ln T}{\kl(\nu_1,\nu_1^{(4)})},\frac{\alpha}{1-\alpha} \frac{L_{\rmb}\ln T}{\kl(\nu_L,\nu_1)}\bigg\}.
\end{align}
we have
\begin{align}
    \sum_{i\notin [L_{\rmg}]}
    \Delta_i\cdot \min\bigg\{\frac{\ln T}{\kl(\nu_i,\nu_1)},\bbE[T_{i,\Tfe}]\bigg\}
    \leq
    \begin{cases}
        \mbox{N.A.},& \Tfe < \tau_1,
        \\
        \Delta_L\cdot\Big(
                \Tfe-\frac{\beta\ln T}{\kl(\nu_1,\nu_1^{(4)})}
            \Big),& \tau_1\leq \Tfe\leq \tau_2,
        \\
        \Delta_L \cdot  \frac{L_{\rmb}\ln T}{\kl(\nu_L,\nu_1)},& \Tfe\geq \tau_2.
    \end{cases}
\end{align}
Furthermore,  for the \ac{fe} budget $\Tfe\in[\tau_1, \frac{L_{\rmb}\cdot\ln T}{\kl(\nu_2,\nu_1)}]$, if $\nu^{(2)}$ exists,
\begin{align}
    &\sum_{i\notin [L_{\rmg}]}
    \Delta_i\cdot \min\bigg\{\frac{\ln T}{\kl(\nu_i,\nu_1)},\bbE[T_{i,\Tfe}]\bigg\}
    \\
    &\quad\leq 
     \Delta_L \cdot \bigg(
            \Tfe - 
            \frac{L_{\rmg}}{\max\Big\{\frac{\Tfe}{\ln T}\cdot \kl(\nu_{L_{\rmg}}^{(2)},\nu_{1}^{(2)}),1\Big\}}\cdot\frac{\beta\ln T}{\kl(\nu_{L_{\rmg}},\nu_{L_{\rmg}}^{(2)})}
        \bigg).
\end{align}
Note that instances $\nu^{(3)}$ and $\nu^{(4)}$ are constructed under the same parameterization and requirements. We merge them as $\nu^{(3)}$ in the main text.

\subsection{Proof of Theorem~\ref{thm:lwbd_general}}\label{app:lwbd_general_case}
This theorem is for the general instance. Similar to the two-item and two-valued case, we construct two alternative instances:
\begin{itemize}[leftmargin=1em,parsep=1pt,topsep=1pt]
    \item Alternative Instance 1 is used to show that, when $\Tfe$ is too small, no algorithm can be $(\alpha,\beta)$-probably saving.
    \item In Alternative Instance 2, we modify the instance by altering a subset item $\calL_{\rmc}$ and show that these items must be sampled sufficiently often to satisfy the $(\alpha,\beta)$-probably saving requirement. Some conditions are required for the existence of the alternative instance.
\end{itemize}
Given an instance  $\nu = \{\nu_i\}_{i\in[L]}$, the \ac{fe} budget $\Tfe$, we denote $\barDelta\in[\Delta_{\ife},\Delta_L]$ as the average saved regret such that $\barDelta\cdot \Tfe = \Saved^*(\nu,\Tfe) $.
Define event 
\begin{align}
    \calE:= \bigg\{\sum_{i:\Delta_{i}>0} \Delta_{i}\cdot \min\bigg\{ T_{i,\Tfe},\frac{1}{\kl(\nu_i,\nu_1)}\ln T\bigg\}
    \geq
    \alpha\cdot\Saved^*(\nu,\Tfe)\bigg\}.
\end{align}

\paragraph{Alternative Instance $1$:} Let $\calL_{\rmg} = \{i:\Delta_i<\barDelta-\frac{1-\alpha}{\alpha}\Delta_L\}$. $\calL_{\rmg}$ is nonempty when $\Tfe$ is not large or $\alpha$ is large.
We construct the following alternative instance
\begin{align}
\nu^{(1)} 
&= \big\{ \nu_i^{(1)} \big\}_{i\in[L]} 
\quad\text{where}\quad \bbE_{X\sim \nu_i^{(1)}}[X]=
\mu_i^{(1)} =
\begin{cases}
\frac{\mu_1+\mu_L-\xi_1}{2}, & i\in\calL_{\rmg}, \\
\frac{\mu_1+\mu_L+\xi_1}{2}, & i\notin\calL_{\rmg},
\end{cases}
\end{align}
where $\xi_1$ is chosen such that $\frac{|\calL_{\rmg}|\cdot \ln T}{\kl(\nu_i^{(1)},\nu_{L}^{(1)})} >\Tfe$.
Under instance $\nu^{(1)}$,
\begin{align}
    \whatSaved(\nu^{(1)},\Tfe) &= 
    \xi_1\cdot \sum_{i\in\calL_{\rmg}} \min\bigg\{\frac{\ln T}{\kl(\nu_i^{(1)},\nu_{L}^{(1)})}, T_{i,\Tfe}\bigg\}
    \quad\mbox{and}\quad
     \Saved^*(\nu^{(1)},\Tfe) = \Tfe\cdot\xi_1.
\end{align}
If $\whatSaved(\nu^{(1)},\Tfe)\geq \alpha\cdot \Saved^*(\nu^{(1)},\Tfe)$, we would have
\begin{align}
    \sum_{i\in\calL_{\rmg}} \min\bigg\{\frac{1}{\kl(\nu_i^{(1)},\nu_{L}^{(1)})}\ln T, T_{i,\Tfe}\bigg\}\geq \alpha\cdot \Tfe.
\end{align}
Thus, under the original instance $\nu$, 
\begin{align}
    \whatSaved(\nu,\Tfe) \leq
    \alpha\cdot \Tfe\cdot \bigg(\barDelta-\frac{1-\alpha}{\alpha}\Delta_L\bigg)+(1-\alpha)\Tfe\cdot \Delta_L
    \leq \alpha\barDelta\Tfe.
\end{align}
Therefore, if $\calE$ holds, then we cannot have $\whatSaved(\nu^{(1)},\Tfe)\geq \alpha\cdot \Saved^*(\nu^{(1)},\Tfe)$. We must have $\bbP_{\nu^{(1)}}(\calE)<\frac{T^{-\beta}}{2.4}$. 
Recall $\bbP_{\nu}(\calE)\geq 1 - \frac{T^{-\beta}}{2.4}$,
we finally obtain
\begin{align}
    &\Tfe\cdot \max_i \kl(\nu_i,\nu_i^{(1)})
    \geq 
    \sum_{i\in[L]}\bbE[T_{i,\Tfe}] \kl(\nu_i,\nu_i^{(1)})
    \geq
    \beta\ln T
    \\
    \Rightarrow
    &
    \Tfe\geq \frac{\beta\ln T}{\max_i \kl(\nu_i,\nu_i^{(1)}) }.
\end{align}
This is the minimum pulls such that any algorithm is an $(\alpha,\beta)$-probably saving algorithm.

\paragraph{Alternative Instance $2$:}
In the following, we will choose a subset of items $\calL_{\rmc}\subset [i_{FE}-1]$ from
\begin{align}
    &\calS_c:=\bigg\{
        \calL_{\rmc}\subset  [i_{FE}-1]: \alpha \geq \frac{\Delta_{1,L}}{\Delta_{1,L}+\barDelta-\Delta_{\calL_{\rmc}}},\;
        \Delta_{\calL_{\rmc}}<\alpha\cdot\barDelta,\;
        \frac{|\calL_{\rmc}|\cdot \ln T}{\kl(\nu_L,\nu_{\idagger})} > \Tfe
    \bigg\},
    \\
    &\quad
    \mbox{where }
    \Delta_{\calL_{\rmc}}:=\max_{i\in\calL_{\rmc}}\Delta_{i}
    \quad\mbox{and}\quad
    \idagger:=\min_{i\notin\calL_{\rmc}}\; i.
\end{align}
Given $\calL_{\rmc}\in\calS_c$, let $\nu_\xi$ be the distribution with mean $\mu_L-\xi$ and we define the permissible $\xi$
\begin{align}
    &\calS_{\calL_{\rmc}}
    :=
    \Big\{\xi: \frac{|\calL_{\rmc}|\cdot\ln T}{\kl(\nu_\xi,\nu_{\idagger})} \geq \Tfe,\xi>\xi_{\calL_{\rmc}}\Big\}
    \\
    &
    \mbox{where}\quad \xi_{\calL_{\rmc}}:=(1-\alpha)\Delta_{\idagger,L}\cdot\frac{\Delta_{1,L}-\Delta_{\calL_{\rmc}}}{\alpha(\barDelta+\Delta_{1,L}-\Delta_{\calL_{\rmc}})-\Delta_{1,L}}.
\end{align}
Given $\calL_{\rmc}\in \calS_c$ with non-empty $\calS_{\calL_{\rmc}}$, we will give a lower bound on the pulls of the items in $\calL_{\rmc}$.
We construct the following alternative instance with $\calL_{\rmc}$ and $\xi\in\calS_{\calL_{\rmc}}$:
\begin{align}
\nu^{(2)} 
&= \bigl\{ \nu_i^{(2)}  \bigr\}_{i\in[L]}
\quad\text{where}\quad
 \bbE_{X\sim \nu_i^{(2)}}[X]=
\mu_i^{(2)} =
\begin{cases}
\mu_L-\xi, & i \in \calL_{\rmc}, \\
\mu_i, & i \in [L]\setminus\calL_{\rmc},
\end{cases}
\end{align}
Under instance $\nu^{(2)}$ with \ac{fe} budget $\Tfe$, we will prove $\bbP_{\nu^{(2)}}(\calE)<\frac{T^{-\beta}}{2.4}$ by contradiction. 
By the choice of $\calL_{\rmc}$ and $\xi$, items in $\calL_{\rmc}$ are the worst items and $\idagger$ is the optimal item under $\nu^{(2)}$. Thus by the definition of $\Saved^*(\nu,\Tfe)$ in~\eqref{equ:oracle_regret}, we have
\begin{align}
    \Saved^*(\nu^{(2)},\Tfe) = \Tfe \cdot (\Delta_{\idagger,L}+\xi).
\end{align}
The empirical saved regret under instance $\nu^{(2)}$ is
\begin{align}
    \whatSaved(\nu^{(2)},\Tfe) &= 
    (\Delta_{\idagger,L}+\xi) \cdot \sum_{i\in\calL_{\rmc}} \min\bigg\{\frac{\ln T}{\kl(\nu_\xi,\nu_{\idagger})} ,T_{i,\Tfe}\bigg\}
    \\
    &\qquad+
     \sum_{i\notin\calL_{\rmc}}\Delta_{\idagger,i}\cdot \min\bigg\{\frac{\ln T}{\kl(\nu_i,\nu_{\idagger})}, T_{i,\Tfe}\bigg\}.
\end{align}
Denote
\begin{align}
    T_{\calL_{\rmc}}: = 
    \sum_{i\in\calL_{\rmc}} \min\bigg\{\frac{\ln T}{\min\{\kl(\nu_\xi,\nu_{\idagger}),\kl(\nu_i,\nu_1)\}}, T_{i,\Tfe}\bigg\}.
\end{align}
$\bullet$ If $\whatSaved(\nu^{(2)},\Tfe)\geq \alpha\cdot \Saved^*(\nu^{(2)},\Tfe)$, we would have
\begin{align}
    &
    \Delta_{\idagger,L}\cdot \sum_{i\notin\calL_{\rmc}} \min\bigg\{\frac{\ln T}{\kl(\nu_i,\nu_{\idagger})}, T_{i,\Tfe}\bigg\}
    +
    (\Delta_{\idagger,L}+\xi)\cdot T_{\calL_{\rmc}}
    \geq
    \alpha\cdot \Tfe \cdot(\Delta_{\idagger,L}+\xi).
\end{align}
Note the fact that
\begin{align}\label{equ:sumlesTfe}
    \sum_{i\notin\calL_{\rmc}} \min\bigg\{\frac{\ln T}{\kl(\nu_i,\nu_{\idagger})}, T_{i,\Tfe}\bigg\}
    +
    T_{\calL_{\rmc}}
    \leq
    \Tfe,
\end{align}
we thus have 
\begin{align}
    &\frac{\Delta_{\idagger,L} }{\Delta_{\idagger,L}+\xi}
    \Big(
        \Tfe - 
        T_{\calL_{\rmc}}
    \Big)
    +
    T_{\calL_{\rmc}}
    \geq
    \alpha\cdot \Tfe
    \\
    \Rightarrow\quad &
    T_{\calL_{\rmc}}
    \geq 
    \frac{\Delta_{\idagger,L}+\xi}{\xi} \Big(\alpha-  \frac{\Delta_{\idagger,L}}{\Delta_{\idagger,L}+\xi }\Big)\cdot \Tfe.
    \label{equ:contradiction_1}
\end{align}
$\bullet$ Under event $\calE$,
\begin{align}
    &
    \Delta_{\calL_{\rmc}}\cdot \sum_{i\in\calL_{\rmc}} \min\bigg\{\frac{\ln T}{\kl(\nu_{i},\nu_1)}, T_{i,\Tfe}\bigg\}
    +
    \Delta_{L}\cdot \sum_{i\notin\calL_{\rmc}} \min\bigg\{\frac{\ln T}{\kl(\nu_{i},\nu_1)}, T_{i,\Tfe}\bigg\}
    \geq
    \alpha\cdot \barDelta\cdot \Tfe
    \\
    \Rightarrow&
    \Delta_{\calL_{\rmc}}\cdot 
    T_{\calL_{\rmc}}
    +
    \Delta_{L}\cdot \sum_{i\notin\calL_{\rmc}} \min\bigg\{\frac{\ln T}{\kl(\nu_{i},\nu_1)}, T_{i,\Tfe}\bigg\}
    \geq
    \alpha\cdot \barDelta\cdot \Tfe.
\end{align}
By utilizing \eqref{equ:sumlesTfe},
we have
\begin{align}
    \Tfe \cdot (\Delta_{1,L}-\alpha\cdot \barDelta )
    \geq
    (\Delta_{1,L}-\Delta_{\calL_{\rmc}})\cdot T_{\calL_{\rmc}}.
\end{align}
Therefore,
\begin{align}\label{equ:contradiction_2}
    T_{\calL_{\rmc}}
    \leq 
    \Tfe \cdot \frac{\Delta_{1,L}-\alpha\cdot \barDelta}{ \Delta_{1,L}-\Delta_{\calL_{\rmc}}}.
\end{align}
We highlight that our choice of $\calL_{\rmc}$ ensures that $\alpha\barDelta>\Delta_{\calL_{\rmc}}$, so the right-hand-side of the above inequality is less than $\Tfe$.
\newline
$\bullet$ When $\alpha\geq \frac{\Delta_{1,L}}{\Delta_{1,L}+\barDelta-\Delta_{\calL_{\rmc}}}$, \eqref{equ:contradiction_1} and \eqref{equ:contradiction_2} lead to a contradiction. Specifically, by the choice of $\xi\in\calS_{\calL_{\rmc}}$.
\begin{align}
    &&\xi
    &\geq 
    \xi_{\calL_{\rmc}}
    \\
    \Leftrightarrow&&
    \alpha-\frac{\Delta_{1,L}-\alpha\barDelta}{\Delta_{1,L}-\Delta_{\calL_{\rmc}}}
    &\geq 
    (1-\alpha)\frac{\Delta_{\idagger,L}}{\xi}
    \\
    \Leftrightarrow&&
    \frac{\Delta_{\idagger,L}+\xi}{\xi } \Big(\alpha-  \frac{\Delta_{\idagger,L}}{\Delta_{\idagger,L}+\xi }\Big)\cdot \Tfe
    &\geq
    \Tfe \cdot \frac{\Delta_{1,L}-\alpha\cdot \barDelta}{ \Delta_{1,L}-\Delta_{\calL_{\rmc}}}
    .\label{equ:epsilonthr}
\end{align}
Therefore, it is established that
\begin{align}
    \calE \subset \{\whatSaved(\nu^{(2)},\Tfe)\geq \alpha\cdot \Saved^*(\nu^{(2)},\Tfe) \}^c.
\end{align}
As the algorithm is an $(\alpha, \beta)$-probably saving algorithm, we have 
\begin{align}
    \bbP_{\nu^{(2)}}(\calE)\leq \bbP_{\nu^{(2)}}(\whatSaved(\nu^{(2)},\Tfe)< \alpha\cdot \Saved^*(\nu^{(2)},\Tfe)) \leq \frac{T^{-\beta}}{2.4}.
\end{align}
By~\citet[Lemma~1]{kaufmann2016complexity},
\begin{align}
    &\sum_{i\in\calL_{\rmc}}\bbE[T_{i,\Tfe}] \cdot \kl(\nu_i,\nu_i^{(2)})
    \geq
    d(\bbP_{\nu}(\calE),\bbP_{\nu^{(2)}}(\calE))
    \geq
    \beta\ln T
    \\
    \Rightarrow&
    \sum_{i\in\calL_{\rmc}}\bbE[T_{i,\Tfe}]\geq \frac{\beta\ln T}{\max_{i\in\calL_{\rmc}} \kl(\nu_i,\nu_i^{(2)})}.
\end{align}
By the definition of  $\calS_{c}$ and $\calS_{\calL_{\rmc}}$, it can be observed that if there exists a set $\calL_{\rmc}$, such that $\calS_{c}$ and $\calS_{\calL_{\rmc}}$ are non-empty, then for all the items $j\leq \max_{i\in\calL_{\rmc}} i$, there exists some $\calL_{\rmc}^\prime$ such that $j\in \calL_{\rmc}^\prime$ and the corresponding $\calS_\xi$ is nonempty. In other words, we can give lower bounds for the items whose indices are below $\max_{i\in\calL_{\rmc}} i$.
In particular, we can construct $\calL_{\rmc}^\prime$ as $\calL_{\rmc}$ with only one item replaced by $j$.

By using the construction of the item sets in \eqref{equ:lwbd_general_partition} and Figure~\ref{fig:lwbd_general},
we obtain the upper bound on the saved regret $\widetilde{\Saved}(\pi;\nu,\Tfe)$:
\begin{align}
    &\limsup_{T\to\infty} \frac{\widetilde{\Saved}(\pi;\nu,\Tfe)}{\ln T}
    \leq 
        \sum_{j\in[J]}\frac{\beta\cdot \max_{i\in\calL_j} \Delta_i}{\max_{i\in\calL_j} \kl(\nu_i,\nu_i^{(2)})}
        +
        \lim_{T\to\infty}\frac{\Saved^*(\nu,\Tfe^\prime)}{\ln T},
        \\
    &\qquad \mbox{where}\quad
    \Tfe^\prime = \Tfe -  \sum_{j\in[J]}\frac{\beta\ln T}{\max_{i\in\calL_j} \kl(\nu_i,\nu_i^{(2)})}.
\end{align}
where the second term is the optimal saved regret with the remaining \ac{fe} budget, excluding those used on the good items.

\section{Proof of the Tightness Results}
We prove the tightness results in this section. We firstly prove the result for the simpler two-item case in Appendix~\ref{app:tigntness_two_item}, and prove corresponding result for the harder two-valued case in Appendix~\ref{app:tightness_two_value}.
\subsection{Proof of Corollary~\ref{cor:tightness_two_item}}\label{app:tigntness_two_item}
Under the two-item case, we consider
\begin{align}\label{equ:two_item_instance}
    \mu_1 = \frac{1}{2}+\varepsilon, 
    \quad\mbox{and}\quad
    \mu_2 = \frac{1}{2}.
\end{align}
We choose two alternative instances $\nu^{(1)} $ and $\nu^{(2)} $ with
\begin{align}
    \begin{cases}
        \mu_1^{(1)} = \frac{1}{2},\\
        \mu_2^{(1)} = \frac{1}{2}+\varepsilon,
    \end{cases}
    \quad\mbox{and}\quad
    \begin{cases}
        \mu_1^{(2)} = \frac{1}{2}-\varepsilon,\\
        \mu_2^{(2)} = \frac{1}{2}.
    \end{cases}
\end{align}
\textbf{Proof Sketch:} 
We only need to apply Theorems~\ref{thm:use} and~\ref{thm:kulcbh} to these instances to obtain the upper bound.
To apply Theorem~\ref{thm:lwbd_2arm} for upper bounding the saved regret and lower bounding the incurred regret, we need to verify that the alternative instance admits the same value of $\alpha$ as the original instance. This follows directly from the construction: the mean gaps remain unchanged across the three instances, and only the identity of the best item is altered.

\paragraph{Upper Bound:}  
We compute the least saved regret for the original instance $\nu$.
To do this, we firstly compute $T_i^-$ and $T_i^+$ in the optimization problem~\eqref{equ:opti_prob}.
Recall in the proof of Theorem~\ref{thm:use} in Appendix~\ref{app:proof_thmuse}, we showed that, conditional on the good event $\calG$, (1) the termination phase $k_2\in[k_2^-,k_2^+]$, (2) $d^{-k_2}\leq \frac{\Delta_2}{c-2}$, (3) $\hatmu_1(k_2)-\hatmu_2(k_2)>cd^{-k}$, (4) $|\hatmu_i(k_2)-\mu_i|\leq d^{-k_2},i=1,2$.
By using Pinsker's inequality~\citep{Friedrich2019High},
\begin{align}
    \klnu(\hatmu_2(k_2)),\hatmu_1(k_2)) \geq 2 (\hatmu_2(k_2))-\hatmu_1(k_2))^2.
\end{align}
By the Taylor expansion of $\kl(p,q)$ at $q=p$, 
\begin{align}\label{equ:taylor_kl}
    \klnu(p,q)=\frac{(p-q)^2}{2p(1-p)} + O((p-q)^3)
    \quad\mbox{and}\quad
    p(1-p) = \frac{1}{4}-\bigg(p-\frac{1}{2}\bigg)^2.
\end{align}
We obtain
\begin{align}
    \frac{\klnu(\mu_2,\mu_1)}{\klnu(\hatmu_2(k_2)),\hatmu_1(k_2))}
    &\leq 
    \frac{\klnu(\hatmu_2(k_2)-d^{-k_2},\hatmu_1(k_2)+d^{-k_2})}{\klnu(\hatmu_2(k_2)),\hatmu_1(k_2))}
    \\
    &\leq
    \frac{2+O(\varepsilon)}{2} \bigg(\frac{\hatmu_1(k_2)-\hatmu_2(k_2)+2d^{-k_2}}{\hatmu_1(k_2)-\hatmu_2(k_2)}\bigg)^2
    \\
    &= \big(1+O(\varepsilon) \big)\Big(\frac{c+2}{c}\Big)^2.
\end{align}
Therefore, we have
\begin{align}\label{equ:tight_rho_max}
     \rho_{2,\max}\Big(\frac{\Delta_2}{c-2}\Big)\leq \big(1+O(\varepsilon) \big)\Big(\frac{c+2}{c}\Big)^2.
\end{align}
Similarly,
\begin{align}
    \frac{\klnu(\mu_2,\mu_1)}{\klnu(\hatmu_2(k_2)),\hatmu_1(k_2))}
    &\geq 
    \frac{\klnu(\hatmu_2(k_2)+d^{-k_2},\hatmu_1(k_2)-d^{-k_2})}{\klnu(\hatmu_2(k_2)),\hatmu_1(k_2))}
    \\
    &\geq
    \frac{2}{2+O(\varepsilon)} \bigg(\frac{\hatmu_1(k_2)-\hatmu_2(k_2)-2d^{-k_2}}{\hatmu_1(k_2)-\hatmu_2(k_2)}\bigg)^2
    \\
    &= \big(1-O(\varepsilon) \big)\Big(\frac{c-2}{c}\Big)^2,
\end{align}
which gives
\begin{align}\label{equ:tight_rho_min}
    \rho_{2,\min}\Big(\frac{\Delta_2}{c-2}\Big)\geq \big(1-O(\varepsilon) \big)\Big(\frac{c-2}{c}\Big)^2.
\end{align}
Then the critical pull counts can be bounded as
\begin{align}
    &
    T_2^- \geq \frac{\ln T}{\kl(\nu_2,\nu_1) }\big(1-O(\varepsilon) \big)
    ,\quad
     t\Big(k_2^+,\frac{T^{-\beta}}{(k_2^++1)^2}\Big)   
        =
        \frac{1}{2}\Big(\frac{c+2}{\varepsilon}\Big)^2\Big(\beta\ln T + O\Big(\ln\frac{1}{\varepsilon}\Big)\Big),
    \\
    &
    T_2^+ \leq \frac{\ln T}{\kl(\nu_2,\nu_1) }\big(1+O(\varepsilon) \big)\Big(\frac{c+2}{c-2}\Big)^2
    ,\quad
    t\Big(k_2^-,\frac{T^{-\beta}}{(k_2^-+1)^2}\Big)
        = 
        \frac{1}{2}\Big(\frac{c-2}{\varepsilon}\Big)^2\Big(\beta\ln T + O\Big(\ln\frac{1}{\varepsilon}\Big)\Big)
    .
\end{align}
These estimates relaxes the constraints in the optimization problem~\eqref{equ:opti_prob}, thus, 
by plugging the above estimates into the optimization problem~\eqref{equ:opti_prob}, we obtain a valid lower bound on the minimum value:
\begin{align}
    &\widehat{\Saved}(\pi;\nu,\Tfe)
    \\
    &\quad\geq
    \begin{cases}
        \Delta_2\cdot \frac{\Tfe}{2}, &\Tfe< 2 t\big(k_2^+,\frac{T^{-\beta}}{(k_2^++1)^2}\big) , 
        \\
        \Delta_2\cdot \Big(\Tfe - t\big(k_2^+,\frac{T^{-\beta}}{(k_2^++1)^2}\big)\Big)  ,&
             2 t\big(k_2^+,\frac{T^{-\beta}}{(k_2^++1)^2}\big)  
            \leq
            \Tfe
            \leq
            T_2^- + t\big(k_2^+,\frac{T^{-\beta}}{(k_2^++1)^2}\big) ,
        \\
        \Delta_2\cdot T_2^-, & \Tfe>T_2^- + t\big(k_2^+,\frac{T^{-\beta}}{(k_2^++1)^2}\big) .
    \end{cases}
\end{align}
The corresponding $\alpha$ is 
\begin{align}\label{equ:tight_alpha_two_item}
    \alpha \geq
    \begin{cases}
        \frac{1}{2}, &\gamma < 2 (c+2)^2\beta + O(\varepsilon),  
        \\
        1-\frac{(c+2)^2\beta}{\gamma}+O(\varepsilon) ,&
             2 (c+2)^2\beta + O(\varepsilon) \leq \gamma \leq
            1 + (c+2)^2\beta + O(\varepsilon),
        \\
        1- O(\varepsilon) , & \gamma> 1 + (c+2)^2\beta + O(\varepsilon),
    \end{cases}
\end{align}
which is greater than $1/2$.
By adopting Theorem~\ref{thm:kulcbh} and $\Tfe = \frac{\gamma\ln T}{\kl(\nu_2,\nu_1)}$, 
\begin{align}\label{equ:tight_up}
    &\overline{\mathfrak{R}}(\pi,\nu)=\limsup_{T\to\infty}  \frac{R_{\Tfe+1:T} (\pi;\nu )}{\ln T}
    \\&
    \leq   
    \begin{cases}
        \frac{\Delta_{2}}{\rmd(\mu_2,\mu_1)}(1-\frac{\gamma}{2}), &\gamma < 2 (c+2)^2\beta + O(\varepsilon),  
        \\
        \frac{\Delta_{2}}{\rmd(\mu_2,\mu_1)}\big(1-\gamma+(c+2)^2\beta+O(\varepsilon)\big)  ,&
             2 (c+2)^2\beta + O(\varepsilon) \leq \gamma \leq
            1 + (c+2)^2\beta + O(\varepsilon),
        \\
        \frac{\Delta_{2}}{\rmd(\mu_2,\mu_1)}\cdot O(\varepsilon) , & \gamma> 1 + (c+2)^2\beta + O(\varepsilon).
    \end{cases}
\end{align}

For the two alternative instances, we can similarly show that \eqref{equ:tight_rho_max} and \eqref{equ:tight_rho_min} hold. Thus, the above bound on the saved regret still holds. In other words, if \ouralgfe{} is an $(\alpha,\beta)$-probably saving algorithm on $\nu$, then it also holds for $\nu^{(j)},j=1,2.$
Therefore, the prerequisites of Theorem~\ref{thm:lwbd_2arm} are satisfied and the theorem can be applied.

\paragraph{Lower Bound:}
By adopting \eqref{equ:taylor_kl}, we have 
\begin{align}
    \kl(P_1,Q_1) = 8\varepsilon^2+O(\varepsilon^3)
    \quad\mbox{and}\quad
    \kl(P_2,Q_2) = 2\varepsilon^2+O(\varepsilon^3)
\end{align}
for $(P_1,Q_1)=(\nu_1,\nu_1^{(2)})$  and $(P_2,Q_2)=(\nu_1,\nu_1^{(1)}),(\nu_2,\nu_2^{(1)}),(\nu_2,\nu_1).$
In this case, the switching times in \eqref{equ:two_item_switch_time} are
\begin{align}
    \tau_1 = \frac{\beta\ln T}{2\varepsilon^2+O(\varepsilon^3)}
    \quad\mbox{and}\quad
    \tau_2 = \min\bigg\{\frac{\ln T}{2\varepsilon^2+O(\varepsilon^3)}+\frac{\beta\ln T}{8\varepsilon^2+O(\varepsilon^3)}, \frac{\frac{\alpha}{1-\alpha}\ln T}{2\varepsilon^2+O(\varepsilon^3)}\bigg\}.
\end{align}
By Theorem~\ref{thm:lwbd_2arm}, the saved regret is upper bounded as
\begin{align}
    \widetilde{\Saved}(\pi;\nu,\Tfe)&=\Delta_2\cdot\min\bigg\{\frac{\ln T}{\kl(\nu_2,\nu_1)},\bbE[T_{2,\Tfe}]\bigg\}
    \\
    &
    \leq
    \begin{cases}
        \mbox{N.A.},& \Tfe < \tau_1,
        \\
        \Delta_2\cdot \Big(\Tfe-\frac{\beta\ln T}{8\varepsilon^2+O(\varepsilon^3)}\Big),& \tau_1\leq \Tfe\leq \tau_2,
        \\
        \Delta_2\cdot \frac{\beta\ln T}{\kl(\nu_2,\nu_1)},& \Tfe> \tau_2.
    \end{cases}
\end{align}
Therefore, by using Lemma~\ref{lem:lwbd_whole}, we obtain
\begin{align}\label{equ:tight_low}
    &\underline{\mathfrak{R}}(\pi,\nu)
    =
    \liminf_{T\to\infty} \frac{R_{\Tfe+1:T} (\pi;\nu ) }{\ln T}
    \\
    &\geq
    \begin{cases}
        \mbox{N.A.},& \gamma<\beta + O(\varepsilon),
        \\
        \frac{\Delta_2}{\rmd(\mu_2,\mu_1)}(1-\gamma+\frac{\beta}{4}+O(\varepsilon)),&
        \beta + O(\varepsilon)\leq \gamma\leq \min\big\{1+\frac{\beta}{4}+O(\varepsilon),\frac{\alpha}{1-\alpha}(1+O(\varepsilon))\big\},
        \\
        0,& \gamma> \min\big\{1+\frac{\beta}{4}+O(\varepsilon),\frac{\alpha}{1-\alpha}(1+O(\varepsilon))\big\}.
    \end{cases}
 \end{align}
 We compare the upper bound \eqref{equ:tight_up} and the lower bound \eqref{equ:tight_low} across different time intervals. 
\begin{itemize}[parsep=1em]
    \item When $2 (c+2)^2\beta + O(\varepsilon) \leq \gamma \leq
            1 + (c+2)^2\beta + O(\varepsilon)$,
            the upper bound of \ouralg{} is larger than the lower bound by $ \frac{\Delta_2}{\rmd(\mu_2,\mu_1)}((c+2)^2-1/4)\beta$.
            As we consider the case where $\beta$ is small, this gap is considered as small.
    \item When $\gamma>1 + (c+2)^2\beta + O(\varepsilon)$, \ouralg{} can save almost all regret up to $\frac{\Delta_{2}\ln T}{\rmd(\mu_2,\mu_1)}\cdot O(\varepsilon)$.
\end{itemize}

\subsection{Proof of Corollary~\ref{cor:tightness}}\label{app:tightness_two_value}
Recall the two-valued instance
\begin{align}
    \begin{cases}
        \mu_i = \frac{1}{2}+\varepsilon,& i\leq L_{\rmg} 
        \\
        \mu_i = \frac{1}{2},&i> L_{\rmg} 
    \end{cases}
\end{align}
where there are $L_{\rmb}:=L-L_{\rmg}$ bad items and $L_{\rmg}>L_{\rmb}$.
We construct the following three alternative instances $\nu^{(j)},j=1,2,3 $:
\begin{align}
    \begin{cases}
       \mu_i^{(1)} = \frac{1}{2},& i\leq L_{\rmg} 
        \\
        \mu_i^{(1)} = \frac{1}{2}+\varepsilon, &i> L_{\rmg} 
    \end{cases},
    \quad
    \begin{cases}
        \mu_i^{(2)} = \frac{1}{2}+\varepsilon,& i\leq L_{\rmg}^\prime 
        \\
        \mu_i^{(2)} = \frac{1}{2}-\varepsilon,& L_{\rmg}^\prime \leq i\leq L_{\rmg} 
        \\
        \mu_i^{(2)} = \frac{1}{2}, &i> L_{\rmg} 
    \end{cases},
    \;\mbox{and}\;
    \begin{cases}
        \mu_i^{(3)} = \frac{1}{2}-\varepsilon,& i\leq L_{\rmg}
        \\
        \mu_i^{(3)} = \frac{1}{2},& i> L_{\rmg} 
    \end{cases}
\end{align}
where $L_{\rmg}^\prime<L_{\rmg}$ is to be specified.

\paragraph{Proof Sketch:} 
For the upper bound, we can directly apply Theorems~\ref{thm:use} and \ref{thm:kulcbh}.
We also derive the upper bounds for the alternative instances, which show that there exists at least one $(\alpha,\beta)$-probably-saving algorithm, i.e., \ouralg{}, for instances considered. In this case, Theorem~\ref{thm:lwbd_2value} can be applied.
For the lower bound, we directly apply Theorem~\ref{thm:lwbd_2value}. 

\paragraph{Upper Bound:}
Similar to \eqref{equ:tight_rho_max} and \eqref{equ:tight_rho_min} in the two-item case,  we can get
\begin{align}\label{equ:tight_rho_two_value}
     \rho_{i,\max}\Big(\frac{\Delta_i}{c-2}\Big)\leq \big(1+O(\varepsilon) \big)\Big(\frac{c+2}{c}\Big)^2
        \quad\mbox{and}\quad
     \rho_{i,\min}\Big(\frac{\Delta_i}{c-2}\Big)\geq \big(1-O(\varepsilon) \big)\Big(\frac{c-2}{c}\Big)^2.
\end{align}
Then for $i>L_{\rmg}$,
\begin{align}
    &
    T_i^- \geq \frac{\ln T}{\kl(\nu_L,\nu_1) }\big(1-O(\varepsilon) \big)
    ,\;
    t\Big(k_i^-,\frac{T^{-\beta}}{(k_i^-+1)^2}\Big)
        = 
        \frac{1}{2}\Big(\frac{c-2}{\varepsilon}\Big)^2\Big(\beta\ln T + O\Big(\ln\frac{1}{\varepsilon}\Big)\Big)
    ,
    \\
    &
    T_i^+ \leq \frac{\ln T}{\kl(\nu_L,\nu_1) }\big(1+O(\varepsilon) \big)\Big(\frac{c+2}{c-2}\Big)^2
    ,\;
     t\Big(k_i^+,\frac{T^{-\beta}}{(k_i^++1)^2}\Big)   
        =
        \frac{1}{2}\Big(\frac{c+2}{\varepsilon}\Big)^2\Big(\beta\ln T + O\Big(\ln\frac{1}{\varepsilon}\Big)\Big).
\end{align}
We denote $\psi_L:=t(k_L^+,\frac{T^{-\beta}}{(k_L^++1)^2})$.
Plugging the above quantities in the optimization problem~\eqref{equ:opti_prob} relaxes the original constraints, which yields a lower bound on the empirical saved regret
\begin{align}
    &\widehat{\Saved}(\pi;\nu,\Tfe)
    \geq
    \begin{cases}
        \Delta_L\cdot \frac{L_{\rmb} \Tfe}{L}, &\Tfe< L \cdot \psi_L , 
        \\
        \Delta_L\cdot \Big(\Tfe -L_{\rmg}\cdot \psi_L\Big)  ,&
             L\cdot  \psi_L 
            \leq
            \Tfe
            \leq
            L_{\rmb}\cdot T_L^- + L_{\rmg}\cdot \psi_L ,
        \\
        \Delta_L\cdot L_{\rmb}\cdot T_L^-, & \Tfe>L_{\rmb}\cdot T_L^- + L_{\rmg}\cdot\psi_L .
    \end{cases}
\end{align}
Denote the switching points $\overline{r}_1 = \frac{L}{L_{\rmb}} (c+2)^2\beta + O(\varepsilon)$ and $\overline{r}_2 = 1 +  \frac{L_{\rmg}}{L_{\rmb}} (c+2)^2\beta+O(\varepsilon)$, and define the following three intervals of $\gamma$:  
\begin{align}
         \mathcal{R}_1   = [0,\overline{r}_1] ,\quad\mathcal{R}_2 = (\overline{r}_1, \overline{r}_2], \quad\mathcal{R}_3= (\overline{r}_2,\infty).
\end{align}
The corresponding $\alpha$ is 
\begin{align}\label{equ:tight_alpha_two_value}
    \alpha \geq
    \begin{cases}
        \frac{L_{\rmb}}{L}, &\gamma \in\calR_1
        \\
        1-\frac{L_{\rmg}\cdot (c+2)^2\beta}{L_{\rmb}\cdot\gamma}+O(\varepsilon) ,&
             \gamma \in\calR_2,
        \\
        1- O(\varepsilon) , & \gamma \in\calR_3.
    \end{cases}
\end{align}
which is greater than $L_{\rmb}/L$ for any $\gamma$ and is greater than $1/2$ when $\gamma>\frac{2L_{\rmg}(c+2)^2\beta}{L_{\rmb}}$.

By Theorem~\ref{thm:kulcbh} and $\Tfe =\gamma \frac{L_{\rmb}\cdot \ln T}{\kl(\nu_2,\nu_1)}$, 
\begin{align}\label{equ:tight_up_two_value}
    & \overline{\mathfrak{R}}(\pi,\nu)=\limsup_{T\to\infty}  \frac{R_{\Tfe+1:T} (\pi;\nu )}{\ln T}
    \leq   
    \begin{cases}
        \frac{L_{\rmb}\cdot\Delta_{L}}{\rmd(\mu_L,\mu_1)}(1-\frac{L_{\rmb}\cdot\gamma}{L}), &\gamma \in\calR_1,  
        \\
        \frac{L_{\rmb}\cdot \Delta_{L}}{\rmd(\mu_2,\mu_1)}\big( \overline{r}_2-\gamma\big)  ,&
             \gamma \in\calR_2,
        \\
        \frac{L_{\rmb}\cdot \Delta_{L}}{\rmd(\mu_2,\mu_1)}\cdot O(\varepsilon), & \gamma \in\calR_3.
    \end{cases}
\end{align}

\paragraph{$\alpha$ in Alternative Instances:}

In the following, we derive the $\alpha$ for the alternative instances. We show that given the error parameter $\beta$ and the \ac{fe} budget, the $\alpha$ under the alternative instance is greater than that of the original instance in \eqref{equ:tight_alpha_two_value}, so that the set of $(\alpha,\beta)$-probably saving algorithms (with $\alpha$ specified in~\eqref{equ:tight_alpha_two_value}) is nonempty in Theorem~\ref{thm:lwbd_2value}.

For instance $\nu^{(1)}$ and $\nu^{(3)}$, we can get a similar result of $\alpha$ as the original instance $\nu$.
Specifically, for both $\nu^{(1)}$ and $\nu^{(3)}$,
\begin{align}\label{equ:tight_alpha_two_value_13}
    \alpha \geq
    \begin{cases}
        \frac{L_{\rmg}}{L}, &\gamma < \frac{L}{L_{\rmb}} (c+2)^2\beta + O(\varepsilon),  
        \\
        1-\frac{(c+2)^2\beta}{\gamma}+O(\varepsilon) ,&
             \frac{L}{L_{\rmb}} (c+2)^2\beta  + O(\varepsilon) \leq \gamma \leq
            \frac{L_{\rmg}}{L_{\rmb}} + (c+2)^2\beta + O(\varepsilon),
        \\
        1- O(\varepsilon) , & \gamma> \frac{L_{\rmg}}{L_{\rmb}} + (c+2)^2\beta + O(\varepsilon).
    \end{cases}
\end{align}
For instance $\nu^{(2)}$, let $L_{\rmc} = L_{\rmg}-L_{\rmg}^\prime$.
We can similarly get for $i>L_{\rmg}$,
\begin{align}
    &T_i^- \geq \frac{\ln T}{\kl(\nu_L^{(2)},\nu_1^{(2)}) }\big(1-O(\varepsilon) \big)
    \quad\mbox{and}\quad
    T_i^+ \leq \frac{\ln T}{\kl(\nu_L^{(2)},\nu_1^{(2)}) }\big(1+O(\varepsilon) \big)\Big(\frac{c+2}{c-2}\Big)^2,
    \\
    &t\bigg(k_i^-,\frac{T^{-\beta}}{(k_i^-+1)^2}\bigg)
        = 
        \frac{1}{2}\Big(\frac{c-2}{\varepsilon}\Big)^2\Big(\beta\ln T + O\Big(\ln\frac{1}{\varepsilon}\Big)\Big)
    \quad\mbox{and}\\
     &t\bigg(k_i^+,\frac{T^{-\beta}}{(k_i^++1)^2}\bigg)   
        =
        \frac{1}{2}\Big(\frac{c+2}{\varepsilon}\Big)^2\Big(\beta\ln T + O\Big(\ln\frac{1}{\varepsilon}\Big)\Big).
\end{align}
For $L_{\rmg}^\prime<i\leq L_{\rmg}$,
\begin{align}
    &T_i^- \geq \frac{\ln T}{\kl(\nu_{L_{\rmg}}^{(2)},\nu_1^{(2)}) }\big(1-O(\varepsilon) \big)
    \quad\mbox{and}\quad
    T_i^+ \leq \frac{\ln T}{\kl(\nu_{L_{\rmg}}^{(2)},\nu_1^{(2)}) }\big(1+O(\varepsilon) \big)\Big(\frac{c+2}{c-2}\Big)^2,
    \\*
    &t\bigg(k_i^-,\frac{T^{-\beta}}{(k_i^-+1)^2}\bigg)
        = 
        \frac{1}{2}\Big(\frac{c-2}{2\varepsilon}\Big)^2\Big(\beta\ln T + O\Big(\ln\frac{1}{\varepsilon}\Big)\Big)
    \quad\mbox{and}\\
     &t\bigg(k_i^+,\frac{T^{-\beta}}{(k_i^++1)^2}\bigg)   
        =
        \frac{1}{2}\Big(\frac{c+2}{2\varepsilon}\Big)^2\Big(\beta\ln T + O\Big(\ln\frac{1}{\varepsilon}\Big)\Big).
\end{align}
To better illustrate the lower bound on the empirically saved regret, we denote
\begin{align}
    &\psi_{L_{\rmg}} := t\left(k_{L_{\rmg}}^+,\frac{T^{-\beta}}{(k_{L_{\rmg}}^+ +1)^2}\right), \quad &&\psi_{L}:=t\left(k_{L}^+,\frac{T^{-\beta}}{(k_{L}^+ +1)^2}\right) \\
        &s_1 = L\cdot \psi_{L_{\rmg}},
        &&s_2 = L_{\rmc}\cdot T_{L_{\rmg}}^- + (L-L_{\rmc}) \cdot \psi_{L_{\rmg}},\\
        &s_3 = L_{\rmc}\cdot T_{L_{\rmg}}^+ + (L-L_{\rmc}) \cdot \psi_{L_{\rmg}},\quad
        &&s_4 = L_{\rmc}\cdot T_{L_{\rmg}}^+ + (L-L_{\rmc})\cdot \psi_{L},\\*
        &s_5 = L_{\rmc}\cdot T_{L_{\rmg}}^+ + L_{\rmb}\cdot T_L^- + L_{\rmg}^\prime\cdot \psi_{L},
\end{align}
and we further denote $\calS_i=[s_{i-1},s_i]$ for $i\in[6]$ where $s_0:=1$ and $s_6:=T$.
By solving \eqref{equ:opti_prob} with the above estimates, we obtain
\begin{align}
    &\widehat{\Saved}(\pi;\nu^{(2)},\Tfe)
    \\&
    \geq
    \begin{cases}
        (\mu_{1}^{(2)}-\mu_{L_{\rmg}}^{(2)})\cdot \frac{L_{\rmc}\Tfe}{L}
        +
        (\mu_{1}^{(2)}-\mu_{L}^{(2)})\cdot \frac{L_{\rmb}\Tfe}{L},
            \quad&
            \Tfe\in\calS_1,
        \\
        (\mu_{1}^{(2)}-\mu_{L_{\rmg}}^{(2)})\cdot \Big(\Tfe - (L-L_{\rmc})\psi_{L_{\rmg}} \Big)
        +
        (\mu_{1}^{(2)}-\mu_{L}^{(2)})\cdot L_{\rmb} \cdot \psi_{L_{\rmg}})
        ,
            \quad&
            \Tfe\in\calS_2,
        \\
        (\mu_{1}^{(2)}-\mu_{L_{\rmg}}^{(2)})\cdot L_{\rmc}\cdot T_{L_{\rmg}}^- 
        +
        (\mu_{1}^{(2)}-\mu_{L}^{(2)})\cdot L_{\rmb} \cdot \psi_{L_{\rmg}}
        ,
            \quad&
            \Tfe\in\calS_3,
        \\
        (\mu_{1}^{(2)}-\mu_{L_{\rmg}}^{(2)})\cdot L_{\rmc}\cdot T_{L_{\rmg}}^- 
        +
        (\mu_{1}^{(2)}-\mu_{L}^{(2)})\cdot \frac{L_{\rmb}}{L-L_{\rmc}} \cdot 
        \Big(\Tfe-L_{\rmc}\cdot T_{L_{\rmg}}^+\Big)
        ,
            \quad&
            \Tfe\in\calS_4,
        \\
        (\mu_{1}^{(2)}-\mu_{L_{\rmg}}^{(2)})\cdot L_{\rmc}\cdot T_{L_{\rmg}}^- 
        +
        (\mu_{1}^{(2)}-\mu_{L}^{(2)})\cdot \Big(
            \Tfe- L_{\rmc}\cdot T_{L_{\rmg}}^+ - L_{\rmg}^\prime \cdot   \psi_{L}
            \Big)
        ,
            \quad&
            \Tfe\in\calS_5,
         \\
         (\mu_{1}^{(2)}-\mu_{L_{\rmg}}^{(2)})\cdot L_{\rmc}\cdot T_{L_{\rmg}}^- 
        +
        (\mu_{1}^{(2)}-\mu_{L}^{(2)})\cdot L_{\rmb} \cdot T_{L}^-
        ,
            \quad&
            \Tfe\in\calS_6.
    \end{cases}
\end{align}
which is a piecewise-linear and non-decreasing function (up to the $O(\varepsilon)$ terms), as demonstrated in Figure~\ref{fig:savedstar_ins2}.
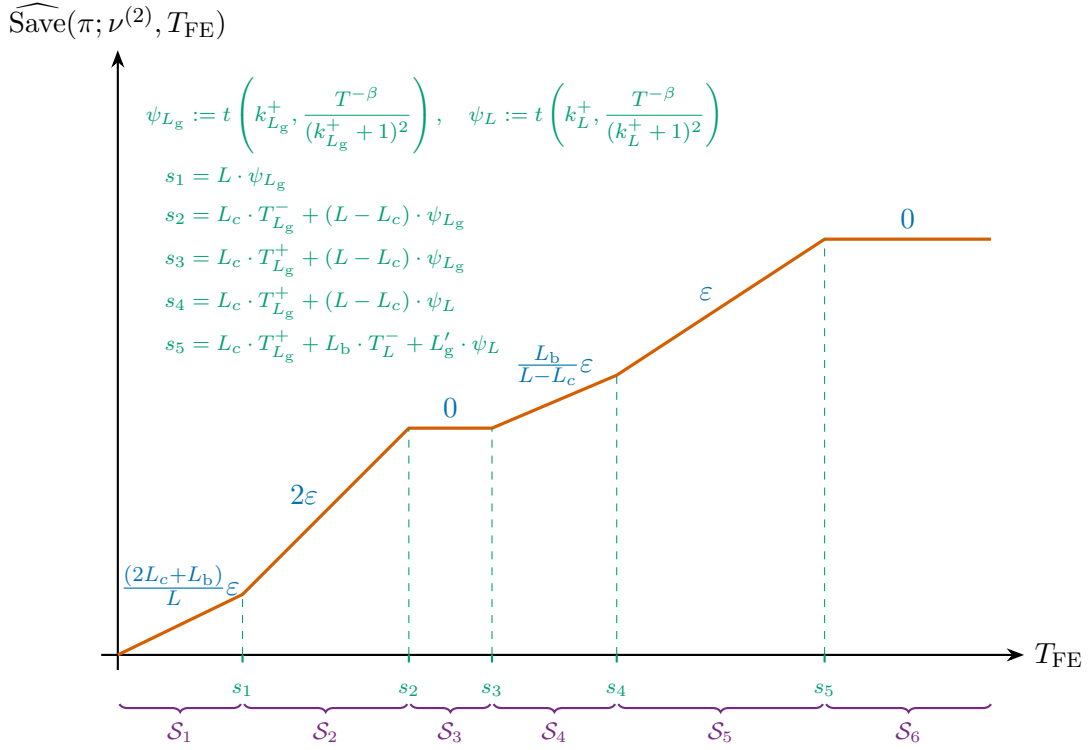
\begin{figure}[t]
    \begin{center}
            \begin{tikzpicture}[
    xscale=1.1, yscale=1.0,
    >=Stealth,
    thick,
    label node/.style={font=\footnotesize},
    tick label/.style={font=\scriptsize, below=4pt, text=acadGreen}
]

\def\xStart{0}
\def\xSone{1.5}     
\def\xStwo{3.5}     
\def\xSthree{4.5}   
\def\xSfour{6.0}    
\def\xSfive{8.5}    
\def\xEnd{10.5}

\def\yZero{0}
\def\ySone{0.8}
\def\yStwo{3.0}
\def\ySthree{3.0}   
\def\ySfour{3.7}
\def\ySfive{5.5}    
\def\yMax{8.0}

\draw[->] (-0.2, 0) -- (\xEnd+0.4, 0) node[right] {$\Tfe$};
\draw[->] (0, -0.2) -- (0, \yMax) node[above, align=center] {$\widehat{\text{Save}}(\pi;\nu^{(2)},\Tfe)$};

\node[anchor=north west, fill=white, fill opacity=0.8, text opacity=1, inner sep=4pt, font=\scriptsize, text=acadGreen] at (0.2, 7.8) {
    $\begin{aligned}
        \psi_{L_{\rmg}} &:= t\left(k_{L_{\rmg}}^+,\frac{T^{-\beta}}{(k_{L_{\rmg}}^+ +1)^2}\right), \quad \psi_{L}:=t\left(k_{L}^+,\frac{T^{-\beta}}{(k_{L}^+ +1)^2}\right) \\
        s_1 &= L\cdot \psi_{L_{\rmg}} \\
        s_2 &= L_c\cdot T_{L_{\rmg}}^- + (L-L_c) \cdot \psi_{L_{\rmg}} \\
        s_3 &= L_c\cdot T_{L_{\rmg}}^+ + (L-L_c) \cdot \psi_{L_{\rmg}} \\
        s_4 &= L_c\cdot T_{L_{\rmg}}^+ + (L-L_c)\cdot \psi_{L} \\
        s_5 &= L_c\cdot T_{L_{\rmg}}^+ + L_{\rmb}\cdot T_L^- + L_{\rmg}^\prime\cdot \psi_{L}
    \end{aligned}$
};

\draw[acadRed, very thick] 
    (0, \yZero) -- (\xSone, \ySone) node[pos=0.5, above=3pt, text=acadBlue] {$\frac{(2L_c+L_{\rmb})}{L}\varepsilon$}
    -- (\xStwo, \yStwo) node[pos=0.5, above left=-2pt, text=acadBlue] {$2\varepsilon$}
    -- (\xSthree, \ySthree) node[pos=0.5, above, text=acadBlue] {$0$}
    -- (\xSfour, \ySfour) node[pos=0.5, above=2pt, text=acadBlue] {$\frac{L_{\rmb}}{L-L_c}\varepsilon$}
    -- (\xSfive, \ySfive) node[pos=0.5, above left=-2pt, text=acadBlue] {$\varepsilon$}
    -- (\xEnd, \ySfive) node[pos=0.5, above, text=acadBlue] {$0$};


\draw[dashed, thin, acadGreen] (\xSone, 0) -- (\xSone, \ySone);
\draw[acadGreen] (\xSone, 0) -- ++(0, -0.1) node[tick label] {$s_1$};

\draw[dashed, thin, acadGreen] (\xStwo, 0) -- (\xStwo, \yStwo);
\draw[acadGreen] (\xStwo, 0) -- ++(0, -0.1) node[tick label] {$s_2$};

\draw[dashed, thin, acadGreen] (\xSthree, 0) -- (\xSthree, \ySthree);
\draw[acadGreen] (\xSthree, 0) -- ++(0, -0.1) node[tick label] {$s_3$};

\draw[dashed, thin, acadGreen] (\xSfour, 0) -- (\xSfour, \ySfour);
\draw[acadGreen] (\xSfour, 0) -- ++(0, -0.1) node[tick label] {$s_4$};

\draw[dashed, thin, acadGreen] (\xSfive, 0) -- (\xSfive, \ySfive);
\draw[acadGreen] (\xSfive, 0) -- ++(0, -0.1) node[tick label] {$s_5$};

\def\braceY{-0.65}
\draw[thick, acadPurple, decoration={brace, mirror, amplitude=3pt}, decorate] 
    (0.02, \braceY) -- (\xSone-0.02, \braceY) node[midway, below=3pt, font=\footnotesize, text=acadPurple] {$\mathcal{S}_1$};

\draw[thick, acadPurple, decoration={brace, mirror, amplitude=3pt}, decorate] 
    (\xSone+0.02, \braceY) -- (\xStwo-0.02, \braceY) node[midway, below=3pt, font=\footnotesize, text=acadPurple] {$\mathcal{S}_2$};

\draw[thick, acadPurple, decoration={brace, mirror, amplitude=3pt}, decorate] 
    (\xStwo+0.02, \braceY) -- (\xSthree-0.02, \braceY) node[midway, below=3pt, font=\footnotesize, text=acadPurple] {$\mathcal{S}_3$};

\draw[thick, acadPurple, decoration={brace, mirror, amplitude=3pt}, decorate] 
    (\xSthree+0.02, \braceY) -- (\xSfour-0.02, \braceY) node[midway, below=3pt, font=\footnotesize, text=acadPurple] {$\mathcal{S}_4$};

\draw[thick, acadPurple, decoration={brace, mirror, amplitude=3pt}, decorate] 
    (\xSfour+0.02, \braceY) -- (\xSfive-0.02, \braceY) node[midway, below=3pt, font=\footnotesize, text=acadPurple] {$\mathcal{S}_5$};

\draw[thick, acadPurple, decoration={brace, mirror, amplitude=3pt}, decorate] 
    (\xSfive+0.02, \braceY) -- (\xEnd, \braceY) node[midway, below=3pt, font=\footnotesize, text=acadPurple] {$\mathcal{S}_6$};

\end{tikzpicture}
        \caption{Illustration of $\widehat{\mathrm{Save}}(\pi;\nu^{(2)},\Tfe)$. 
        This function is piecewise linear, with change points 
        \textcolor{acadGreen}{$\{s_i\}_{i\in[5]}$} 
        and segment slopes indicated by the \textcolor{acadBlue}{numbers above each line segment}. 
        These change points partition the domain into six intervals 
        \textcolor{acadPurple}{$\{\mathcal{S}_i\}_{i\in[6]}$}.
        }
        \label{fig:savedstar_ins2}
    \end{center}
\end{figure}
Specifically, recall in $\nu^{(2)}$, there are $L_{\rmg}^\prime$ good arms with mean $1/2+\varepsilon$, $L_{\rmb}$ middle arms with mean $1/2$ and $L_{\rmc}$ bad arms with mean $1/2-\varepsilon$. The function is partitioned into $6$ intervals $\mathcal{S}_i$ for $i\in[6]$.
\begin{itemize}[leftmargin=2em,topsep=1pt]
    \item  When $\Tfe\in\mathcal{S}_1$, all arms are uniformly sampled,
    \item  When $\Tfe\in\mathcal{S}_2$, the worst arms are identified and are sampled.
    \item  When $\Tfe\in\mathcal{S}_3$, the worst arms have been fully sampled and are eliminated, and there is a short plateau due to the relaxation in the bound.
    \item  When $\Tfe\in\mathcal{S}_4$, the rest (good and middle) arms are uniformly sampled.
    \item  When $\Tfe\in\mathcal{S}_5$, the middle arms are identified and are sampled.
    \item  When $\Tfe\in\mathcal{S}_6$, the middle arms are eliminated and only the good arms remain active.
\end{itemize}
Denote the switching points for this instance as follows
\begin{align}
    &\overline{r}_1 \!=\! \frac{L(c+2)^2\beta}{4L_{\rmb}}+O(\varepsilon)
    ,
    &&\overline{r}_2 \!= \!\frac{(L-L_{\rmc})(c+2)^2\beta}{4L_{\rmb}}+\frac{L_{\rmc}}{4L_{\rmb}}+O(\varepsilon)
    ,
    \\
    &
    \overline{r}_3\! =\! \frac{(L-L_{\rmc})(c+2)^2\beta}{4L_{\rmb}}+\frac{L_{\rmc}(c+2)^2}{4L_{\rmb}(c-2)^2}+O(\varepsilon)
    ,
    &&\overline{r}_4\! =\! \frac{L_{\rmg}^\prime (c+2)^2\beta}{4L_{\rmb}}+\frac{L_{\rmc}(c+2)^2}{4L_{\rmb}(c-2)^2} + 1+O(\varepsilon),
    \\
    &\overline{r}_5\! =\! \frac{L_{\rmg}^\prime(c+2)^2\beta}{L_{\rmb}}+ \frac{L_{\rmc}(c+2)^2}{4L_{\rmb}(c-2)^2} + 1+O(\varepsilon) .
\end{align}
With $\Tfe = \gamma \frac{L_{\rmb}\cdot \ln T}{\kl(\nu_L,\nu_1)}$,
the corresponding $\alpha$ can be computed as
\begin{align}\label{equ:tight_alpha_two_value_alt2}
    \alpha \geq
    \begin{cases}
        \frac{L_{\rmc}+\frac{1}{2}L_{\rmb}}{L},
        &
            \gamma \in (0,\overline{r}_1),
        \\
        1-\frac{(c+2)^2\beta}{4\gamma}\Big(\frac{1}{2}+\frac{L_{\rmg}^\prime}{L_{\rmb}}\Big) +O(\varepsilon),
        &
           \gamma \in [\overline{r}_1,\overline{r}_2),
        \\
        \frac{L_{\rmc}}{4\gamma L_{\rmb}}+ \frac{(c+2)^2\beta}{8\gamma} +O(\varepsilon),
        &
             \gamma \in [\overline{r}_2,\overline{r}_3),
        \\
        \frac{L_{\rmb}}{1-L_{\rmc}} + \frac{\frac{L_{\rmc}}{2}}{\frac{L_{\rmc}}{4}+\gamma L_{\rmb}}
        -\frac{\frac{L_{\rmb} L_{\rmc}}{4(1-L_{\rmc})}(1+(\frac{c+2}{c-2})^2)}{\frac{L_{\rmc}}{4}+\gamma L_{\rmb}} +O(\varepsilon),
        & 
              \gamma \in [\overline{r}_3,\overline{r}_4),
         \\
         1-\frac{L_{\rmg}^\prime(c+2)^2\beta}{\gamma L_{\rmb}+\frac{L_{\rmc}}{4}} +O(\varepsilon),
        & 
             \gamma \in [\overline{r}_4,\overline{r}_5),
        \\
         1-O(\varepsilon),
        & 
            \gamma \geq   \gamma \in [\overline{r}_5,\infty).
    \end{cases}
\end{align}
We now specify the choice of $L_{\rmg}^\prime$: we note that the conditions in Theorem~\ref{thm:lwbd_2value}, $\frac{(L_{\rmg}-L_{\rmg}^\prime)\ln T}{\kl(\nu_{L_{\rmg}}^{(2)},\nu_1^{(2)})}=\frac{L_{\rmc}\ln T}{\kl(\nu_{L_{\rmg}}^{(2)},\nu_1^{(2)})}\geq \Tfe$, requires
$L_{\rmc}>4\gamma L_{\rmb}+O(\varepsilon)$, thus we are in the second case in the above equation. 
Additionally, this condition can be satisfied for $\gamma\in[\frac{L(c+2)^2\beta}{4L_{\rmb}}+O(\varepsilon),\frac{L_{\rmg}-2}{4L_{\rmb}}]$, as $L_{\rmc}<L_{\rmg}$.
Moreover, Theorem~\ref{thm:lwbd_2value} also requires $\xi_2\geq \frac{1-\alpha}{2\alpha-1}\Delta_L$. In order to make our choice of $\xi_2=\varepsilon$ in $\nu^{(2)}$ valid, we restrict $\gamma\in[\max\{\frac{L}{4},3L_{\rmg}\}\frac{(c+2)^2\beta}{L_{\rmb}}+O(\varepsilon),\frac{L_{\rmg}-2}{4L_{\rmb}}]$, so that $\alpha\geq \frac{2}{3}$ in~\eqref{equ:tight_alpha_two_value} and $\xi_2=\varepsilon$ is permissible.

Lastly, because $L_{\rmg}>L_{\rmb}$, we can show that given the \ac{fe} budget $\Tfe$, the $\alpha$ under the alternative instances is greater than that of the original instance up to $O(\varepsilon)$:
\begin{itemize}[leftmargin=1em,parsep=1pt]
    \item for $\nu^{(1)}$ and $\nu^{(3)}$, when $ \frac{L}{L_{\rmb}} (c+2)^2\beta  + O(\varepsilon) \leq \gamma \leq
            1 +  \frac{L_{\rmg}}{L_{\rmb}} (c+2)^2\beta + O(\varepsilon)$,
            by \eqref{equ:tight_alpha_two_value} and~\eqref{equ:tight_alpha_two_value_13},
        \begin{align}
            &1-\frac{L_{\rmg}\cdot (c+2)^2\beta}{L_{\rmb}\cdot\gamma}+O(\varepsilon)
            \leq
            1-\frac{(c+2)^2\beta}{\gamma}+O(\varepsilon) , 
        \end{align}
    \item for $\nu^{(2)}$, when $ \frac{3L_{\rmg}}{L_{\rmb}} (c+2)^2\beta  + O(\varepsilon) \leq \gamma \leq
            \min\Big\{1,
            \frac{L_{\rmg}-2}{4L_{\rmb}}
            \Big\}$, by  \eqref{equ:tight_alpha_two_value} and \eqref{equ:tight_alpha_two_value_alt2},
    \begin{align}
        1-\frac{L_{\rmg}\cdot (c+2)^2\beta}{L_{\rmb}\cdot\gamma}+O(\varepsilon)
        \leq 
        1-\frac{(c+2)^2\beta}{4\gamma}\Big(\frac{1}{2}+\frac{L_{\rmg}^\prime}{L_{\rmb}}\Big) +O(\varepsilon).
    \end{align}
\end{itemize}
Thus, there exists $(\alpha,\beta)$-probably saving algorithm on the instances considered, $\nu,\nu^{(j)},j=1,2,3 $. The conditions in Theorem~\ref{thm:lwbd_2value} are satisfied.
We now derive the lower bound via Theorem~\ref{thm:lwbd_2value}.

\paragraph{Lower Bound:}
By the construction of the alternative instances, the set of $(\alpha,\beta)$-probably saving algorithms is nonempty, where $\alpha$ is specified in \eqref{equ:tight_alpha_two_value}, for the instances considered.
By adopting \eqref{equ:taylor_kl}, we have
\begin{align}
    \kl(P_1,Q_1) = 8\varepsilon^2+O(\varepsilon^3)
    \quad\mbox{and}\quad
    \kl(P_2,Q_2) = 2\varepsilon^2+O(\varepsilon^3)
\end{align}
for $(P_1,Q_1)= (\nu_1,\nu_1^{(3)}),(\nu_{L_{\rmg}}^{(2)},\nu_1^{(2)}),(\nu_{L_{\rmg}},\nu_{L_{\rmg}}^{(2)})$ and $(P_2,Q_2)=(\nu_1,\nu_1^{(1)}),(\nu_L,\nu_L^{(1)}),(\nu_L,\nu_1).$
In this case, $\tau_1,\tau_2$ in \eqref{equ:lwbd_tau_two_value} become
\begin{align}
    \tau_1 = \frac{\beta\ln T}{2\varepsilon^2+O(\varepsilon^3)}
    ,\quad
    \tau_2 = \min\bigg\{\frac{L_{\rmb}\ln T}{2\varepsilon^2+O(\varepsilon^3)}+\frac{\beta\ln T}{8\varepsilon^2+O(\varepsilon^3)}, \frac{\frac{\alpha}{1-\alpha}L_{\rmb}\ln T}{2\varepsilon^2+O(\varepsilon^3)}\bigg\}.
\end{align}
By considering instances $\nu^{(1)}$ and $\nu^{(3)}$,
the saved regret is upper bounded as
\begin{align}
    \widetilde{\Saved}(\pi;\nu,\Tfe)&=\sum_{i\notin [L_{\rmg}]}
    \Delta_i\cdot \min\bigg\{\frac{\ln T}{\kl(\nu_i,\nu_1)},\bbE[T_{i,\Tfe}]\bigg\}
    \\
    &
    \leq
    \begin{cases}
        \mbox{N.A.},& \Tfe < \tau_1,
        \\
        \Delta_L\cdot \Big(\Tfe-\frac{\beta\ln T}{8\varepsilon^2+O(\varepsilon^3)}\Big),& \tau_1\leq \Tfe\leq \tau_2,
        \\
        \Delta_L\cdot \frac{L_{\rmb}\cdot\ln T}{\kl(\nu_L,\nu_1)},& \Tfe> \tau_2.
    \end{cases}
\end{align}
For the regime, $\Tfe\in[\tau_1,\frac{L_{\rmb}\ln T}{2\varepsilon^2+O(\varepsilon^3)}]$, the bound can be improved by considering alternative instance $\nu^{(2)}$. Specifically, 
when 
$\frac{3L_{\rmg}(c+2)^2\beta \ln T}{2\varepsilon^2+O(\varepsilon^3)}\leq \Tfe \leq \min\{\frac{L_{\rmb}\ln T}{2\varepsilon^2+O(\varepsilon^3)},\frac{(L_{\rmg}-2)\ln T}{8\varepsilon^2+O(\varepsilon^3)}\} $, 
alternative instance $\nu^{(2)}$ can be used and
we have
\begin{align}
    \widetilde{\Saved}(\pi;\nu,\Tfe)&=\sum_{i\notin [L_{\rmg}]}
    \Delta_i\cdot \min\bigg\{\frac{\ln T}{\kl(\nu_i,\nu_1)},\bbE[T_{i,\Tfe}]\bigg\}
    \leq
    \Delta_L \cdot \bigg(
            \Tfe - 
            \frac{\zeta\cdot L_{\rmg}\cdot \beta\ln T}{8\varepsilon^2+O(\varepsilon^3)}
        \bigg)
\end{align}
where $\zeta=\frac{1}{\max\big\{\frac{\Tfe}{\ln T}\cdot (8\varepsilon^2+O(\varepsilon^3)),1\big\}}$.

We assume $\Tfe = \gamma \frac{L_{\rmb}\ln T}{d(\mu_L,\mu_1)}$. By using Lemma~\ref{lem:lwbd_whole}, we obtain
\begin{align}\label{equ:tight_low_two_value}
    &\underline{\mathfrak{R}}(\pi^\prime,\nu)=
    \liminf_{T\to\infty} \frac{R_{\Tfe+1:T} (\pi^\prime;\nu ) }{\ln T}
    \\
    &\geq
    \begin{cases}
        \mbox{N.A.},& \gamma<\frac{\beta}{L_{\rmb}}+O(\varepsilon),
        \\
        \frac{L_{\rmb}\Delta_L}{\rmd(\mu_L,\mu_1)}(1-\gamma+\frac{\beta}{4L_{\rmb}}+O(\varepsilon)),&
        \frac{\beta}{L_{\rmb}} + O(\varepsilon)\!\leq\! \gamma\!\leq \!\min\!\big\{1+\frac{\beta}{4L_{\rmb}}+O(\varepsilon),\frac{\alpha}{1-\alpha}(1+O(\varepsilon))\big\},\!
        \\
        0,& \gamma> \min\!\big\{1+\frac{\beta}{4L_{\rmb}}+O(\varepsilon),\frac{\alpha}{1-\alpha}(1+O(\varepsilon))\big\}.\!
    \end{cases}
 \end{align}
When $\frac{3L_{\rmg}(c+2)^2\beta}{L_{\rmb}} + O(\varepsilon)\leq \gamma\leq \min\big\{1,\frac{L_{\rmg}-2}{4L_{\rmb}}\big\}$, we have $\zeta = \frac{1}{\max\{4\gamma L_{\rmb}(1+O(\varepsilon)),1\}}$ and 
\begin{align}
    \liminf_{T\to\infty} \frac{R_{\Tfe+1:T} (\pi ) }{\ln T}
    \geq
    \frac{L_{\rmb}\Delta_L}{\rmd(\mu_L,\mu_1)}
    \bigg(1-\gamma+\frac{\zeta \cdot L_{\rmg}\cdot\beta}{4L_{\rmb}}+O(\varepsilon)\bigg).
\end{align}
This bound is better when the \ac{fe} budget is small  or there are more good items.

 We compare the upper bound of \ouralg{} in \eqref{equ:tight_up_two_value} and the lower bound of any $ (\alpha,\beta)$-probably saving algorithms in \eqref{equ:tight_low_two_value} across different time regimes in terms of
 $$\overline{\mathfrak{R}}(\pi,\nu)-\underline{\mathfrak{R}}(\pi^\prime,\nu)= \limsup_{T\to\infty}  \frac{R_{\Tfe+1:T} (\pi )}{\ln T} -  \liminf_{T\to\infty}  \frac{R_{\Tfe+1:T} (\pi^\prime )}{\ln T}.$$
\begin{itemize}
    \item When $ \frac{L}{L_{\rmb}} (c+2)^2\beta  + O(\varepsilon) \leq \gamma \leq 1 +  \frac{L_{\rmg}}{L_{\rmb}} (c+2)^2\beta + O(\varepsilon)$,
        this gap is upper bounded as
        $$\overline{\mathfrak{R}}(\pi,\nu)-\underline{\mathfrak{R}}(\pi^\prime,\nu)\le \frac{L_{\rmb}\cdot \Delta_{2}}{\rmd(\mu_2,\mu_1)}\bigg(\frac{(L_{\rmg}(c+2)^2-\frac{1}{4})\beta}{L_{\rmb}}+O(\varepsilon)\bigg).$$     
    \item Furthermore, when $\frac{3L_{\rmg}(c+2)^2\beta}{L_{\rmb}} + O(\varepsilon)\leq \gamma\leq \min\big\{1,\frac{L_{\rmg}-2}{4L_{\rmb}}\big\}$,
    the bound on the gap can be further decreased to
     $$\overline{\mathfrak{R}}(\pi,\nu)-\underline{\mathfrak{R}}(\pi^\prime,\nu)\le \frac{L_{\rmb}\cdot \Delta_{2}}{\rmd(\mu_2,\mu_1)}\bigg(\frac{((c+2)^2-\zeta)L_{\rmg}\beta}{L_{\rmb}}+O(\varepsilon)\bigg)$$  
     where $\zeta = \frac{1}{\max\{4\gamma L_{\rmb}(1+O(\varepsilon)),1\}}= \min\{ \frac{1+O(\varepsilon)}{4\gamma L_{\rmb}},1\}$.
     As we consider the case where $\beta$ is small, this gap is considered as small.
    \item When $\gamma>1 +  \frac{L_{\rmg}}{L_{\rmb}} (c+2)^2\beta + O(\varepsilon)$, \ouralg{} can save almost all regret up to $\frac{\Delta_{L}\ln T}{\rmd(\mu_L,\mu_1)}\cdot O(\varepsilon)$.
\end{itemize}
This completes the proof of Corollary~\ref{cor:tightness}.

\section{More Experimental Results}\label{app:exp}

\begin{figure}[t]
\centering

\subfigure[Instance 1]{
    \includegraphics[width=0.3\textwidth]{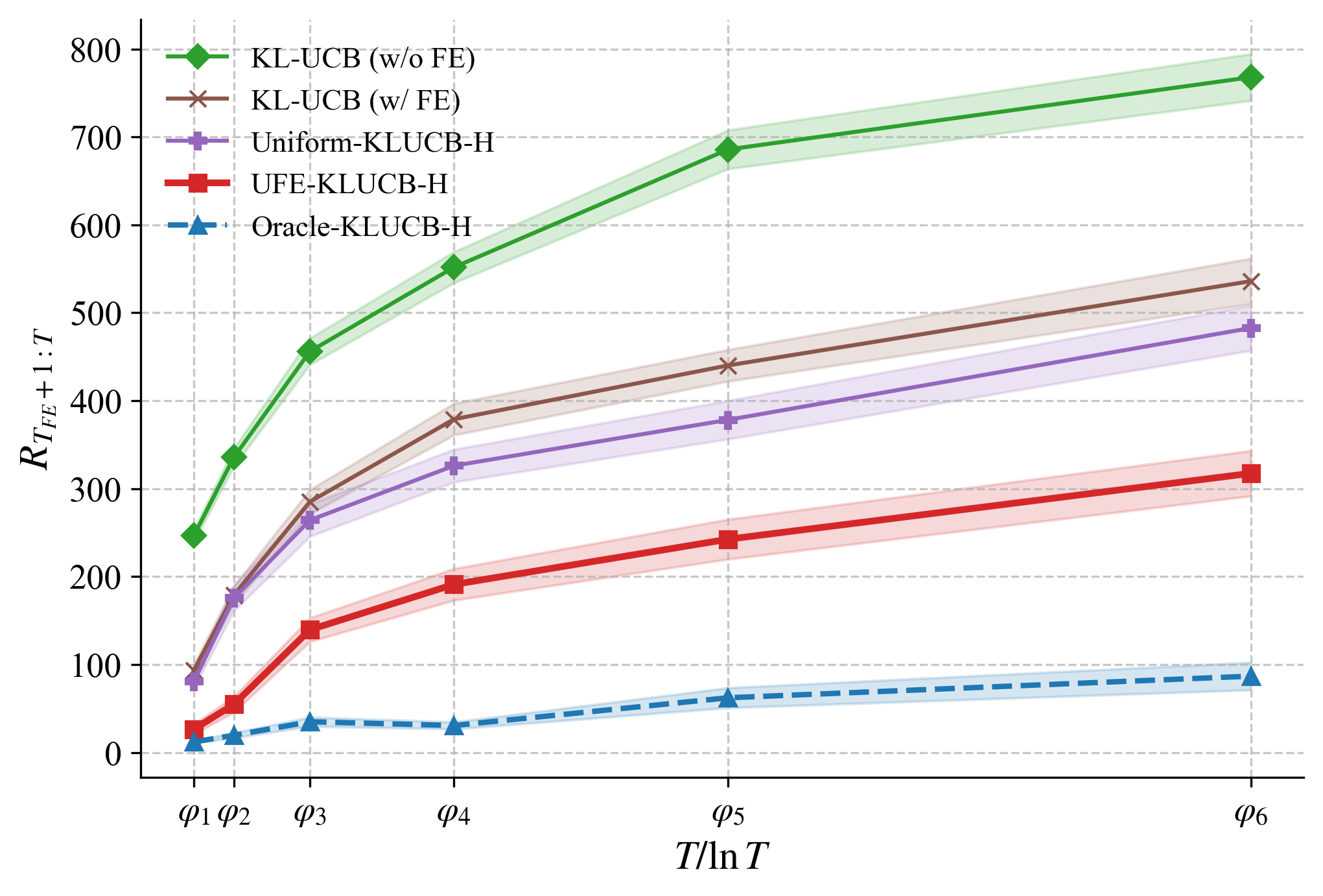}
}
\hfill
\subfigure[Instance 2]{
    \includegraphics[width=0.3\textwidth]{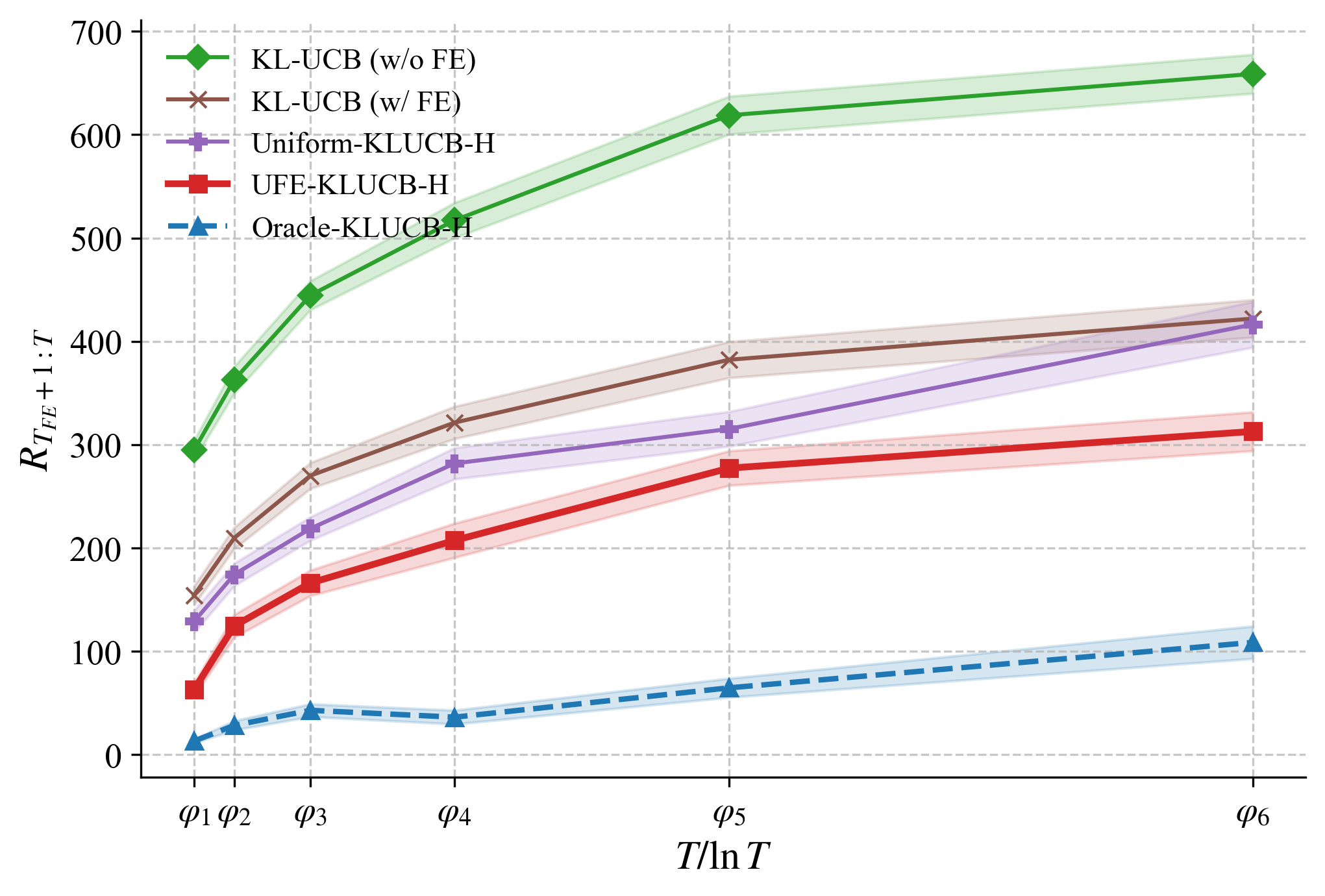}
}
\hfill
\subfigure[Instance 3]{
    \includegraphics[width=0.3\textwidth]{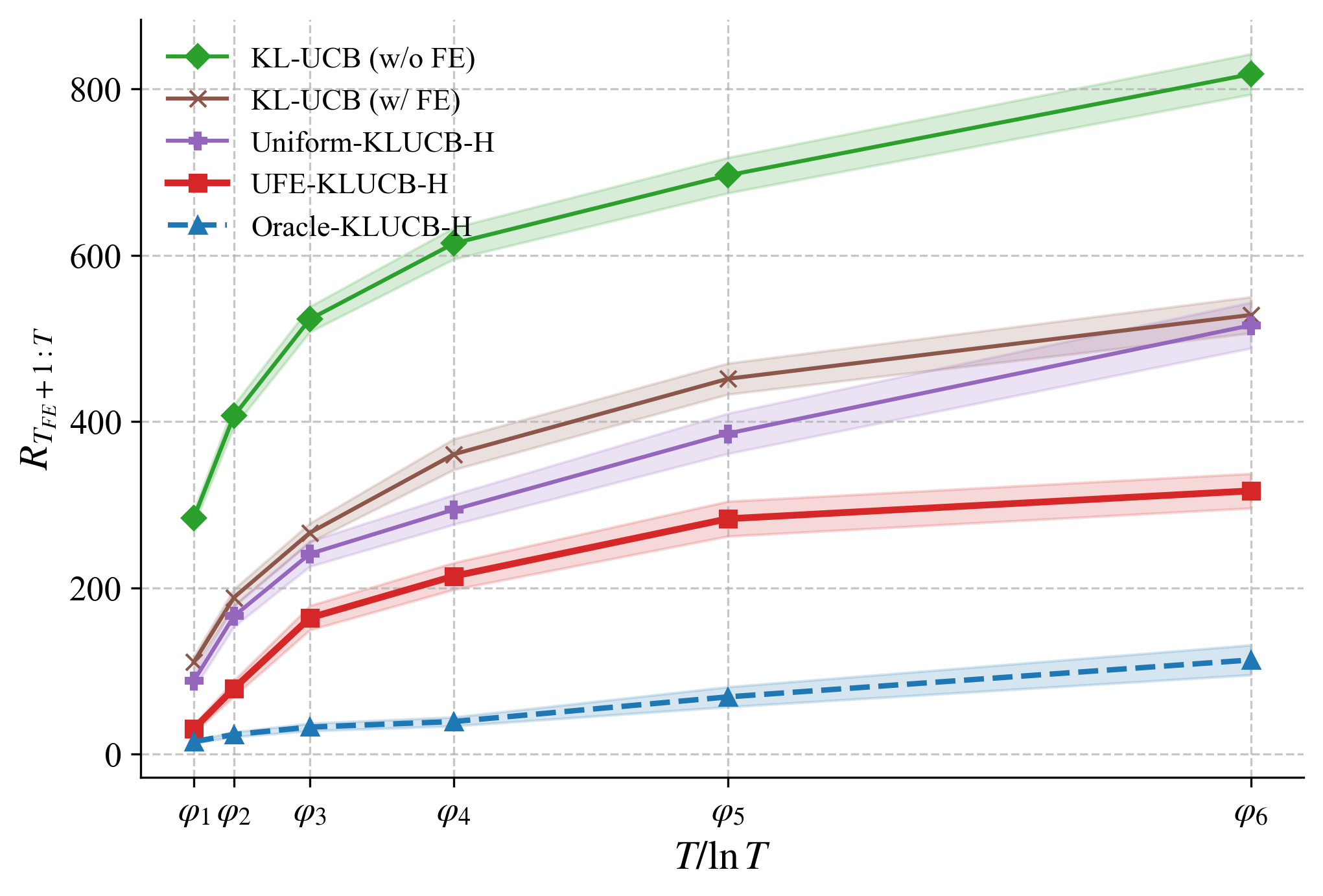}
}

\vspace{0.5em}

\subfigure[Instance 4]{
    \includegraphics[width=0.3\textwidth]{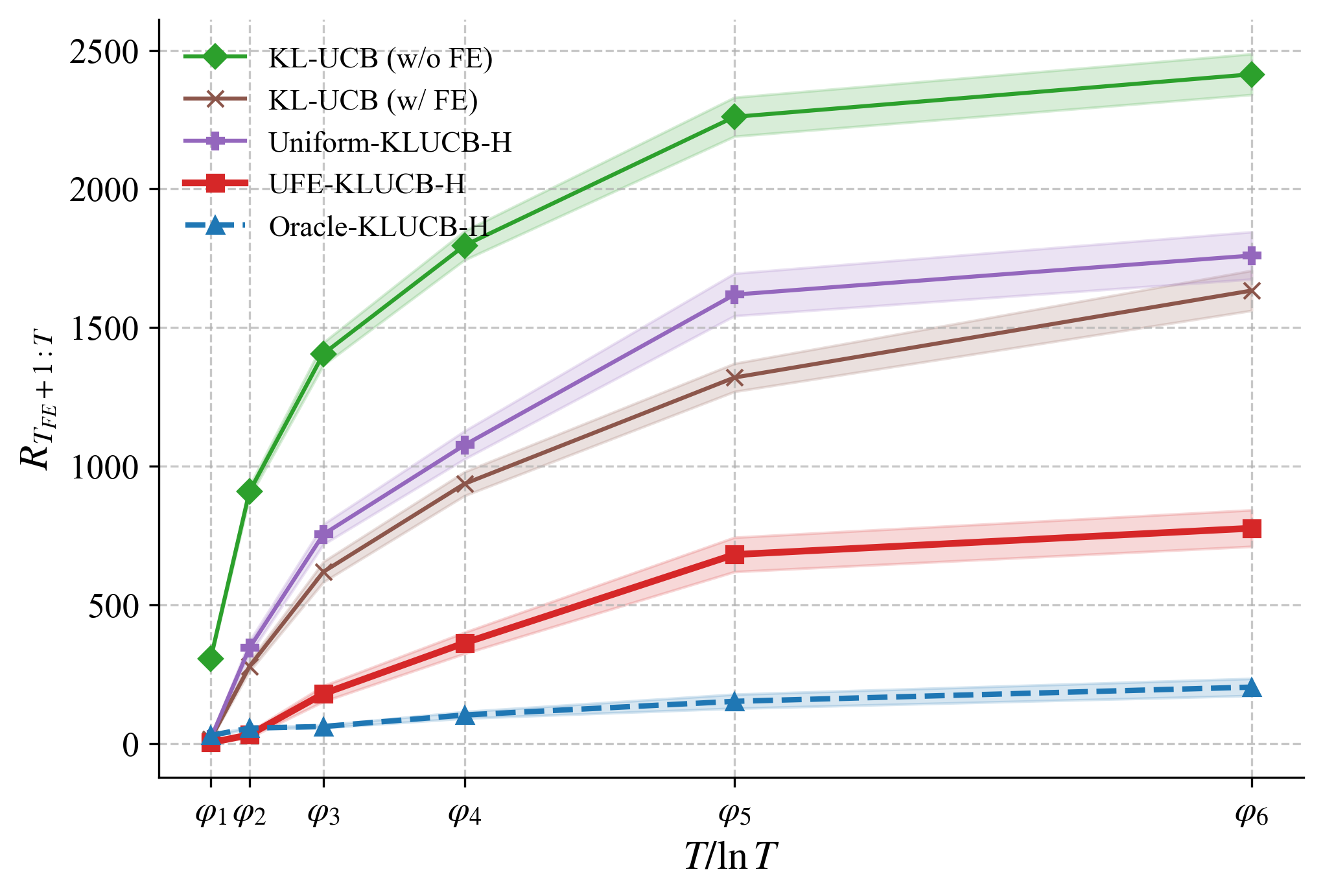}
}
\hfill
\subfigure[Instance 5]{
    \includegraphics[width=0.3\textwidth]{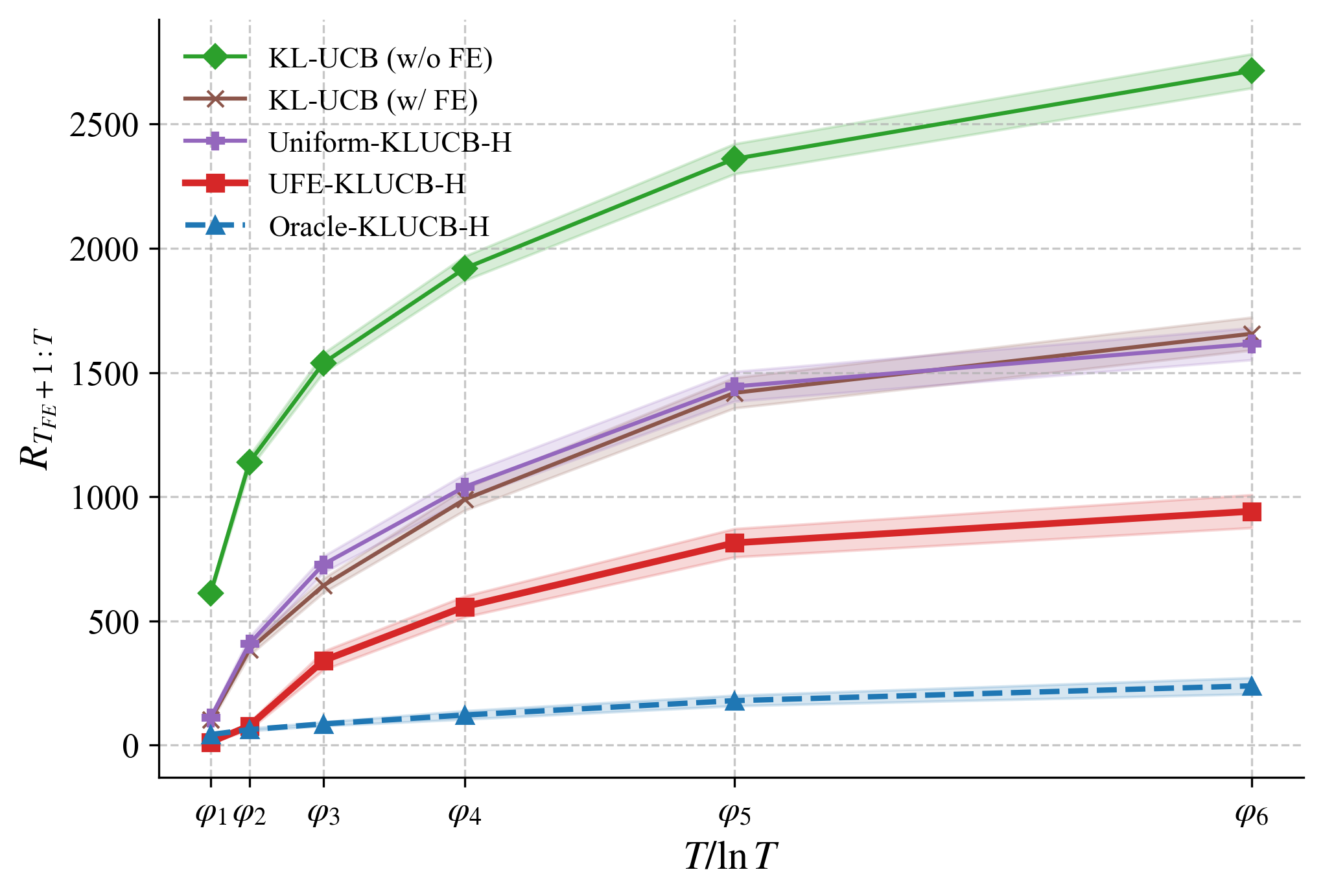}
}
\hfill
\subfigure[Instance 6]{
    \includegraphics[width=0.3\textwidth]{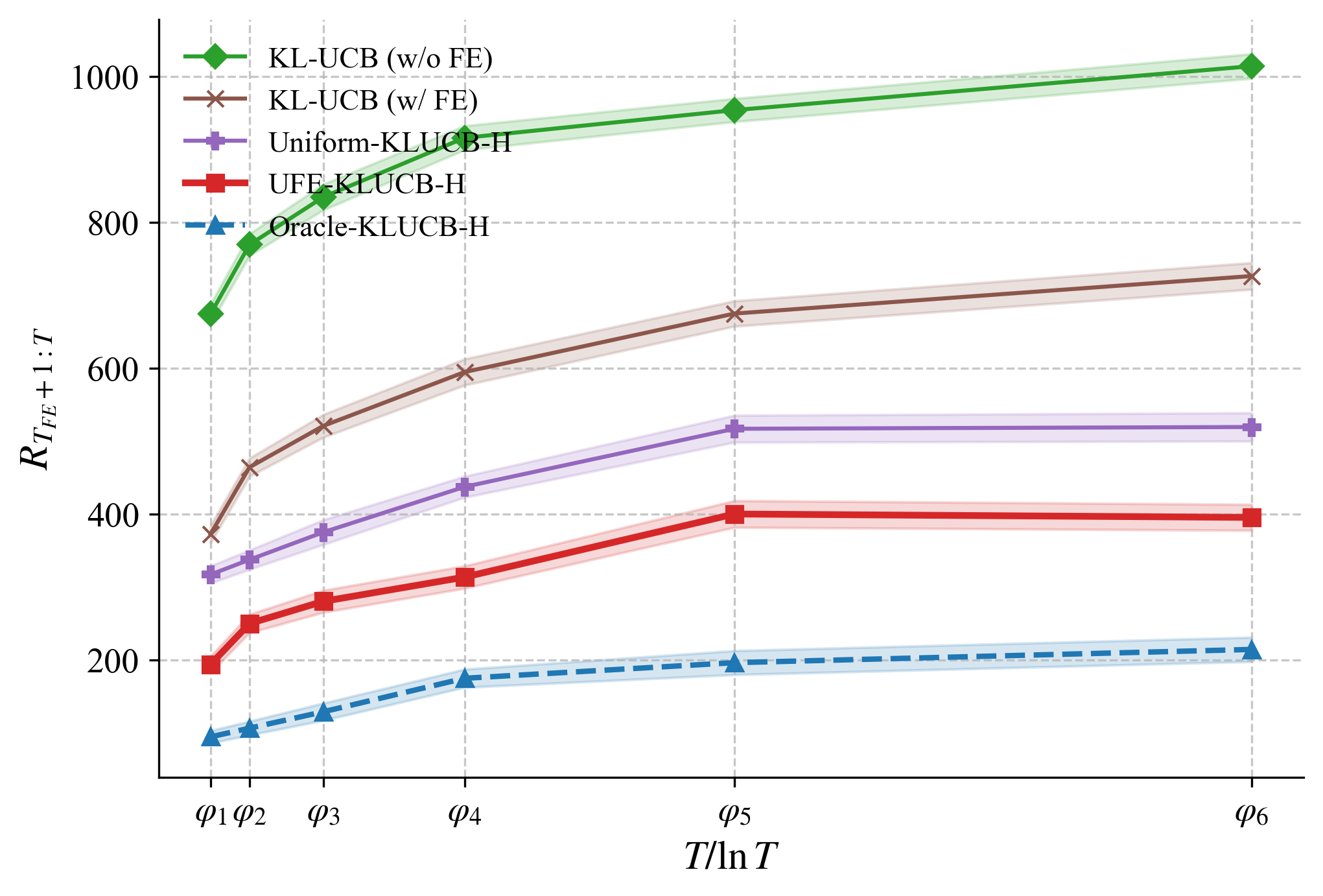}
}

\vspace{0.5em}

\subfigure[Instance 7]{
    \includegraphics[width=0.3\textwidth]{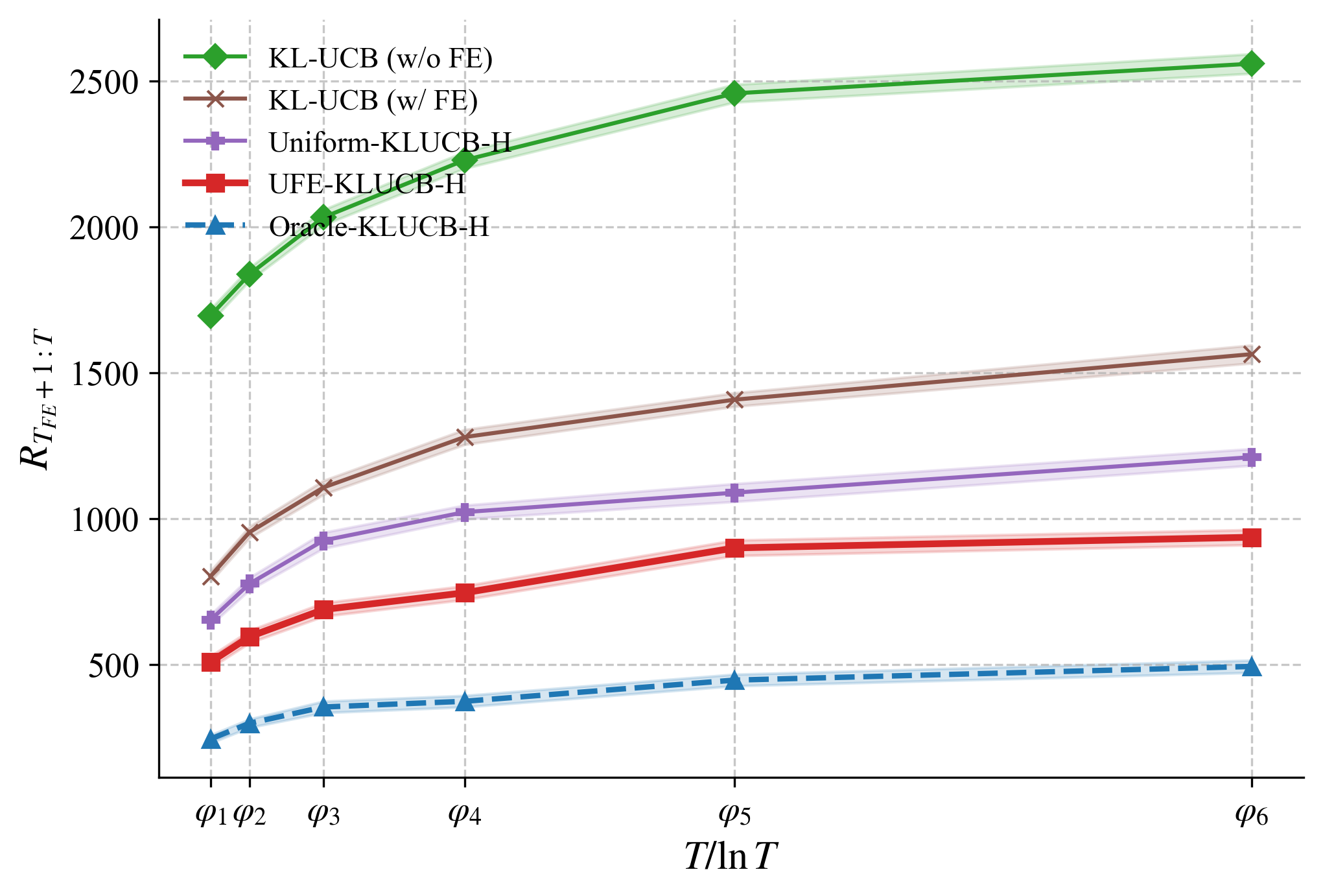}
}
\hfill
\subfigure[Instance 8]{
    \includegraphics[width=0.3\textwidth]{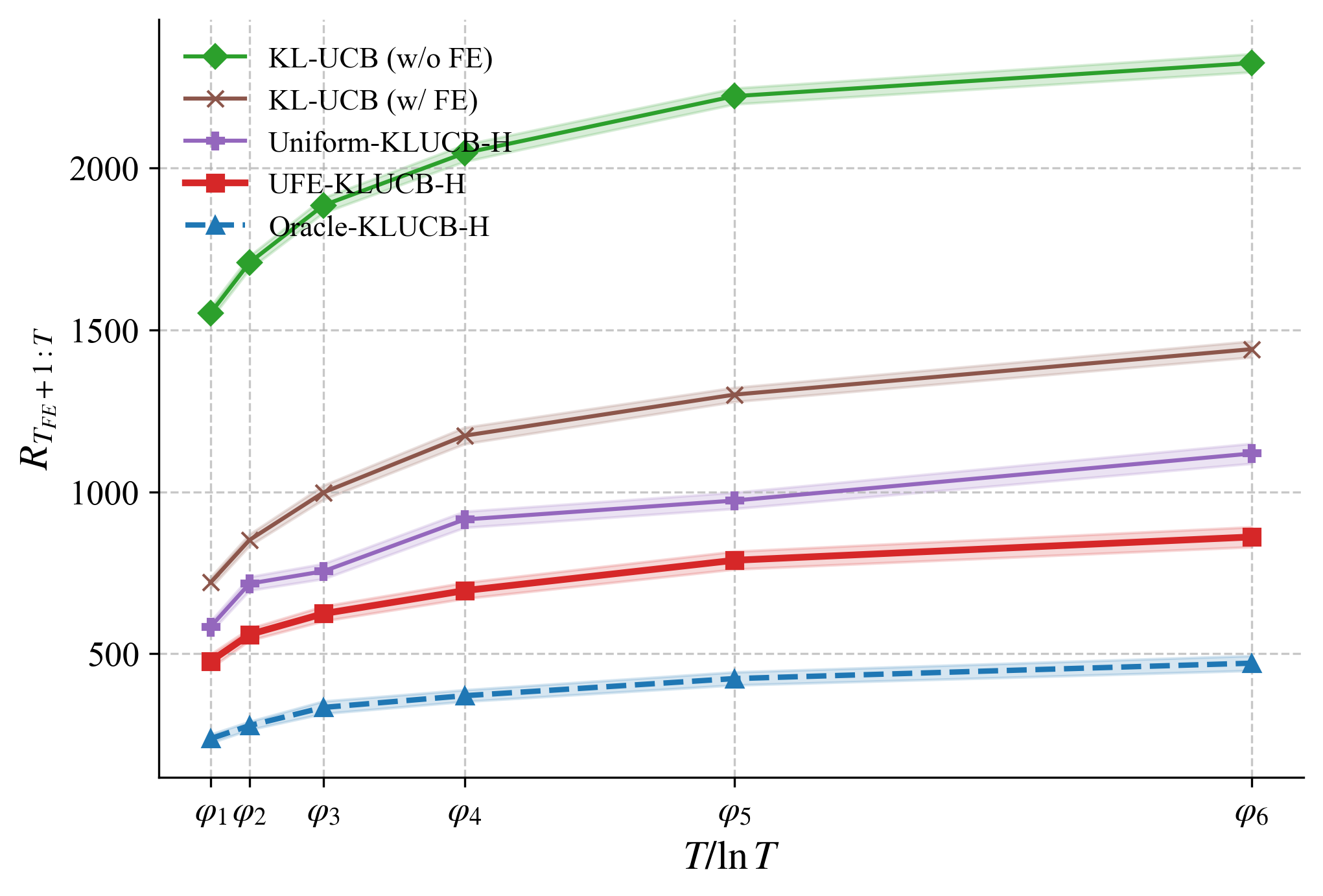}
}
\hfill
\subfigure[Instance 9]{
    \includegraphics[width=0.3\textwidth]{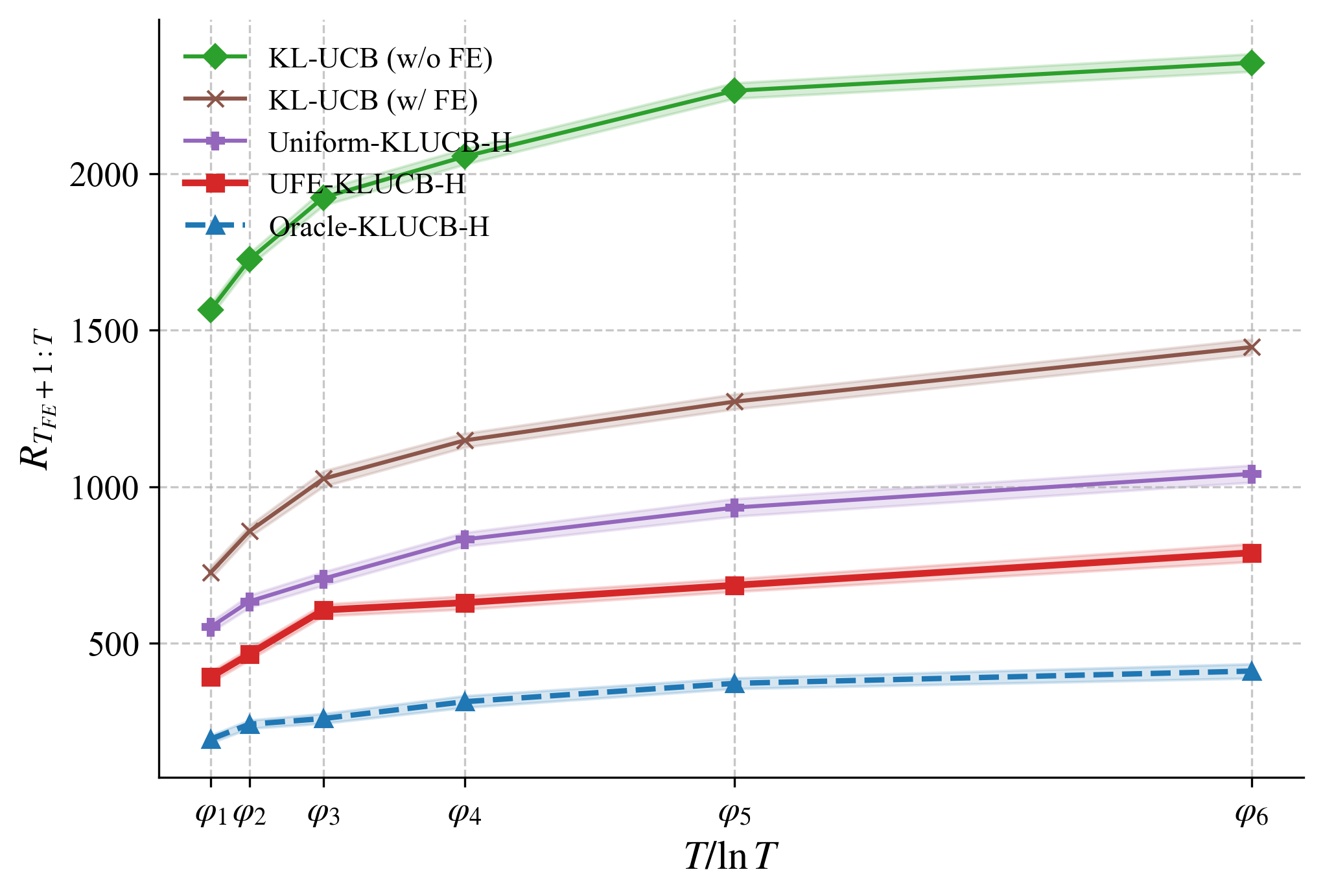}
}

\caption{Comparison between \ouralg{} and other algorithms across randomly generated Gaussian instances with $L=5,10,20$ arms. \ouralg{} consistently outperforms other candidates across instances, which showcases its practicality.}
\label{fig:exp_random}
\end{figure}
We present more experimental results to showcase the potential of \ouralg{}.
We randomly generated 3 instances for $L=5,10,20$ arms, respectively.
We denote $\varphi_i=x_i/\ln x_i$ and $x_i=T_0\times 2^i$ for $i\in[6]$. For $L=5$, $T_0=40,000$; for $L=10$ and $20$, $T_0=200,000$.
\begin{itemize}[leftmargin=1em,parsep=0.5pt]
    \item 5 arms: 
    \begin{itemize}[leftmargin=1em,parsep=0.5pt]
        \item Ins 1: $[.01, .65, .20, .03, .62]$;
    	\item Ins 2: $[.92,.96,.21,.04,.68]$;
    	\item Ins 3: $[.09,.72,.69,.23,.43]$;
    \end{itemize}
    	
    \item 10 arms: 
    \begin{itemize}[leftmargin=1em,parsep=0.5pt]
        \item Ins 4: $[.05,.26,.29,.09,.37,.54,.72,.71,.21,.11]$;
    	\item Ins 5: $[.44,.03,.67,.78,.99,.01,.18,.90,.69,.98]$;
    	\item Ins 6: $[.55,.59,.86,.04,.05,.39,.33,1.00,.31,.94]$;
    \end{itemize}
    	
    \item 20 arms:
    \begin{itemize}[leftmargin=1em,parsep=0.5pt]
        \item Ins 7: $[.51,.80,.91,.30,.92,.99,.72,.17,.01,.32,.67,.35,.80,.12,.64,.94,.36,.04,.88,.19]$;
    	\item Ins 8: $[.45,.55,.78,.39,.44,.16,.75,.85,.89,.95,.30,.62,.39,.07,.16,.35,.57,.89,.49,.05]$;
    	\item Ins 9: $[.66,.40,.40,.91,.88,.81,.28,.56,.82,.59,.43,.10,.24,.57,.36,.05,.81,.45,.41,.96]$.
    \end{itemize}
\end{itemize}

The simulation results are presented in Figure~\ref{fig:exp_random},
Across randomly generated instances with varying numbers of arms, the proposed \ouralg{} consistently saves more regret, which establishes the practicality of the algorithm and also corroborates our theoretical findings.

\section{Technical Lemmas}\label{app:tech_lem}

\begin{lemma}\label{lem:oracle_regret}
    Given any instance $\nu$ and any consistent algorithm $\pi$ with a deterministic \ac{fe} policy $\pi_{\FE}$, with \ac{fe} budget $\Tfe$, the regret saved by the algorithm $\pi$ is upper bounded by $\Saved^*(\nu,\Tfe)$ as in \eqref{equ:oracle_regret}.
\end{lemma}
\begin{proof}
Let $\mathrm{KL}_i := \kl(\nu_i,\nu_1)$ and $R_{\Tfe+1:T}=R_{\Tfe+1:T}(\pi;\nu)$. The regret is
\begin{align}\label{equ:lwbd}
    R_{\Tfe+1:T}=\sum_{i:\Delta_i>0}\bbE[(T_{i,T}-T_{i,\Tfe})_+]\Delta_i
    \ge\sum_{i:\Delta_i>0}\bigg(\frac{\ln T}{\mathrm{KL}_i}-T_{i,\Tfe}\bigg)_+\Delta_i,
\end{align}
where Jensen's inequality is applied to the convex function $(\cdot)_+$, and $\bbE[T_{i,\Tfe}]=T_{i,\Tfe}$ since we use a deterministic \ac{fe} policy. Thus, for any asymptotically optimal algorithm,
$$R_{1:T}-R_{\Tfe+1:T}\le\sum_{i:\Delta_i>0}\min\bigg\{T_{i,\Tfe},\frac{\ln T}{\mathrm{KL}_i}\bigg\}\Delta_i=:\widehat{\Saved}(\pi;\nu,\Tfe).$$
The maximum value $\Saved^*(\nu,\Tfe)$ is the solution to
$$\max_{T_{i,\Tfe}, i\in[K]}\widehat{\Saved}(\pi;\nu,\Tfe),
\quad\text{s.t. }\quad T_{i,\Tfe}\ge 0,\\ \sum_i T_{i,\Tfe}=\Tfe.$$
where the \ac{fe} budget is allocated to larger-gap arms first, which gives \eqref{equ:oracle_regret}.
\end{proof}

\begin{lemma}\label{lem:KLUCBH_lem1}
    For any arm $i$,
    \begin{align}
        \bbE\bigg[ \sum_{t=\Tfe+1}^T\mathbbm{1}\{i_t=i,\mu_1\leq u_{1,t}\}\bigg]
        \leq
        \bbE\bigg[ \sum_{s=T_{i,\Tfe}}^T \mathbbm{1}\{s\cdot \rmd^+(\hatmu_i^s,\mu_1)\leq \log T+3 \log\log T\}\bigg]
    \end{align}
\end{lemma}
\begin{proof}
    By the design of $i_t$, $i_t=i$ and $\mu_1\leq u_{1,t}$ indicates $u_{i,t}\geq u_{1,t}\geq\mu_1$. Therefore, 
    \begin{align}
        \rmd^+(\hatmu_{i,t},\mu_1)\leq \rmd(\hatmu_{i,t},\mu_{1,t})\leq \frac{\log t+3 \log\log t}{T_{i,t}}.
    \end{align}
    We can bound the summation as follows
    \begin{align}
        & \sum_{t=\Tfe+1}^T\mathbbm{1}\{i_t=i,\mu_1\leq u_{1,t}\}
        \\
        &\quad\leq
        \sum_{t=\Tfe+1}^T\mathbbm{1}\{i_t=i,T_{i,t}\cdot \rmd^+(\hatmu_{i,t},\mu_1)\leq \log t+3 \log\log t\}
        \\
        &\quad\leq
        \sum_{t=\Tfe+1}^T\sum_{s=T_{i,\Tfe}+1}^t\mathbbm{1}\{i_t=i,T_{i,t}=s,s\cdot \rmd^+(\hatmu_{i}^s,\mu_1)\leq \log T+3 \log\log T\}
        \\
        &\quad\leq
        \sum_{s=T_{i,\Tfe}+1}^T \sum_{t=\max\{s,\Tfe+1\}}^T\mathbbm{1}\{i_t=i,T_{i,t}=s,s\cdot \rmd^+(\hatmu_{i}^s,\mu_1)\leq \log T+3 \log\log T\}
        \\
        &\quad\leq
        \sum_{s=T_{i,\Tfe}+1}^T \mathbbm{1}\{s\cdot \rmd^+(\hatmu_{i}^s,\mu_1)\leq \log T+3 \log\log T\}\cdot \!\!\sum_{t=\max\{s,\Tfe+1\}}^T\mathbbm{1}\{i_t=i,T_{i,t}=s\}
        \\
        &\quad\leq
        \sum_{s=T_{i,\Tfe}+1}^T \!\!\mathbbm{1}\{s\cdot \rmd^+(\hatmu_{i}^s,\mu_1)\leq \log T+3 \log\log T\}.
    \end{align}
    This completes the proof.
\end{proof}

\begin{lemma}\label{lem:klvsgap}
    Given two distributions $p_{\theta_i},i=1,2$ in the one-dimensional canonical exponential family, i.e.,
    \begin{align}
        p_{\theta_i}(x) = \exp\big(
                x\theta-b(\theta)+c(x)
            \big)
    \end{align}
    where $\theta$ is a real parameter, $c$ is a real function and the log-partition function $b$ is assumed to be twice differentiable,
    then there exist $\hat{\theta},\tilde{\theta}\in[\theta_1,\theta_2]$ such that
    \begin{align}
        \kl(p_{\theta_1},p_{\theta_2}) 
        =
        \frac{b^{\prime\prime}(\tilde{\theta})}{2} \frac{(\mu(\theta_1)-\mu(\theta_2))^2}{(b^{\prime\prime}(\hat{\theta}))^2}.
    \end{align}
    Furthermore, if the variance $\bbV_{p_\theta}[X]\in [\sigma_{\min}^2,\sigma_{\max}^2]$ for any $\theta\in[\theta_1,\theta_2]$, then
    \begin{align}
        \frac{(\mu(\theta_1)-\mu(\theta_2))^2}{2} \frac{\sigma_{\min}^2}{\sigma_{\max}^4} 
        \leq
        \kl(p_{\theta_1},p_{\theta_2}) 
        \leq
        \frac{(\mu(\theta_1)-\mu(\theta_2))^2}{2} \frac{\sigma_{\max}^2}{\sigma_{\min}^4} .
    \end{align}
\end{lemma}
\begin{proof}
    Denote the mean $\mu(\theta)=\bbE_{p_\theta}[X] $ and the variance $\sigma^2(\theta)=\bbV_{p_\theta}[X]$.
    According to \citet[Lemma 6]{garivier2011KLUCB}, 
    \begin{align}
        \kl(p_{\theta_1},p_{\theta_2}) = \mu(\theta_1) (\theta_1-\theta_2)-b(\theta_1)+b(\theta_2).
    \end{align}
    By the property of the one-dimensional canonical exponential family, we also have $\mu(\theta)=b^\prime(\theta) $ and $\sigma^2(\theta)=b^{\prime\prime}(\theta) $. Because $b$ is twice differentiable, by the mean value theorem
    \begin{align}
        &b(\theta_1)-b(\theta_2)=b^\prime(\theta_1)+\frac{b^{\prime\prime}(\tilde{\theta})}{2} (\theta_1-\theta_2)^2,
        \\
        &b^{\prime}(\theta_1)-b^{\prime}(\theta_2)= b^{\prime\prime}(\hat{\theta}) (\theta_2-\theta_1)
    \end{align}
    Therefore, we have
    \begin{align}
        \kl(p_{\theta_1},p_{\theta_2}) 
        =
        \frac{b^{\prime\prime}(\tilde{\theta})}{2} (\theta_1-\theta_2)^2
        =
        \frac{b^{\prime\prime}(\tilde{\theta})}{2} \frac{(\mu(\theta_1)-\mu(\theta_2))^2}{(b^{\prime\prime}(\hat{\theta}))^2}.
    \end{align}
    Additionally, if $\bbV_{p_\theta}[X]\in [\sigma_{\min}^2,\sigma_{\max}^2]$, then
    \begin{align}
        \frac{(\mu(\theta_1)-\mu(\theta_2))^2}{2} \frac{\sigma_{\min}^2}{\sigma_{\max}^4} 
        \leq
        \kl(p_{\theta_1},p_{\theta_2}) 
        \leq
        \frac{(\mu(\theta_1)-\mu(\theta_2))^2}{2} \frac{\sigma_{\max}^2}{\sigma_{\min}^4} .
    \end{align}

    This lemma is a version of the well-known connection between the relative entropy and the Fisher information~\citep{amari2000methods} adapted to our needs in which the distributions are not necessarily close.
\end{proof}

\vskip 0.2in
\bibliography{cleaned_references}

@inproceedings{audibert2010best,
  author     = {Audibert, Jean-Yves and Bubeck, S{\'e}bastien},
  title      = {Best arm identification in multi-armed bandits},
  booktitle  = {Proceedings of the 23rd Conference on Learning Theory},
  pages      = {41--53},
  year       = {2010},
}

@article{bubeck2011pure,
  author     = {Bubeck, S{\'e}bastien and Munos, R{\'e}mi and Stoltz, Gilles},
  title      = {Pure exploration in finitely-armed and continuous-armed bandits},
  journal    = {Theoretical Computer Science},
  volume     = {412},
  number     = {19},
  pages      = {1832--1852},
  year       = {2011},
}

@article{mao2019newsheadline,
  author     = {Mao, Yizhi and Chen, Miao and Wagle, Abhinav and Pan, Junwei and Natkovich, Michael and Matheson, Don},
  title      = {A Batched Multi-Armed Bandit Approach to News Headline Testing},
  journal    = {arXiv preprint arXiv:1908.06256},
  year       = {2019},
}

@article{flore2025balancing,
  author     = {Sentenac, Flore and Lee, Ilbin and Szepesvari, Csaba},
  title      = {Balancing optimism and pessimism in offline-to-online learning},
  journal    = {arXiv preprint arXiv:2502.08259},
  year       = {2025},
}

@book{amari2000methods,
  author     = {Amari, Shun-ichi and Nagaoka, Hiroshi},
  title      = {Methods of Information Geometry},
  volume     = {191},
  year       = {2000},
  publisher  = {American Mathematical Society},
  address    = {Providence, RI},
}

@inproceedings{yang25,
  author     = {Yang, Le and Tan, Vincent and Cheung, Wang Chi},
  title      = {Best Arm Identification with Possibly Biased Offline Data},
  booktitle  = {Proceedings of the 41st Conference on Uncertainty in Artificial Intelligence},
  pages      = {4715--4730},
  year       = {2025},
}

@inproceedings{yu2011unimodal,
  author     = {Yu, Jia Yuan and Mannor, Shie},
  title      = {Unimodal bandits},
  booktitle  = {Proceedings of the 28th International Conference on Machine Learning},
  pages      = {41--48},
  year       = {2011},
}

@inproceedings{hou2024probably,
  author     = {Hou, Yunlong and Tan, Vincent Y. F. and Zhong, Zixin},
  title      = {Probably Anytime-Safe Stochastic Combinatorial Semi-Bandits},
  booktitle  = {Proceedings of the 40th International Conference on Machine Learning},
  volume     = {202},
  pages      = {13353--13409},
  year       = {2023},
  series     = {Proceedings of Machine Learning Research},
  publisher  = {PMLR},
}

@inproceedings{Amani2019Linear,
  author     = {Amani, Sanae and Alizadeh, Mahnoosh and Thrampoulidis, Christos},
  title      = {Linear stochastic bandits under safety constraints},
  booktitle  = {Proceedings of the 33rd International Conference on Neural Information Processing Systems},
  volume     = {32},
  pages      = {9256--9266},
  year       = {2019},
}

@inproceedings{Combes2014unimodal,
  author     = {Combes, Richard and Proutiere, Alexandre},
  title      = {Unimodal Bandits: Regret Lower Bounds and Optimal Algorithms},
  booktitle  = {Proceedings of the 31st International Conference on Machine Learning},
  volume     = {32},
  pages      = {521--529},
  year       = {2014},
  series     = {Proceedings of Machine Learning Research},
  publisher  = {PMLR},
}

@article{yang2026,
  author     = {Yang, Junwen and Jin, Tianyuan and Tan, Vincent Y. F.},
  title      = {Best Arm Identification with Minimal Regret},
  journal    = {Journal of Machine Learning Research},
  year       = {2026},
}

@article{Eyal2006action,
  author     = {Even-Dar, Eyal and Mannor, Shie and Mansour, Yishay},
  title      = {Action Elimination and Stopping Conditions for the Multi-Armed Bandit and Reinforcement Learning Problems},
  journal    = {Journal of Machine Learning Research},
  volume     = {7},
  number     = {39},
  pages      = {1079--1105},
  year       = {2006},
}

@inproceedings{Riou2020bandit,
  author     = {Riou, Charles and Honda, Junya},
  title      = {Bandit Algorithms Based on {Thompson} Sampling for Bounded Reward Distributions},
  booktitle  = {Proceedings of the 31st International Conference on Algorithmic Learning Theory},
  volume     = {117},
  pages      = {777--826},
  year       = {2020},
  publisher  = {PMLR},
}

@inproceedings{Kalyanakrishnan2012PAC,
  author     = {Kalyanakrishnan, Shivaram and Tewari, Ambuj and Auer, Peter and Stone, Peter},
  title      = {{PAC} subset selection in stochastic multi-armed bandits},
  booktitle  = {Proceedings of the 29th International Conference on Machine Learning},
  pages      = {227--234},
  year       = {2012},
  series     = {ICML'12},
}

@article{Friedrich2019High,
  author     = {G{\"o}tze, Friedrich and Sambale, Holger and Sinulis, Arthur},
  title      = {{Higher order concentration for functions of weakly dependent random variables}},
  journal    = {Electronic Journal of Probability},
  volume     = {24},
  pages      = {1 -- 19},
  year       = {2019},
  publisher  = {Institute of Mathematical Statistics and Bernoulli Society},
}

@inproceedings{bubeck2013multiple,
  author     = {Bubeck, S\'ebastian and Wang, Tengyao and Viswanathan, Nitin},
  title      = {Multiple Identifications in Multi-Armed Bandits},
  booktitle  = {Proceedings of the 30th International Conference on Machine Learning},
  pages      = {258--265},
  year       = {2013},
}

@inproceedings{carpentier2015simple,
  author     = {Carpentier, Alexandra and Valko, Michal},
  title      = {Simple regret for infinitely many armed bandits},
  booktitle  = {Proceedings of the 32nd International Conference on Machine Learning},
  pages      = {1133--1141},
  year       = {2015},
}

@inproceedings{carpentier16tight,
  author     = {Carpentier, Alexandra and Locatelli, Andrea},
  title      = {Tight (Lower) Bounds for the Fixed Budget Best Arm Identification Bandit Problem},
  booktitle  = {Proceedings of the 29th Annual Conference on Learning Theory},
  volume     = {49},
  pages      = {590--604},
  year       = {2016},
  series     = {Proceedings of Machine Learning Research},
  publisher  = {PMLR},
}

@inproceedings{cheung2024leveraging,
  author     = {Cheung, Wang Chi and Lyu, Lixing},
  title      = {Leveraging (biased) information: multi-armed bandits with offline data},
  booktitle  = {Proceedings of the 41st International Conference on Machine Learning},
  pages      = {8286--8309},
  year       = {2024},
}

@inproceedings{degenne2019bridging,
  author     = {Degenne, R\'emy and Nedelec, Thomas and Calauzenes, Clement and Perchet, Vianney},
  title      = {Bridging the gap between regret minimization and best arm identification, with application to {A/B} tests},
  booktitle  = {Proceedings of the 22nd International Conference on Artificial Intelligence and Statistics},
  pages      = {1988--1996},
  year       = {2019},
}

@inproceedings{kaufmann13information,
  author     = {Kaufmann, Emilie and Kalyanakrishnan, Shivaram},
  title      = {Information Complexity in Bandit Subset Selection},
  booktitle  = {Proceedings of the 26th Annual Conference on Learning Theory},
  pages      = {228--251},
  year       = {2013},
}

@inproceedings{kausik2025leveraging,
  author     = {Kausik, Chinmaya and Tan, Kevin and Tewari, Ambuj},
  title      = {Leveraging Offline Data in Linear Latent Contextual Bandits},
  booktitle  = {Proceedings of the 42nd International Conference on Machine Learning},
  year       = {2025},
}

@inproceedings{kirschner2021asymptotically,
  author     = {Kirschner, Johannes and Lattimore, Tor and Vernade, Claire and Szepesv{\'a}ri, Csaba},
  title      = {Asymptotically optimal information-directed sampling},
  booktitle  = {Proceedings of the 34th Conference on Learning Theory},
  pages      = {2777--2821},
  year       = {2021},
}

@article{qin2024optimizing,
  author     = {Qin, Chao and Russo, Daniel},
  title      = {Optimizing adaptive experiments: A unified approach to regret minimization and best-arm identification},
  journal    = {arXiv preprint arXiv:2402.10592},
  year       = {2024},
}

@inproceedings{shivaswamy2012multi,
  author     = {Shivaswamy, Pannagadatta and Joachims, Thorsten},
  title      = {Multi-armed bandit problems with history},
  booktitle  = {Proceedings of the 15th International Conference on Artificial Intelligence and statistics},
  pages      = {1046--1054},
  year       = {2012},
}

@article{zhang2023fast,
  author     = {Zhang, Qining and Ying, Lei},
  title      = {Fast and regret optimal best arm identification: Fundamental limits and low-complexity algorithms},
  journal    = {Advances in Neural Information Processing Systems},
  volume     = {36},
  pages      = {16729--16769},
  year       = {2023},
}

@inproceedings{zhao2023revisiting,
  author     = {Zhao, Yao and Stephens, Connor and Szepesv{\'a}ri, Csaba and Jun, Kwang-Sung},
  title      = {Revisiting simple regret: Fast rates for returning a good arm},
  booktitle  = {Proceedings of the 40th International Conference on Machine Learning},
  pages      = {42110--42158},
  year       = {2023},
}

@article{zhong2023achieving,
  author     = {Zhong, Zixin and Cheung, Wang Chi and Tan, Vincent},
  title      = {Achieving the {Pareto} Frontier of Regret Minimization and Best Arm Identification in Multi-Armed Bandits},
  journal    = {Transactions on Machine Learning Research},
  year       = {2023},
}

@inproceedings{jamieson2014best,
  author     = {Jamieson, Kevin and Nowak, Robert},
  title      = {Best-arm identification algorithms for multi-armed bandits in the fixed confidence setting},
  booktitle  = {2014 48th Annual Conference on Information Sciences and Systems (CISS)},
  pages      = {1--6},
  year       = {2014},
}

@inproceedings{jamiesonlilucb2014,
  author     = {Jamieson, Kevin and Malloy, Matthew and Nowak, Robert and Bubeck, S{\'e}bastien},
  title      = {lil' {UCB} : An Optimal Exploration Algorithm for Multi-Armed Bandits},
  booktitle  = {Proceedings of The 27th Conference on Learning Theory},
  volume     = {35},
  pages      = {423--439},
  year       = {2014},
  series     = {Proceedings of Machine Learning Research},
  publisher  = {PMLR},
  address    = {Barcelona, Spain},
  editor     = {Balcan, Maria Florina and Feldman, Vitaly and Szepesv{\'a}ri, Csaba},
}

@inproceedings{garivier2016optimal,
  author     = {Garivier, Aur{\'e}lien and Kaufmann, Emilie},
  title      = {Optimal Best Arm Identification with Fixed Confidence},
  booktitle  = {29th Annual Conference on Learning Theory},
  volume     = {49},
  pages      = {998--1027},
  year       = {2016},
  series     = {Proceedings of Machine Learning Research},
  publisher  = {PMLR},
}

@inproceedings{karnin2013optimal,
  author     = {Karnin, Zohar and Koren, Tomer and Somekh, Oren},
  title      = {Almost Optimal Exploration in Multi-Armed Bandits},
  booktitle  = {Proceedings of the 30th International Conference on Machine Learning},
  volume     = {28},
  pages      = {1238--1246},
  year       = {2013},
  series     = {Proceedings of Machine Learning Research},
  publisher  = {PMLR},
}

@inproceedings{magureanu2014lipschitz,
  author     = {Magureanu, Stefan and Combes, Richard and Proutiere, Alexandre},
  title      = {Lipschitz Bandits: Regret Lower Bound and Optimal Algorithms},
  booktitle  = {Proceedings of The 27th Conference on Learning Theory},
  volume     = {35},
  pages      = {975--999},
  year       = {2014},
  series     = {Proceedings of Machine Learning Research},
  publisher  = {PMLR},
}

@inproceedings{garivier2011KLUCB,
  author     = {Garivier, Aur{\'e}lien and Capp{\'e}, Olivier},
  title      = {The {KL-{UCB}} Algorithm for Bounded Stochastic Bandits and Beyond},
  booktitle  = {Proceedings of the 24th Annual Conference on Learning Theory},
  volume     = {19},
  pages      = {359--376},
  year       = {2011},
  series     = {Proceedings of Machine Learning Research},
  publisher  = {PMLR},
}

@inproceedings{Abbasi2011improved,
  author     = {Abbasi-yadkori, Yasin and P\'{a}l, D\'{a}vid and Szepesv\'{a}ri, Csaba},
  title      = {Improved Algorithms for Linear Stochastic Bandits},
  booktitle  = {Advances in Neural Information Processing Systems},
  volume     = {24},
  pages      = {2312--2320},
  year       = {2011},
}

@article{auer2002finite,
  author     = {Auer, Peter and Cesa-Bianchi, Nicolo and Fischer, Paul},
  title      = {Finite-time Analysis of the Multi-armed Bandit Problem},
  journal    = {Machine Learning},
  volume     = {47},
  pages      = {235--256},
  year       = {2002},
}

@book{lattimore2020bandit,
  author     = {Lattimore, Tor and Szepesv{\'a}ri, Csaba},
  title      = {Bandit Algorithms},
  year       = {2020},
  publisher  = {Cambridge University Press},
}

@article{agrawal2017near,
  author     = {Agrawal, Shipra and Goyal, Navin},
  title      = {Near-optimal regret bounds for {Thompson} Sampling},
  journal    = {Journal of the ACM (JACM)},
  volume     = {64},
  number     = {5},
  pages      = {1--24},
  year       = {2017},
  publisher  = {ACM New York, NY, USA},
}

@article{burnetas96,
  author     = {Burnetas, A. N. and Katehakis, M. N.},
  title      = {Optimal adaptive policies for sequential allocation problems},
  journal    = {Advances in Applied Mathematics},
  volume     = {17},
  number     = {2},
  pages      = {122--142},
  year       = {1996},
  publisher  = {Academic Press},
}

@article{lai1985asymptotically,
  author     = {Lai, Tze Leung and Robbins, Herbert},
  title      = {Asymptotically efficient adaptive allocation rules},
  journal    = {Advances in Applied Mathematics},
  volume     = {6},
  number     = {1},
  pages      = {4--22},
  year       = {1985},
  publisher  = {Academic Press},
}

@inproceedings{zhong2021probabilistic,
  author     = {Zhong, Zixin and Cheung, Wang Chi and Tan, Vincent},
  title      = {Probabilistic Sequential Shrinking: A Best Arm Identification Algorithm for Stochastic Bandits with Corruptions},
  booktitle  = {Proceedings of the 38th International Conference on Machine Learning},
  volume     = {139},
  pages      = {12772--12781},
  year       = {2021},
  series     = {Proceedings of Machine Learning Research},
  publisher  = {PMLR},
}

@article{kaufmann2016complexity,
  author     = {Kaufmann, Emilie and Capp{\'e}, Olivier and Garivier, Aur{\'e}lien},
  title      = {On the Complexity of Best Arm Identification in Multi-Armed Bandit Models},
  journal    = {Journal of Machine Learning Research},
  volume     = {17},
  pages      = {1--42},
  year       = {2016},
}


\end{document}